\documentclass[11 pt]{article}
\usepackage[utf8]{inputenc}
\usepackage[T1]{fontenc}
\usepackage[colorlinks = true, linkcolor=Mahogany, citecolor = ForestGreen]{hyperref}
\usepackage{url}
\usepackage{booktabs}
\usepackage{amsfonts}
\usepackage{amsmath}
\usepackage{amssymb}
\usepackage{amsthm}
\usepackage{nicefrac}
\usepackage{microtype}
\usepackage[dvipsnames]{xcolor}
\usepackage{newtxtext}

\usepackage[square, sort&compress, numbers]{natbib}
\usepackage{parskip}
\usepackage[margin=1in]{geometry}
\usepackage{changepage}

\usepackage{amsfonts}
\usepackage{breqn}
\usepackage{todonotes}
\usepackage{dsfont}
\usepackage{marginnote}
\usepackage{tikz,pgf}
\usepackage{tkz-euclide}
\usepackage{mathtools}
\usepackage{algorithm}
\usepackage{algpseudocode}
\usepackage{hyperref}
\usepackage{enumitem}
\usepackage{tcolorbox}
\usepackage{tikz}
\usetikzlibrary{positioning}
\usepackage{fullpage}

\DeclarePairedDelimiter\floor{\lfloor}{\rfloor}

\newtheorem{theorem}{Theorem}
\newtheorem{proposition}{Proposition}
\newtheorem{lemma}[proposition]{Lemma}
\newtheorem{definition}{Definition}
\newtheorem{claim}{Claim}

\newtheorem{assumption}{Assumption}

\newcommand{\indep}{\perp \!\!\! \perp}
\newcommand{\E}[2][]{\mathbb{E}_{#1}\left[#2\right]}

\newcommand{\prob}[2][]{\mathbb{P}_{#1}\left[#2\right]}

\definecolor{forestgreen}{RGB}{34,139,34}

\newtheorem{remark}{Remark}

\newcommand{\norm}[1]{\left\lvert #1 \right\rvert}
\newcommand{\var}[2][]{\operatorname{var}_{#1}\left(#2\right)}
\newcommand{\poly}{\operatorname{poly}}

\newcommand{\sigmamin}{\sigma_{\rm min}}
\newcommand{\sigmamax}{\sigma_{\rm max}}
\newcommand{\ymax}{y_{\rm max}}
\newcommand{\bmax}{\beta_{\rm max}}
\newcommand{\bmin}{\beta_{\rm min}}

\newcommand{\abs}[1]{\left| #1 \right|}

\newcommand{\Pbb}{\mathbb{P}}
\newcommand{\Ebb}{\mathbb{E}}

\newcommand{\KL}{D_{\mathrm{KL}}}

\newcommand{\Var}{\operatorname{Var}}
\newcommand{\algone}{LET-GL}
\newcommand{\algtwo}{BTR-GL}
\newcommand{\algonefull}{Local Edge Testing for Graph Learning}
\newcommand{\algtwofull}{Burn-in/Thinning Reduction for Graph Learning}

\DeclareGraphicsExtensions{.png,.pdf}
\graphicspath{{Supplementary/Images/}}

\title{Local and Mixing-Based Algorithms for Gaussian Graphical Model Selection from Glauber Dynamics}

\vspace{-5mm}

\author{
{\small
\begin{tabular}{c}
Vignesh Tirukkonda \qquad Anirudh Rayas \qquad Gautam Dasarathy \\
Arizona State University \\
{\footnotesize\texttt{\{vtirukko, ahrayas, gautamd\}@asu.edu}}
\end{tabular}
}
}

\date{}

\begin{document}
\maketitle
\vspace{-5mm}
\begin{abstract}

Gaussian graphical model selection is usually studied under independent
sampling, but in many applications observations arise from dependent dynamics.
We study structure learning when the data consist of a single trajectory of
Gaussian Glauber dynamics. We develop two complementary approaches. The first
is a local edge-testing estimator based on an appropriately designed correlation test that reveals edges. This estimator does not require waiting for the chain
to mix and admits an embarrassingly parallel edgewise implementation. The
second is a burn-in/thinning reduction: under a Dobrushin contraction condition,
we prove that a suitably subsampled Gaussian Gibbs trajectory is close in total
variation to an i.i.d. product sample, allowing standard i.i.d. Gaussian
graphical model learners to be used as black boxes. The key technical ingredient, which may be of independent interest,
is a high-dimensional total-variation bound for random-scan Gaussian Gibbs
samplers, obtained by combining Wasserstein contraction with an approximate
Lipschitz smoothing argument. We prove finite-sample recovery guarantees for
both approaches, establish information-theoretic lower bounds on the observation
time, and empirically compare the resulting sample-computation tradeoffs.
\end{abstract}

\section{Introduction}\label{sec: introduction}

Graphical models are a powerful class of stochastic models that represent conditional dependence structure between random variables using a graph, with extensive applications in bioinformatics~\citep{Murad2021Biology}, finance~\citep{zhan2021graphical}, and social network analysis~\citep{Farasat2015}. An important subclass is the {\em Gauss-Markov random field} (GMRF) or {\em Gaussian graphical model} (GGM), which models multivariate Gaussian distributions. A central challenge for graphical models is {\em structure learning}: identifying the graph that best fits observed data. Traditional methods assume access to independent and identically distributed (i.i.d.) samples from the underlying distribution (see~\citep{drton2016structure} for a review). In many practical scenarios, however, i.i.d.\ samples are unavailable, and the data are more naturally modeled as stochastic processes that evolve over time~\citep{Dennis2022Markov, Lara2019GlauberEpidemicModel}. In this paper, we study GGM structure learning when the data are generated by {\em Glauber dynamics} (also known as the Gibbs sampler): a Markov chain that iteratively updates each node according to its local conditional distribution and whose stationary distribution is the GGM of interest. Glauber dynamics has been used to model the spread of information and diseases across networks~\citep{Montanari2010InformationSpread, Lara2019GlauberEpidemicModel}, agent decisions in coordination games~\citep{Auletta2017}, and processes whose stationary law is multivariate Gaussian. This work extends~\citep{bresler2014learning}, which considered the Ising case.  

{\bf Correction to an earlier version and relation to recent work.}
A preliminary version of this work~\citep{TRD25} used an ``\(i,j,i\)-test'' based ratio
estimator. That argument has a dependence gap: the denominator in the ratio is not
conditionally decoupled from the earlier noise injected at coordinate \(i\), and removing this gap required strong assumptions on the underlying precision matrix. 
We thank Mahbod Majid for identifying this issue. The present version replaces
that estimator with a local product statistic based on ``\(i,i,j,i\)-test'' 
windows, where the additional update of coordinate \(i\) acts as a buffer that
removes the problematic dependence. A recent and closely related work of Shen,
Wu, Majid, and Moitra~\citep{SWMM26} also identifies this obstruction and
develops an algorithm based on \(i,i,j,i\)-type windows with a diagonal
normalization and robust aggregation of ratio statistics. Our local estimator, which was developed independently of~\cite{SWMM26}, uses a
different product-moment statistic and a different concentration argument. The Dobrushin condition we require however also enables our
second contribution: a separate burn-in/thinning reduction based on
total-variation decoupling of Gaussian Gibbs trajectories, which results in an algorithm with significantly better data requirements. We next discuss our proposed algorithms.

{\bf \algone{} (\algonefull{}): A local edge-testing based estimator.}
Our first algorithm exploits the local nature of Glauber updates directly, without waiting for the chain to mix. For each candidate edge $\{i,j\}$, we identify short windows of the trajectory in which $i$ and $j$ update in a specific alternating pattern while no other neighbor of $i$ is updated. On such a window, the product of increments in $X_i$ and $X_j$ has an expectation that is informative about the edge strength $\beta_{ij}$, and zero when $\{i,j\}$ is not an edge. Averaging over many such windows yields a test statistic that we threshold to recover the edge set. Under natural conditions, we prove that observing the Glauber dynamics for $\mathcal{O}\!\left(\frac{d^{3} p \operatorname{polylog} p}{\bmin^{5}}\right)$ time suffices to recover a graph of maximum degree~$d$, where $\bmin=\min_{\{i,j\}\in E}\min\{|\beta_{ij}|,|\beta_{ji}|\}$ is the smallest edge strength. A distinctive feature of this estimator is that it is perfectly parallelizable across edges, with an effective per-core cost as low as $\tilde{\mathcal{O}}(d^{2} p)$\footnote{The notation $g = \tilde{\mathcal{O}}(f)$ means there exist constants $c_1,c_2$ such that $g \leq c_1 f \log(f)^{c_2}$.}. This is comparable to the best parallel cost of regression-based neighborhood selection~\citep{Meinshausen2006Neighborhood}, and is substantially lower than the cost of penalized likelihood methods such as GLASSO~\citep{friedman2008sparse} ($\mathcal{O}(p^{3})$), constraint-based methods such as the PC algorithm~\citep{10.7551/mitpress/1754.001.0001} ($\mathcal{O}(p^{d+2})$), or sample-optimal combinatorial methods such as DICE~\citep{misra2020information} ($\mathcal{O}(p^{2d+1})$).

{\bf \algtwo{} (\algtwofull{}): a post-mixing estimator via burn-in and thinning.}
Our second algorithm follows a complementary post-mixing route. After a suitable burn-in period, thinned states of the Glauber dynamics are close in joint law to independent draws from the Gaussian stationary distribution, and may therefore be passed to a standard i.i.d.\ GGM structure learner. Under a Dobrushin-style contraction condition on the conditional regressions, we discard an initial burn-in segment and then retain every $\mathfrak{t}$-th sample. When the base learner is instantiated with DICE~\citep{misra2020information}, the resulting procedure recovers the graph after
\(\mathcal{O}\!\left(\frac{d\, p\, {\rm polylog}(p/\delta)}{\kappa^2(1-r)}\right)\)
Glauber updates, where $\kappa=\min_{\{i,j\}\in E}\sqrt{|\beta_{ij}\beta_{ji}|}=\min_{\{i,j\}\in E}|\Theta_{ij}|/\sqrt{\Theta_{ii}\Theta_{jj}}$ is the minimum normalized edge strength (the absolute partial correlation, matching the convention of~\citep{misra2020information}) and $r=\max_i\sum_{j\ne i}|\beta_{ij}|$ is the Dobrushin radius. The principal technical obstacle in analyzing this estimator is converting Wasserstein contraction for the random-scan Gaussian Gibbs sampler~\citep{AMIT199182, roberts1997updating, liu1994covariance, wang2014convergence} into total-variation bounds suitable for this reduction. Since the transition kernel is not globally Lipschitz, the standard Kantorovich--Rubinstein argument~\citep{Villani2009} is insufficient. We resolve this difficulty by introducing a {\em thresholded} Lipschitz property, in which the kernel is Lipschitz on a high-probability set and the failure event contributes an explicit additive error. The resulting high-dimensional TV bound may be of independent interest, and yields end-to-end guarantees for the burn-in/thinning estimator in terms of the chosen i.i.d.\ base learner. Relative to \algone{}, \algtwo{} can require shorter trajectories in the Dobrushin regime, but it relies on a mixing condition and does not retain the per-edge parallel structure of the local estimator.

{\bf TV decoupling for random-scan Gaussian Gibbs.}
A key technical ingredient behind \algtwo{} --- and a contribution of independent interest --- is a high-dimensional total-variation mixing bound for the random-scan Gaussian Gibbs sampler. Under the same Dobrushin condition $r<1$, we show (Lemma~\ref{lem:tv-mixing-bound}) that
\[
\left\|K^{t}(x,\cdot)-\pi\right\|_{\mathrm{TV}} \;\le\; \varepsilon
\qquad\text{whenever}\qquad
t \;\ge\; C\,\frac{p}{1-r}\,\log\!\left(\frac{p^{3/2}}{\varepsilon}\right).
\]
That is, the normalized mixing time is only polylogarithmic in the ambient dimension $p$ and the precision $\varepsilon$, with no spectral-gap or condition-number assumptions on $\Theta$. The standard route from Wasserstein contraction~\citep{AMIT199182, roberts1997updating, liu1994covariance, wang2014convergence, wang2017convergence} to TV via Kantorovich--Rubinstein smoothing~\citep{Villani2009} fails here because the transition kernel is not globally Lipschitz; we circumvent this by introducing a {\em thresholded} (or approximate) Lipschitz property, in which the smoothed kernel is Lipschitz on a high-probability set and the failure event enters as an explicit additive defect. Combined with a burn-in/thinning decomposition, this yields a joint TV bound between the subsampled trajectory and a product of stationary draws, which is exactly the input format an i.i.d.\ Gaussian structure learner expects. The resulting TV mixing estimate may be of independent interest to the MCMC community.

{\bf Lower bound.}
We complement these upper bounds with an information-theoretic lower bound on
the normalized observation time required by any structure-learning algorithm.
For a broad family of GGMs, every estimator must observe the Glauber dynamics
for at least
\(
    \Omega\!\left(\frac{\log(p-d)}{\beta_{\min}^{2}}\right)
\)
normalized time. In the scaling $\beta_{\min}=\Theta(1/d)$, this becomes
$\Omega(d^{2}\log p)$. This benchmark places \algtwo{} close to minimax
optimal in the Dobrushin regime: if $\kappa\asymp\beta_{\min}$, the maximum
degree is bounded, and $1-r$ is bounded below by a constant, then its normalized
observation time matches the lower bound up to polylogarithmic factors. For
general sparse graphs, the remaining factors have transparent sources, namely
the sample requirement of the i.i.d.\ base learner and the cost of converting a
dependent trajectory into approximately independent samples. The local
estimator \algone{} has a larger statistical overhead, but it provides a
complementary guarantee that admits a fully
parallel implementation.

{\bf Experiments.}
Finally, we include synthetic experiments on random $d$-regular GGMs to
illustrate the finite-sample behavior suggested by the theory. The experiments
track recovery as a function of observation time, the computational scaling of
the local statistic, and the effect of burn-in and thinning in the Dobrushin
regime. They support the qualitative tradeoff between the two approaches:
\algone{} exposes substantial edgewise parallelism and operates directly on the
dependent trajectory, whereas \algtwo{} behaves closer to an i.i.d.\ learner
after sufficient decoupling and inherits the assumptions and computational
profile of the chosen base procedure.

\section{Preliminaries \& Problem Statement}\label{sec: problem statement}

\textbf{Notation.} Scalar random variables are typically denoted by uppercase letters such as $X$ and $Y$, and the values taken by these random variables are denoted by corresponding lowercase letters such as $x$ and $y$. Random vectors are typically denoted in boldface; e.g.,  $\mathbf{X} \in \mathbb{R}^{p}$ is a $p-$dimensional random vector. The components of $\mathbf{X}$ and their corresponding sample values are written as $X_{i}$ and $x_{i}$, $i \in [p]$. Furthermore, for an index set $A \subseteq [p]$, $\mathbf{X}_{A}$ denotes the the $|A|-$dimensional subvector of $\mathbf{X}$ that has the variables corresponding indices in $A$ concatenated.  Matrices are often denoted with boldface Greek letters in uppercase. For example, $\boldsymbol{\Theta} \in \mathbb{R}^{p \times p}$ is a $p \times p$ matrix. Let $A,B \subseteq [p]$. The submatrix, denoted by $\boldsymbol{\Theta}_{AB}$, is a $|A| \times |B|$ matrix with the elements $\theta_{ab}$ where $a \in A$ and $b \in B$. Typically, continuous-time (and discrete-time stochastic processes) are denoted as a collection $\{\mathbf{Y}^{(t)}\}_{t \in \mathbb{R}} $ (and as a sequence $\{\mathbf{X}^{(n)}\}_{n \in \mathbb{N}}$) where $t \in \mathbb{R}$ represents time (and $n\in \mathbb{N}$ represents the state number.) 

\subsection{Gaussian Graphical Models}
\label{subsec: gaussian graphical model}
Let $G = (V,E)$ be a graph on a vertex set $V = [p] = \{1,2,\ldots, p\}$ with an edge set $E \subseteq {V \choose 2}$. For each vertex $i \in[p]$, its neighborhood is defined as $N(i) \triangleq \left\{j : \left\{i,j\right\} \in E\right\}$ and its closure is defined as $\operatorname{cl}(i) \triangleq N(i)\cup\{i\}$. The degree of node $i$ is defined as $d_i \triangleq |N(i)|$, while degree of the graph $G$ is $d = \max_{i\in [p]} d_i$. A graphical model is an ordered pair $(G, \mathbf{X})$, where $\mathbf{X}$ is a zero-mean $p$-dimensional random vector such that (a) each vertex $i \in [p]$ of the graph is identified with the corresponding component $X_i$ of $\mathbf{X}$ and (b) the edge structure of $G$ encodes the conditional (in)dependences of $\mathbf{X}$: $X_i \indep X_j \mid X_{V\setminus\{i,j\}} \Rightarrow \{i,j\}\notin E$. If this relationship holds, we say that $\mathbf{X}$ is markov with respect to $G$. 

$(G, f_{\mathbf{X}})$ is said to be a {\em Gaussian graphical model} or a {\em Gauss-Markov random field} if $f_{\mathbf{X}}$ is the multivariate normal distribution. That is, there is a symmetric positive definite matrix $\boldsymbol{\Sigma} \in \mathbb{S}^{p}_{++}$ such that 
\begin{align*}
    f_{\mathbf{X}}(x) = \frac{1}{(2\pi)^{p/2} \abs{\boldsymbol{\Sigma}}^{1/2}} \exp\left( -\frac{1}{2}x^{T} \boldsymbol{\Sigma}^{-1}x \right).
\end{align*}

The matrix $\boldsymbol{\Theta} \triangleq \boldsymbol{\Sigma}^{-1}$ is called the {\em precision} or {\em concentration} matrix of $\mathbf{X}$. An attractive feature of the Gaussian graphical model, which lies at the heart of its broad applicability, is that when $\mathbf{X}$ is Markov with respect to $G$, the following holds\footnote{the fact that $\theta_{ij} = 0$ implies that the corresponding edge does not exist in the graph requires the so-called faithfulness assumption~\citep{lauritzen1996}, which is only violated on a set of measure 0; we will assume faithfulness holds here.} 
\begin{equation}\label{eq: Conditional Independence Property of GGM}
   \{i,j\} \notin E \iff (\boldsymbol{\Sigma}^{-1})_{ij} = \boldsymbol{\Theta}_{ij} = 0.  
\end{equation}

Another useful feature of the multivariate Gaussian is that it is closed under conditioning~\citep{rue2005gaussian}. That is, suppose $A \subset V$ and denote by $B$ the set $V \setminus A$ so that $
\mathbf{X} = \begin{pmatrix}
\mathbf{X}_{A} \\
\mathbf{X}_{B}
\end{pmatrix}
$ is partitioned by the subvectors $\mathbf{X}_A$ and $\mathbf{X}_{B}$ with zero mean and its precision matrix partitioned accordingly, 
\begin{align}
\boldsymbol{\Theta} = \begin{pmatrix}
\boldsymbol{\Theta}_{AA} & \boldsymbol{\Theta}_{AB} \\
\boldsymbol{\Theta}_{BA} & \boldsymbol{\Theta}_{BB}
\end{pmatrix}.
\end{align}

Then the distribution of $\mathbf{X_A}$ conditioned on $\mathbf{X_B} = x_B$ is also Gaussian 
with mean 
$\boldsymbol{\mu}_{A|B}$ and precision matrix $\boldsymbol{\Theta}_{A| B}$, where 
\begin{align*}
\boldsymbol{\mu} _{A | B}= -\boldsymbol{\Theta}^{-1}_{AA} \boldsymbol{\Theta}_{AB}\left(x_B\right) \quad \text { and } \quad \boldsymbol{\Theta}_{A | B}=\boldsymbol{\Theta}_{AA}.
\end{align*}

Using the block matrix inversion lemma~\citep{Horn_Johnson_2012_Matrix}, we get $\boldsymbol{\Theta}_{AA}^{-1} =\boldsymbol{\Sigma}_{AA}-\boldsymbol{\Sigma}_{AB} \boldsymbol{\Sigma}_{BB}^{-1} \boldsymbol{\Sigma}_{BA}$.
From this one can check that, for any $i\in V$, the conditional distribution of $X_i$ given $\mathbf{X}_{V \setminus \{i\}}$ is a normal distribution with mean and variance given by 
\begin{align}\label{eq: conditional normal distribution}
       \mu_{X_i | \boldsymbol{X}_{V \setminus \{i\}}} &= \sum_{j \in V\setminus \{i\}}-\frac{\theta_{ij}}{\theta_{ii}} X_j\text{ and } \nonumber \\  \sigma^{2}_{X_i | \mathbf{X}_{V \setminus \{i\}}} &= \sigma_{i}^2-\boldsymbol{\Sigma}_{i, V \setminus \{i\}} \boldsymbol{\Sigma}_{V \setminus \{i\},V \setminus \{i\} }^{-1} \boldsymbol{\Sigma}_{V \setminus \{i\}, i}.
\end{align}

From the above and the relationship between the precision matrix and the graph structure  \ref{eq: Conditional Independence Property of GGM}, it can be checked that $X_i$, given $\mathbf{X}_{V \setminus \{i\}}$, satisfies
\begin{equation}\label{eq: update equation}
    X_i = \sum_{j \in N(i)}\beta_{ij} X_{j} + \epsilon_i,
\end{equation}
where $\epsilon_i \sim \mathcal{N}(0, \sigma^{2}_{X_i | V \setminus \{i\}})$, $\epsilon_i \indep X_{j}$ for $j \in V\setminus\{i\}$, and $\beta_{ij} = -\frac{\theta_{ij}}{\theta_{ii}}$. Indeed, this fact played an important role in~\citep{Meinshausen2006Neighborhood}, which is considered a landmark paper on Gaussian graphical model selection.

The above property~\ref{eq: update equation} allows one to define a natural Markov chain on the GGM, called the Glauber dynamics, where the transitions depend only on local conditions.

\subsection{Glauber Dynamics \& The Data Model}\label{subsec: Glauber dynamics}

Glauber dynamics, originally introduced in statistical physics by~\citep{Glauber1963Ising}, has found important applications across various disciplines to study and model the behavior of dynamical systems. Given that it is also equivalent to single-site Gibbs sampling, it is often used as a Markov chain Monte Carlo (MCMC) algorithm to obtain samples from stationary distributions. For a comprehensive overview, the reader may consult~\citep{LevinPeresWilmer2006, Martinelli1999}. 
In this section, we provide a self-contained definition of Glauber dynamics as a stochastic process whose stationary distribution is $f_{\mathbf{X}}$, which in turn is associated with a Gaussian graphical model (GGM) $(G,\mathbf{X})$. We will then study the problem of learning the edge structure of $G$, given observations from this stochastic process. 

Continuous-time Glauber dynamics is a random process $\{\mathbf{Y}^{(t)}\}_{t \geq \mathbb{R}}$ of $p-$dimensional random vectors $\mathbf{Y}^{(t)}$ defined as follows. The process is initialized with $\mathbf{Y}^{(0)}$, an arbitrary (possibly random) vector.  It is then updated at an increasing sequence of random time points $\{S_n\}_{n \in \mathbb{N}} \subset \mathbb{R}_+$, with $S_0 = 0$, where $n$ represents the state or the round. These update times are distributed such that the time between updates $H_{n} \triangleq S_{n+1} - S_{n}$, called {\em holding times}, are independent and exponentially distributed with parameter $p$. Notice that this implies that in one unit of time, there are $p$ updates on average\footnote{It can also be checked that this is equivalent to associating a rate $1$ Poisson clock with each random variable, and updating the corresponding random variable at the arrival time according to the appropriate update rule.}. 

During the holding time, the process does not change, and it can be represented by a random vector called the {\em state vector}. Therefore, the $n^{th}$ state vector is $\mathbf{X}^{(n)} \triangleq \mathbf{Y}^{(t)}$, $S_{n} \leq t < S_{n+1}$. The updates or transitions between states obey the following rule: At time $S_{n}$, an index $I^{(n)} \in [p]$ is selected uniformly at random. If node $i$ was chosen at round $n$, i.e, $I^{(n)} = i$, the random variable $X_{i}^{(n)}$ is then sampled from $f\left(\mathbf{X}_{i} | \mathbf{X}_{V \setminus \{i\}} = \mathbf{X}^{(n-1)}_{V \setminus \{i\}}\right)$. The values of all other components of the random vector remain unchanged. In the case of multivariate Gaussians the update takes the following form: 
\begin{align}\label{eqn: glauber ggm updates}
    X_{i}^{(n)} &= \sum_{j\in N(i)} \beta_{ij} X_{j}^{(n-1)} + \epsilon_{i}^{(n)} \text{ and } \\
    X_{k}^{(n)} &= X_{k}^{(n-1)}, k \neq i,
\end{align}
where again $\epsilon_i^{(n)} \sim \mathcal{N}\left(0, \sigma^{2}_{X_i | \mathbf{X}_{N(i)}}\right)$ is independent of $X_{k}^{(n')}$ for all $k\in [p],n' \in \mathbb{N}$ except $X_{i}^{(n)}$ and $\beta_{ij} = -\frac{\theta_{ij}}{\theta_{ii}}$. This process is known to converge to corresponding multivariate Gaussian distribution $f_{X}$. We refer the interested reader to~\citep{AMIT199182}, e.g., for convergence results.

Equivalently, the update can be written as an affine transformation. That is if $I^{(t)} = i$, then $\mathbf{X}^{(t)} = M_i \mathbf{X}^{(t-1)} + \Theta_{ii}^{-1/2} e_i \xi^{(t)}$, where $M_i = I - D_i \Theta$ and $D_i$ is the diagonal matrix whose only nonzero entry is $(D_i)_{ii} = \Theta_{ii}^{-1}$. Fix an update schedule $\sigma = (\sigma_1,\dots,\sigma_t) \in [p]^t$. For integers $1 \le a \le b \le t$, define
\begin{align*}
\Phi_a^b(\sigma) := M_{\sigma_b} M_{\sigma_{b-1}} \cdots M_{\sigma_a}, \text{ and } \mathbf{B}_a^b(\sigma) := \sum_a^b \Phi_{s+1}^{t}(\sigma)\Theta_{\sigma_s \sigma_s}^{-1/2} e_{\sigma_s} \xi^{(s)}.
\end{align*}

Iterating the affine update gives $\mathbf{X}^{(t)} = M_\sigma x + \mathbf{B}_\sigma$, where $M_\sigma = \Phi_1^t(\sigma)$ and $\mathbf{B}_{\sigma} = \mathbf{B}_1^t(\sigma)$. Further, if $\mathbf{Z}=(\xi^{(1)},\dots,\xi^{(t)})^\top \sim \mathcal{N}(0,I_t)$, then $\mathbf{B}_\sigma = V_\sigma \mathbf{Z}$, with
\begin{align*}
V_\sigma =
\begin{bmatrix}
V_\sigma(1) & V_\sigma(2) & \cdots & V_\sigma(t)
\end{bmatrix},
\qquad
V_\sigma(s) := \Phi_{s+1}^t(\sigma) \Theta_{\sigma_s \sigma_s}^{-1/2} e_{\sigma_s}.
\end{align*}
Therefore, conditional on the schedule $\sigma$, $\mathbf{X}^{(t)} = M_\sigma x + V_\sigma \mathbf{Z}$ is a Gaussian vector since it is an affine transformation of $\mathbf{Z}$ with mean $M_{\sigma}x$ and covariance $V_{\sigma}V_{\sigma}^{\top}$. When $\sigma$ is not fixed, we shall write the affine update as $\mathbf{X}^{(t)} = \Phi_{1}^{t}x + \mathbf{B}_{1}^{t}$.

As discussed in Section~\ref{sec: introduction}, we suppose that this stochastic process generates the data we observe. In particular, we suppose that one observes the continuous-time Glauber dynamics $\{\mathbf{Y}^{(t)}\}$ for $T$ units of time, generating $N = Tp$ updates (in expectation) of the underlying state. In the sequel we will treat the number of updates $N$ as a deterministic quantity given that it concentrates very rapidly to $Tp$. For convenience, we suppose that the dataset includes the time, the identity of the updated node, and the corresponding updated state vector, forming the collection $\mathcal{D} = \left\{\left(\mathbf{X}^{(n)}, S_{n},I^{(n)}\right)\right\}_{n=0}^{N}$,  where we set $I^{(0)} = 1$. 
The sequence $\{\mathbf{Y}^{(t)}\}_{t=0}^{T}$ is equivalent to the recorded dataset $\mathcal{D}$ (with probability 1) and we will use these interchangeably.

\subsection{Transition kernel, total variation, and subsampling}\label{subsec: tv burn-in thinning}

Before turning to the structure learning problem itself, we collect three pieces of notation that recur throughout the paper: the one-step transition kernel of the discrete-time chain, the total variation distance on probability measures, and the burn-in/thinning parameters that govern subsampling of Glauber dynamics.

\textbf{Transition kernel.} Let $K$ denote the one-step transition kernel of the discrete-time chain $\{\mathbf{X}^{(n)}\}_{n\in\mathbb{N}}$, and $K^{n}$ its $n$-step iterate; that is, $K^{n}(x,\cdot)$ is the law of $\mathbf{X}^{(n)}$ started from $\mathbf{X}^{(0)}=x$.

\textbf{Total variation.} For probability measures $\mu,\nu$ on $\mathbb{R}^{p}$, the total variation distance is
\begin{align}
\norm{\mu-\nu} := \sup_{A \text{ measurable}}\abs{\mu(A) - \nu(A)} = \frac{1}{2}\sup_{\substack{\norm{f}_{\infty}\leq 1 \\ f \text{ measurable}}}\abs{\mu f - \nu f}.
\label{eq:tv-norm}
\end{align}

\textbf{Burn-in and thinning.} For burn-in parameter $\mathfrak{b}\ge 0$ and thinning parameter $\mathfrak{t}\ge 1$, the retained sequence is
\begin{align}
\mathbf{Y}^{(s)} := \mathbf{X}^{(\mathfrak{b}+s\mathfrak{t})}, \qquad s=0,1,\ldots,m-1,
\label{eq:burned-thinned-samples}
\end{align}
with $\mathfrak{b}+(m-1)\mathfrak{t}\le N$. We write $\mathbf{Y}_{0}^{m-1}:=(\mathbf{Y}^{(0)},\ldots,\mathbf{Y}^{(m-1)})$ for the retained subsequence and $\mathcal{L}(\mathbf{Y}_{0}^{m-1}\mid \mathbf{X}^{(0)}=x)$ for its joint law. The i.i.d.\ reference sequence $\mathbf{Z}_{0}^{m-1}\sim\pi^{\otimes m}$ consists of $m$ independent draws from the stationary target $\pi=\mathcal{N}_p(0,\boldsymbol{\Theta}^{-1})$. This subsampling is the basis for the algorithm developed in Section~\ref{sec:algorithm-2-rewrite}.

\subsection{The Structure Learning Problem}\label{subsec: the structure learning problem}
Our goal is to learn the edge set $E$ of the graph underlying the Gaussian graphical model $\left(G, f_{\mathbf{X}}\right)$, given access to observations $\mathcal{D} = \left\{ \left( \mathbf{X}^{(n)}, S_n, I^{(n)} \right) \right\}_{n=0}^{N} \in \left(\mathbb{R}^p \times \mathbb{R} \times V\right)^{N}$
from the corresponding Glauber dynamics. We define a structure learning algorithm as a map $\phi:\left(\mathbb{R}^p \times \mathbb{R} \times V\right)^{Tp} \rightarrow \binom{V}{2}$
that accepts a dataset and produces an estimate of the graph's edge structure. The goal is to understand how long the Glauber dynamics need to be observed to confidently recover the graph as a function of various parameters of the problem. Toward this, for any $\delta\in (0,1)$, we define $T_\ast(\phi, \delta)$ as the smallest observation time with which the algorithm $\phi$ can be guaranteed to recover the graph structure with probability at least $1 - \delta$. That is,
\begin{align}
T_\ast (\phi, \delta) &= \inf \{T > 0: \mathbb{P}\left[ \phi(\mathcal{D}) = E \right]\geq 1 - \delta\}.
\end{align}
In Section~\ref{sec: the algorithm}, we introduce an algorithm $\phi$ for which Theorem~\ref{thm: final theorem} gives an upper bound on $T_\ast(\phi, \delta)$. In  Section~\ref{sec: lower bounds}, we provide lower bounds on $\min_\phi T_\ast(\phi, \delta)$ for a natural and broad family of problems, thereby showing that the algorithm we propose is nearly minimax optimal for this class of problems.

\section{Related Work}\label{sec: related work}

{\bf Graphical Model Selection from i.i.d Data. }
There is a large literature on graphical model selection from i.i.d.\
samples, including classical tree and polytree methods~\citep{chow1968tree,
chow1973polytrees}, penalized likelihood and neighborhood-selection estimators
for Gaussian models~\citep{banerjee2007model,Yuan2007Gaussian,
Meinshausen2006Neighborhood}, and information-theoretically optimal procedures
such as DICE~\citep{misra2020information}. We refer the reader
to~\citep{drton2016structure,zhou2011structurelearningprobabilisticgraphical}
for recent and broader surveys of this literature.

{\bf Graphical Model Selection from Dynamics. }Recently, there has been growing interest in learning graphical models from dynamical data generated by local Markov chains, particularly Glauber dynamics. The work of~\citep{bresler2014learning} pioneered this area, focusing on learning Ising models from Glauber dynamics.~\citep{dutt2021exponentialreductionsamplecomplexity} further improved upon this by enhancing node-wise regression algorithms, such as RISE, which were originally designed for the i.i.d. settings. On the other hand,~\citep{gaitonde2023unifiedapproachlearningising} generalized this problem to a broader class of local Markov chains, while~\citep{gaitonde2024efficientlylearningmarkovrandom} extended the analysis to MRFs with higher-order interactions, beyond binary pairwise models. All these works focus on binary or, at most finite-sized alphabets for the random variables in their models. In contrast, we tackle the problem of structure learning from the well-known dynamical process of Glauber dynamics with Gaussian random variables that take values in $\mathbb{R}$.
The fact that these random variables are unbounded introduces unique challenges that make structure learning difficult. We discuss these challenges in detail in the following section and provide intuition for how \algone{} addresses these challenges.

{\bf Mixing-Time Bounds for High-Dimensional Gibbs Samplers. }A second body of literature, relevant to \algtwo{}, comes from the Markov chain Monte Carlo (MCMC) community. There, the canonical recipe for extracting approximately i.i.d.\ samples from a Markov chain is \emph{burn-in and thinning}: an initial segment is discarded to remove initialization bias, and thereafter only every $\mathfrak{t}$-th sample is retained to reduce temporal correlation~\citep{jones2004sufficient, jones2001honest, maceachern1994subsampling, link2012thinning, gamerman2006markov, owen2017statistically, riabiz2022optimal}. Quantifying how aggressively one can burn in and thin while keeping the retained subsample close in joint law to a product of stationary draws ultimately reduces to sharp upper bounds on the total variation (TV) distance between the chain at time $t$ and its stationary distribution.

Drift-and-minorization is a standard tool for quantitative TV mixing
analysis~\citep{meyn1994computable,rosenthal1995minorization,jones2001honest},
but the resulting bounds can be loose in high dimensions. This limitation is
formalized through convergence-complexity analyses~\citep{rajaratnam2015mcmc}
and through recent exposure of limitations in high-dimensional Markov
chains~\citep{qin2022wasserstein}.

Alternative routes carry their own restrictions in high dimension. Conductance- and isoperimetry-based methods have recently produced sharp entropy contraction rates for the Gibbs sampler under log-concavity~\citep{ascolani2024entropy}, but the resulting bounds presume a warm start and degrade under arbitrary initialization; explicit mixing-time estimates for log-smooth log-concave targets in~\citet{wadia2024gibbs} depend on the condition number. For Gaussian targets specifically, spectral-gap analyses going back to~\citet{AMIT199182} and continued in~\citet{roberts1997updating, liu1994covariance} give convergence rates in Wasserstein distance but depend on the condition number $\kappa(\Theta)$, which we do not wish to assume controls on.

The most useful high-dimensional estimates for the random-scan Gaussian Gibbs sampler instead come from the transport side. Under a Dobrushin-style contraction condition on the conditional regressions,~\citet{wang2014convergence} obtain a Wasserstein contraction for deterministic-scan Gibbs samplers, and~\citet{wang2017convergence} establishes the analogous Wasserstein convergence rate for the random-scan Gibbs sampler, producing a high-dimensional Wasserstein convergence rate without spectral or condition-number assumptions. What is missing from this line of work, and what \algtwo{} supplies, is a conversion from these Wasserstein bounds into TV bounds usable inside the burn-in/thinning reduction; this is the gap closed in Section~\ref{sec:algorithm-2-rewrite}.

{\bf Gaussian GGMs from Glauber dynamics. }
Most closely related to the present work is Shen, Wu, Majid, and Moitra~\cite{SWMM26}, which studies Gaussian graphical model selection from a
single Glauber trajectory without waiting for the chain to mix. They give a
polynomial-time algorithm under sparsity and normalized edge-strength
assumptions, and their local estimator uses \(i,i,j,i\)-type update windows
together with robust aggregation of ratio statistics. Their paper also explains
the dependence obstruction in the \(i,j,i\)-based ratio estimator used in the
preliminary version of this work. A key strength of their result is that
their no-mixing guarantee does not require a Dobrushin or fast-mixing condition

Our local method, which was developed independently of~\cite{SWMM26}, addresses the same
obstruction using a product-moment statistic rather than a ratio statistic,
which leads to a different expectation calculation and martingale concentration
analysis. Our results are framed under the Dobrushin condition holding, and this allows us to devise a second method, which is orthogonal: under contraction, we prove
a total-variation decoupling bound for burned-in and thinned Gaussian Gibbs
trajectories, yielding a black-box reduction to i.i.d. Gaussian graphical model
selection.

\section{\algone{}: \algonefull{}}\label{sec: algone}

\subsection{Intuition behing the Algorithm}\label{sec: intuition}

We introduce \algone{} (\algonefull{}), which estimates the edge set of the graph underlying the Gaussian graphical model directly from observations of Glauber dynamics. Our strategy is to construct a hypothesis test for the existence of each potential edge $\{i,j\} \in \binom{V}{2}$. In particular, we will create an estimate of the edge strength $\beta_{ij}$, and declare that an edge exists in $G$ if this estimate is above an appropriate threshold. We will next provide some intuition on how we construct this test.

Consider any two nodes, $\{i,j\} \in {V \choose 2}$. Imagine there is a sequence of updates where these two nodes take turns being updated. First, node $i$ gets updated twice at separate times, then node $j$, and then node $i$ again as shown in Fig~\ref{fig: updates}. Further, let us suppose that during this sequence of updates, none of the neighbors of node $i$ (other than possibly node $j$) are updated.
Formally, assume there exist integers $0 \leq n_0 \leq n_1 < n_2 < n_3 < n_4 \leq n_{5}$ such that: (a) $I^{(n)} \notin N(i) \setminus \{j\}$ for all $n$ in the range $n_0 \leq n \leq n_5$, (b) $I^{(n)} = i$ for $n = n_1, n_2, n_4$, and $I^{(n)} = j$ for $n = n_3$.
 
\begin{figure}[!t]
\centering
\scalebox{0.85}{
\begin{tikzpicture}
    \draw[dashed, thick] (-1.5, 1.7) rectangle (11.5, -2.5);
\draw[thick, ->] (-1, -1) -- (10.5, -1) node[anchor=north west, xshift=-20pt, yshift=28pt, align = center] {update \\ number};
\node at (0, -1.3) {$n_0$};
\node at (2, -1.3) {$n_1$};
\node at (4, -1.3) {$n_2$};
\node at (6, -1.3) {$n_3$};
\node at (8, -1.3) {$n_4$};
\node at (10, -1.3) {$n_5$};
\draw[->, thick] (2, -1.1) -- (2, 0.5) node[anchor=south] {\fcolorbox{black}{blue!20}{$i$ updated}};
\draw[->, thick] (4, -1.1) -- (4, 0.5) node[anchor=south] {\fcolorbox{black}{blue!20}{$i$ updated}};
\draw[->, thick] (6, -1.1) -- (6, 0.5) node[anchor=south] {\fcolorbox{black}{red!20}{$j$ updated}};
\draw[->, thick] (8, -1.1) -- (8, 0.5) node[anchor=south] {\fcolorbox{black}{blue!20}{$i$ updated}};
\draw[decoration={brace,mirror,raise=5pt},decorate] (0, -1.5) -- (10, -1.5)
node[pos=0.5,anchor=north,yshift=-7pt]{no updates in $N(i) \setminus \{j\}$};
\end{tikzpicture}
}
\caption{\footnotesize Visualization of the idealized update sequence for vertices $i$ and $j$. At times $n_1$, $n_2$, and $n_4$, node $i$ is updated, while node $j$ is updated at time $n_3$. No updates occur for any other neighbors of $i$ in the interval between $n_0$ and $n_5$, except possibly $j$.}
\label{fig: updates}
\end{figure}

Under this idealized scenario, given that no neighbor of node $i$ was updated between $n_0$ and $n_5$, from \ref{eqn: glauber ggm updates}, we expect the change in the value of $X_{i}$ between $n_{1}$ and $n_{5}$ to be solely due to the update of node $j$ at $n_{3}$ (if $j$ is indeed a neighbor of $i$), aside from zero-mean noise terms. This change must depend on (a) the change in $X_{j}$, and (b) the influence of $j$ on $i$; the latter, of course, is captured by $\beta_{ij}$. The following lemma formalizes this relationship and forms the basis for our edge test. 

\begin{lemma}\label{lemma: intuition}
Under the idealized conditions on $n_{0},n_{1},n_{2},n_{3},n_{4}, n_5$ described above, for any $\bar{x}\in \mathbb{R}^{\abs{N(i)\setminus j}}$ we have,
\begin{align}\label{eqn: intuition}
	\abs{\E[]{\left(X_{i}^{(n_{4})}-X_{i}^{(n_{1})}\right)\left( X_{j}^{(n_{3})}-X_{j}^{(n_{1})} \right) \middle |  \mathbf{X}^{(n_{0})}_{N(i)\setminus j}=\bar{x}}} \geq \abs{\beta_{ij}}\theta_{jj}^{-1}.
\end{align}
\end{lemma}

Note  when $\{i,j\} \not\in E$ then the expectation is exactly zero. This suggests a strategy to estimate the magnitude of $\beta_{ij}$: find multiple update intervals that satisfy the above condition, and compute the average of the product of the changes in between $X_i$ and $X_j$ values. This is the basis for the construction of our test statistic (see Definition~\ref{def: Test statistic}). 

However, this naive strategy has some limitations. First, while we can tell whether $i$ and $j$ alternate in updates, there is no way to guarantee that the neighbors of $i$ remain unaffected during the interval in question (since we do not know the graph structure!). To address this, we show how one can choose the length of the update interval appropriately to make sure that the required sequence of updates for nodes $i$ and $j$ is observed frequently enough, while making sure the neighbors of node $i$ are not updated too often; see Lemmas~\ref{lemma: Probability of event U}-\ref{lemma: Probability of event Q}.

The idea behind our test and this fix is inspired by~\citep{bresler2014learning}. There are however several challenges specific to the Gaussian case that do not occur in~\citep{bresler2014learning} for the Ising model case. First, the form of the test statistic is rather different and demands a different analysis technique. That is, the proposed test statistic considers not merely the probability of sign flips, but instead the  product of the change in the values of samples. Second, this test statistic could be uninformative (even in expectation) if the updates are too large; this is not a concern in the Ising model case, where all random variables are bounded. Our algorithm and analysis carefully deal with these while also ensuring that the test statistic retains enough power to recover the underlying graph.

The asymmetry of the \(i,i,j,i\)-type window is essential. A preliminary version~\cite{TRD25}
of this work used an \(i,j,i\)-type update pattern and a ratio of coordinate
increments. That statistic is not valid in general: when \(\{i,j\}\in E\), the
middle \(j\)-update can depend on the noise injected during the preceding
\(i\)-update, so the denominator of the ratio is not conditionally independent
of the numerator noise. The additional update of coordinate \(i\) acts as a
buffer. After this second \(i\)-update, the subsequent change in \(i\) can be
related to the intervening change in \(j\) without carrying the problematic
dependence from the first \(i\)-update. We therefore aggregate products of
increments rather than ratios. The product form is less direct as an estimator
of \(\beta_{ij}\), but it has the conditional moment structure needed for a
valid edge test.
 
\subsection{The \algone{} Algorithm}\label{sec: the algorithm}

From the discussion above, our strategy is to construct a test statistic to detect edges in $G$ by estimating the conditional expectation in Lemma~\ref{lemma: intuition}. We know this quantity is only useful when the conditions described in Figure~\ref{fig: updates} occur. Towards constructing an estimator, we will suppose that we choose a small $\tau >0$, and divide the length $T$ of the Glauber dynamics process into $k_{\max} \triangleq \floor{\frac{T}{\tau}}$ intervals of length $\tau$. We will next show that with a careful choice of $\tau$, we can enforce that our desired events occur frequently enough to get a usable test statistic.

First, for any pair of vertices $\{i,j\} \in [p]$ and $k\in [k_{\max}]$, we define the update event $U_{ij}^k$, where $U$ stands for update, which records whether the updates of the pair $\{i,j\}$ occur in the required order in the $k$-th time interval. We write $I_{t}$ to represent the identity of the last update node on or before time $t$. 
\begin{definition}[Update event $U_{ij}^k$]\label{def: Event U} For $k \in [k_{\max}]$,
     let the interval $[(k-1)\tau,k\tau)$ be divided into fourths as below
     \begin{align*}
         &W_{1} = [(k-1)\tau,(k-3/4)\tau), \\
         &W_{2} = [(k-3/4)\tau,(k-1/2)\tau), \\
         &W_{3} = [(k-1/2)\tau,(k-1/4)\tau), 
         \text{ and } 
         W_{4} = [(k-1/4)\tau,k\tau),
    \end{align*}
    For $i,j \in [p]$, we define
    \begin{align*}
        U_{ij}^{k} =
        \bigg\{&\exists t \in W_{1} : I_{t}= i,  \forall t \in W_{1} : I_{t} \neq j \bigg\}  \bigcap
        \bigg\{\exists t \in W_{2}:I_{t}  = i , \forall t \in W_{2}:I_{t} \neq j  \bigg\}  \\ &\bigcap
        \bigg\{\exists t \in W_{3}:I_{t}  = j , \forall t \in W_{2}:I_{t} \neq i  \bigg\} \\ &\bigcap
        \bigg\{\exists t \in W_{4}: I_{t} = i, \forall t \in W_{3}: I_{t} \neq j 
        \bigg\}.
    \end{align*}
\end{definition} 

Whenever this event occurs, we update our test statistic with the product of changes as suggested by \ref{eqn: intuition}. We next show that the event occurs with a non-vanishing probability, and therefore we know that its impact on the test statistic is controllable by the choice of    $\tau$.
\begin{lemma}[Probability of update event $U_{ij}^k$]\label{lemma: Probability of event U}
    For $k \in [k_{\max}]$, $i,j \in [p]$ we have
    \begin{align*}
        \prob[]{U_{ij}^{k}} = \left[(1-e^{-\tau/4})e^{-\tau/4}\right]^{4} \triangleq q.
    \end{align*}
\end{lemma}
Next, we define $Q_{ij}^{k}$, the event where the rest of the neighborhood is quiet, that is, no element in $N(i)\setminus\{j\}$ is updated in the interval $[(k-1)\tau, k\tau)$.
\begin{definition}[Quiet neighborhood event $Q_{ij}^{k}$]
    For any $k \in [k_{\max}]$ and $\{i,j\}$ pair we define
    \begin{align*}
        Q_{ij}^{k} = \left\{I_{t} \notin N(i)\setminus \{j\}, \forall (k-1)\tau \leq t < k\tau\right\}.
    \end{align*}
\end{definition}
As we noted earlier, we cannot observe whether this event occurs. However, in the following lemma, we show that we can control how often this event occurs by controlling the parameter $\tau$. The smaller $\tau$ is, the more likely it is that $Q_{ij}^k$ occurs. On the other hand, the smaller $\tau$ is, the less likely it is that the update event $U_{ij}^k$ occurs. As part of our analysis, we will identify a value for $\tau$ that balances this tension.

\begin{lemma} \label{lemma: Probability of event Q}The probability of event $Q_{ij}^{k}$ for any $k \in [k_{\max}]$ satisfies $\prob[]{Q_{ij}^{k}} \geq e^{-\tau d}.$
\end{lemma}

Following the discussion in Lemma~\ref{lemma: intuition}, our test statistic will be computed as the product of changes in the values of coordinates when the appropriate $U_{ij}^{k}$ event occurs. There is, however, a catch. Given that the $Q_{ij}^{k}$ events are not observed, the expected value of the test statistic could depart from its idealized value (i.e., the edge strength) since we cannot condition on the appropriate $Q_{ij}^{k}$. While we can make the probability of $Q_{ij}^{k}$ not occurring as small as we want, if the test statistic is unbounded, this can essentially render the test powerless; that is, the test may not be able to detect the presence or absence of edges even in expectation. Note that~\citep{bresler2014learning},~\citep{dutt2021exponentialreductionsamplecomplexity}, and~\citep{gaitonde2024efficientlylearningmarkovrandom} do not have to deal with this as the Ising model is bounded and nicely behaved. To account for these eventualities, we next define two events and establish controls on their probabilities.

\begin{lemma}[Event B]\label{lemma: event B}
    For any $\delta > 0$ and sufficiently large $p$, there exist constants $C_{1}, C_{1}', C_{1}'' > 0$ such that 
    \begin{align*}
        \prob[]{\max_{i\in [p], t < C_1' p^{C_1''}} \left| Y_i^{(t)} \right| \leq C_1 \sigma_{\max} \sqrt{\log \left( \frac{p}{\delta} \right)}} \geq 1 - \delta/2.
    \end{align*}
\end{lemma}

For a fixed $\delta$, we will define $\ymax \triangleq C_1 \sigma_{\max} \sqrt{\log \left( \frac{p}{\delta} \right)},$ and $B_{\delta} \triangleq \left\{\max_{i\in [p], t < C_1' p^{C_1''}} \left| Y_i^{(t)} \right| \leq \ymax \right\},$
and define the appropriate conditional measure and expectation as $\mathbb{P}_{B}[\cdot] = \prob[]{\cdot|B_{\delta}}$ and $\E[B]{\cdot} = \E[]{\cdot | B_{\delta}}$. Letting $T_{\max} \triangleq C_{1}^{'}p^{C_{1}''}$, the bound above ensures that the observed trajectories of the Glauber dynamics are bounded for all time $t< T_{\max}$. As we prove in Theorem~\ref{thm: final theorem}, we only need to observe the process for $\Omega\left(\frac{d^3\operatorname{polylog}p}{\bmin^5}\right)$ time, and for appropriately chosen constants, the event $B_{\delta}$ holds for significantly longer. This forms one part of our strategy to control the test statistic in light of the possible unboundedness discussed above.

We can now define our test statistic based on Lemma~\ref{lemma: intuition} and the event we have defined and Lemma~\ref{lemma: Test statistic bounded as} shows that our test statistic is bounded almost surely.

\begin{definition}[Test statistic]\label{def: Test statistic}
    For any $i,j \in [p]$ we define 
    \begin{align*}       
    T_{ij} = \sum_{k=1}^{k_{\max}}\mathds{1}_{\{U_{ij}^{k}\}}\Delta Y_{i}^{k} \cdot \Delta Y_{j}^{k}.
    \end{align*}
    Here $\Delta Y_{i}^{k} = Y_{i}^{k\tau -} - Y_{i}^{(k-3/4)\tau -}$ and $\Delta Y_{j}^{k} = Y_{j}^{(k-1/4)\tau} - Y_{j}^{(k-1)\tau}$. We will let $T_{ij}^{(k)}$ denote the $k$-th term in the above summation.
\end{definition}

The event $B_{\delta}$ allows us to control the behavior of the test statistic
by bounding each summand of the statistic, as stated in
Lemma~\ref{lemma: Test statistic bounded as} below, and by bounding the
conditional variance in Lemma~\ref{appseclemma: variance bound} of
Appendix~\ref{sec: proofs of lemmas}. These two bounds address the main
unboundedness issues raised in Section~\ref{sec: intuition}.

\begin{lemma}\label{lemma: Test statistic bounded as}
    For $k \in [k_{\max}]$ and $i,j \in [p]$ the $k$-th term of the test statistic is bounded almost surely with respect to the conditional measure $\mathbb{P}_{B}$. That is, 
    \begin{align*}
        \abs{T_{ij}^{k}} \leq 4C_1^{2} \sigma_{\max}^{2} \log \left( \frac{p}{\delta} \right) = 4\ymax^{2} \quad(\text{a.s.}).
    \end{align*}
\end{lemma}

We now present our structure learning algorithm. The algorithm takes as parameters the dataset $\mathcal{D}$, the interval length $\tau$, a threshold $\rho$, and a margin of error $\delta$. First, it calculates $k_{\max}$, the number of $\tau$-length intervals in the dataset, and checks if the high-probability event $B_{\delta}$ occurs in the dataset. If not, it terminates the procedure and returns an empty edgeset (error). Otherwise, the algorithm iterates over all pairs of nodes $\{i, j\} \in \binom{V}{2}$, and, for each pair, computes the test statistic $T_{ij}$. If this test statistic over $k_{\max}$ intervals exceeds the threshold $\rho$, we declare an edge between nodes $i$ and $j$. 

\begin{algorithm}[H]
\caption{\algone{}}
\label{alg:cap}
    \begin{algorithmic}[1]
        \State {\bf Input:} $\mathcal{D}$ (dataset), $\tau$ (interval length), $\rho$ (threshold), $\delta$ (confidence parameter)
        \State Let $\hat{E} = \emptyset$ and $k_{\max} = \floor{\frac{T}
        {\tau}}$
        \State {\bf if} {$\mathds{1}_{B_{\delta}} = 1$}
             \State \qquad {\bf for every pair}\; {$\{i,j\}, i \neq j$:}
                \State \qquad\qquad {\bf compute} {$T_{ij} = \left|\frac{1}{k_{\max}} \sum_{k=1}^{k_{\max}} T_{ij}^{(k)}\right|$}
                \State \qquad\qquad {\bf if} {$T_{ij}\geq \rho$}, add $\{i,j\}$ to $\hat{E}$.
        \State \textbf{return} $\hat{E}$
    \end{algorithmic}
\end{algorithm}

\subsection{Performance guarantees}\label{sec: theorem}

We state the standing assumptions used by both \algone{} and \algtwo{}, and then give the performance guarantee for \algone{}.

We suppose that the edge strengths ($\beta_{ij}$ for some nodes $i,j$) are bounded away from zero and the edges are statistically distinguishable from non-edges.
\begin{assumption}[Bounded edge strength]\label{assumption: Bounded edge strength}
For all pair of edges $\{i,j\} \in E$, the partial regression coefficients $\beta_{ij}$ are bounded as \(\bmin \leq \left|\beta_{ij}\right|\leq \bmax\).
\end{assumption}
Further, we assume that no node has zero or arbitrarily large variance. In precision form, this is equivalent to a uniform bound on the diagonal of $\Theta$.
\begin{assumption}[Bounded variance]\label{assumption: Bounded variance}
For all $i \in [p]$, the marginal variance is bounded above by $\sigmamax^{2}$ and the conditional variance of $X_{i}$ given $X_{V \setminus \{i\}}$ is bounded below by $\sigmamin^{2}$. Equivalently, the diagonal of the precision matrix satisfies $\Theta_{\min}\le\Theta_{ii}\le\Theta_{\max}$ for all $i\in[p]$.
\end{assumption}

In many real-world processes, such as opinion dynamics~\citep{degroot}, each node update is a weighted average of neighboring values, so large values should attenuate rather than amplify across the graph. The next assumption captures this contractive behavior. It is also a simple sufficient form of the Dobrushin contraction condition for Gibbs samplers~\citep{wang2014convergence}: for Gaussian conditional laws, the influence of coordinate $j$ on coordinate $i$ is $|\beta_{ij}|$, and the total influence on any coordinate is at most $d\bmax$.

\begin{assumption}[Bounded degree and contraction]\label{assumption: bounded degree and sample decay}
    The graph $G$ has maximum degree at most $d$ and satisfies $d\bmax < 1$. Equivalently, this implies the Dobrushin uniqueness condition
    \begin{align}
    r:=\max_{i\in[p]}\sum_{j\ne i}|\beta_{ij}|\ \le\ d\bmax\ <\ 1,
    \end{align}
    and we write $\alpha := 1-r$ for the contraction gap.
\end{assumption}

\begin{theorem}\label{thm: final theorem}
    Consider observations from Glauber dynamics $\{\mathbf{Y}^{(t)}\}_{t=0}^{T}$ associated with a Gaussian graphical model $(G,f_{\mathbf{X}})$. Under the assumptions ~\ref{assumption: Bounded edge strength}-\ref{assumption: bounded degree and sample decay}, for any $\delta \in (0,1/2)$  and sufficiently large $p$ there exists constants $C_{2}, C_{3}$ and $C_{4}$ such that if the observation time $T$ satisfies ($T_{\max}$ is as in Lemma~\ref{lemma: High probability event B})
    \begin{align}
        T_{\max} \geq T \geq  \frac{C_{2}}{(\bmin\sigmamin)^5}d^{3}\log\left(\frac{p}{\delta}\right)^{5},  
    \end{align}
    
    and the interval length $\tau$ and threshold $\rho$ are set as 
    \begin{align*}
        \tau = d^{-1}\left[\frac{C_{3}{\log\left(\frac{p}{\delta}\right)}}{\sigmamin\bmin} + 1\right]^{-1}, \qquad
        \rho = C_{4}\tau {\log\left(\frac{p}{\delta}\right)},
    \end{align*}
    then \algone{} recovers the true edge set with probability at least $1-\delta$.
\end{theorem}

{\bf Remark 1.} Our algorithm is guaranteed to recover the true edge set given $\mathcal{O}(\frac{d^{3}\operatorname{polylog} p}{\bmin^{5}})$ Glauber updates per node. Notice that this quantity is an upper bound on $T_\ast$ defined in Section~\ref{sec: problem statement}. The f on $\beta_{\rm min}^4$ is possibly improvable with a different algorithmic choice and is an interesting avenue for future work.

{\bf Remark 2.} The computational complexity of \algone{} is  $\mathcal{O}\left(d^3p^3\operatorname{polylog}\bmin^{-5}\right)$ as the algorithm iterates over ${p \choose 2}$ pairs of vertices. It is worth noting that in the high-dimensional regime, \algone{} recovers the edge set in a significantly shorter observation time than it takes Glauber dynamics to mix which is $\Omega(p)$.

{\bf Remark 3 (Dobrushin generalization).} Theorem~\ref{thm: final theorem} is stated under the sufficient form $d\bmax<1$ of Assumption~\ref{assumption: bounded degree and sample decay}. The conclusion extends to the strictly weaker Dobrushin uniqueness condition $r<1$ with $\alpha=1-r$ in place of $1-d\bmax$, at the cost of tracking $r$ rather than $d\bmax$ inside the separation and concentration arguments.

{\bf Remark 4 (Comparison with lower bound).} The minimax lower bound of Theorem~\ref{thm:lower-bound-glauber} in Section~\ref{sec: lower bounds} shows that any consistent estimator requires $T = \Omega\bigl(\log(p-d)/\bmin^{2}\bigr)$ Glauber updates per node. In the regime $\bmin = \mathcal{O}(1/d)$, this shows that \algone{} matches the information-theoretic lower bound up to polynomial factors in $d$ and $\log p$. 

We provide only a brief proof strategy here. The full proof of
Theorem~\ref{thm: final theorem} appears in
Appendix~\ref{appsubsec: theorem 1 statement and proof}, and the supporting
lemma statements and proofs are collected in
Appendix~\ref{sec: proofs of lemmas}.
 
{\em Proof Strategy. } To bound the probability of error of \algone{}, we first condition on the event that $B_\delta$ happens (which we choose to be a high-probability event). Then, for each pair of vertices $\{i, j\}$, we first establish (in Lemma~\ref{appseclemma: separation2}) that the test statistic proposed is able to distinguish, in expectation, between the cases where $\{i,j\}$ are connected and where they are not; this requires a careful selection of the interval $\tau$ and allows us to pick an appropriate threshold $\rho$. Finally, we establish the finite data performance of \algone{} using a martingale concentration argument that leverages our analysis involving the events $U_{ij}$. It is worth noting that without carefully accounting for the above events, the underlying test statistics is not a (sub)martingale and is not well-behaved.

\section{\algtwo{}: \algtwofull{}}\label{sec: algtwo}\label{sec:algorithm-2-rewrite}

\subsection{Intuition}\label{subsec: algtwo intuition}

The preceding section treats the Glauber trajectory itself as the object of
analysis, using local update windows to test edges one at a time. In this
section we ask a different question: can the dependent trajectory be converted
into an input suitable for an existing i.i.d.\ Gaussian graphical model learner?
The resulting method, \algtwo{} (\algtwofull{}), is less local than \algone{}
and relies on a contraction condition, but it lets us inherit the sharper
i.i.d.\ sample-complexity guarantees available for procedures such as
DICE~\citep{misra2020information}. The main task is to prove a quantitative decoupling theorem for the Gaussian Gibbs trajectory.

A canonical strategy is to burn-in and thin the chain to achieve such a decoupling.
Starting from an arbitrary initial condition, we discard the first
$\mathfrak{b}$ iterates and then retain only states separated by
$\mathfrak{t}$ Glauber updates. The burn-in period controls initialization
bias, while the thinning interval controls temporal dependence among the
retained states. If these two parameters are chosen appropriately, the retained
sequence should be close to independent draws from the stationary Gaussian law
$\pi=\mathcal{N}_{p}(0,\Theta^{-1})$, and can be supplied to an i.i.d.\
structure learner. Writing
$\mathbf{Y}_{0}^{m-1}$ for the retained sequence, the quantity we need to
control is
\[
    \left\|
        \mathcal L(\mathbf{Y}_{0}^{m-1}\mid \mathbf{X}^{(0)}=x)
        - \pi^{\otimes m}
    \right\|_{\mathrm{TV}}.
\]
This total-variation error is exactly the price paid for replacing the retained
trajectory by ideal i.i.d.\ samples: if it is at most $\varepsilon$, then the
failure probability of any downstream learner increases by at most
$\varepsilon$. Thus the central technical task is to choose
$\mathfrak{b}$ and $\mathfrak{t}$ so that this joint decoupling error is small.

As surveyed in Section~\ref{sec: related work}, existing tools for bounding TV mixing of high-dimensional Gibbs samplers are either quantitatively loose (drift-and-minorization), or require restrictive side conditions (warm starts, bounded condition number of $\Theta$). The most useful high-dimensional estimates for the random-scan Gaussian Gibbs sampler instead come from the transport side: Dobrushin-style contraction arguments give Wasserstein convergence rates without spectral or condition-number assumptions~\citep{wang2017convergence}. Converting these Wasserstein bounds into TV, however, is not automatic. Several off-the-shelf conversions are available --- the discrete-metric identification of $W_{1}$ with TV, one-shot coupling~\citep{RobertsRosenthal2002}, and Kantorovich--Rubinstein smoothing of bounded test functions by a few kernel applications~\citep{Villani2009,MadrasSezer2010,GibbsSu2002} --- but each has hypotheses that fail for the random-scan Gaussian Gibbs sampler, as we discuss next.

Each of these routes has a specific obstruction in our setting. If one equips each coordinate with the discrete metric, then $W_{1}$ coincides with total variation, but the Dobrushin condition fails to hold and the usual smoothing argument breaks down. A one-shot coupling argument~\citep{RobertsRosenthal2002} controls TV by a single coupled step, but the available bounds require contracting Lipschitz constants that blow up on the unbounded state space of the Gaussian Gibbs sampler. A Kantorovich--Rubinstein conversion, finally, requires the smoothed kernel $K^{s}f$ to be globally Lipschitz for every bounded test function $f$~\citep{Villani2009,MadrasSezer2010,GibbsSu2002}; for random-scan Gibbs, at any fixed time there is positive probability that some coordinate has not yet been refreshed, so no global Lipschitz statement of this kind can hold. This is the conversion problem we address in the remainder of this section.

In what follows we establish a high-dimensional total variation convergence rate for the Gaussian random-scan Gibbs sampler by introducing a notion of \emph{thresholded} (or \emph{approximate}) Lipschitzness for the transition kernel: the smoothed kernel need not be globally Lipschitz, only Lipschitz on a high-probability set, with the failure event entering the bound as an explicit additive defect. Using this device, we convert a Wasserstein contraction from~\citet{wang2017convergence} into a TV bound; the resulting mixing estimate may be of independent interest to the MCMC community as a usable TV convergence rate for high-dimensional Gaussian Gibbs without spectral or condition-number assumptions. We then plug this mixing bound into the burn-in/thinning reduction described above, and obtain end-to-end performance guarantees for \algtwo{} in terms of the parameters of an i.i.d.\ base learner~$\mathcal{A}$. We instantiate $\mathcal{A}$ as DICE~\citep{misra2020information} to obtain an explicit sample complexity in $p$, $d$, $\kappa$, and $\delta$; the same reduction applies to any other i.i.d.\ Gaussian graphical model selection procedure.

\subsection{Algorithm}\label{subsection: algtwo}

The one-step transition kernel $K$, the total variation distance $\norm{\cdot}$, the burn-in/thinning parameters $(\mathfrak{b},\mathfrak{t})$, and the retained and i.i.d.\ reference sequences $\mathbf{Y}_{0}^{m-1},\mathbf{Z}_{0}^{m-1}$ were all introduced in Section~\ref{subsec: tv burn-in thinning}. Recall in particular that $\mathbf{Y}^{(s)}=\mathbf{X}^{(\mathfrak{b}+s\mathfrak{t})}$ for $s=0,\ldots,m-1$ as in~\eqref{eq:burned-thinned-samples}, and that $\mathbf{Z}_{0}^{m-1}\sim\pi^{\otimes m}$ is i.i.d.\ from $\pi=\mathcal{N}_p(0,\Theta^{-1})$.

We instantiate \algtwo{} with DICE~\citep{misra2020information} as the i.i.d.\ base learner; the construction below applies verbatim to any other i.i.d.\ Gaussian graphical model selection procedure (see the remark following Theorem~\ref{thm:algorithm-2-pac}). Let $\mathcal{A}:\mathbb{R}^{m}\to\mathcal{G}_{p,d}$ denote DICE, which takes $m$ i.i.d.\ samples and outputs a graph from the class $\mathcal{G}_{p,d}$ of graphs with $p$ vertices and maximum degree $d$. Our strategy is to subsample Glauber dynamics based on burn-in and thinning parameters $\mathfrak{b}, \mathfrak{t}$ to obtain approximately i.i.d.\ data and apply $\mathcal A$ to the result.

We define $\widetilde{\mathcal{A}}_{\mathfrak{b}, \mathfrak{t}}: \mathbb{R}^{n} \to \mathcal{G}_{p,d}$ such that
\begin{align}
    {\widetilde{\mathcal{A}}}_{{\mathfrak{b}, \mathfrak{t}}}(\mathbf{X}_{0}^{n-1}) = \mathcal{A}(\mathbf{Y}_{0}^{m-1}).
\end{align}
Thus \algtwo{} applies DICE to the burned-in and thinned samples.

We evaluate our algorithm with the zero-one loss $\mathds{1}[\widetilde{\mathcal{A}}(\mathbf{X}_{0}^{n-1}) \neq G]$. The probability of error is
\begin{align}
    \prob[\mathbf{X}_{0}^{n-1} \sim \mathcal{L}(\mathbf{X}_{0}^{n-1} \mid \mathbf{X}^{(0)} = x)]{\widetilde{\mathcal{A}}(\mathbf{X}_{0}^{n-1}) \neq G} = \prob[\mathbf{Y}_{0}^{m-1} \sim \mathcal{L}(\mathbf{Y}_{0}^{m-1} \mid \mathbf{X}^{(0)} = x)]{\mathcal{A}(\mathbf{Y}_{0}^{m-1}) \neq G}.
\end{align}

\subsection{Performance guarantee}\label{subsec: algtwo theorem}

The standing assumptions of Section~\ref{sec: theorem} (Assumptions~\ref{assumption: Bounded edge strength}--\ref{assumption: bounded degree and sample decay}) are used throughout. We additionally use the DICE i.i.d.\ sample complexity~\citep{misra2020information}: for $\mathbf{Z}_{0}^{m-1}\sim\pi^{\otimes m}$ and any $\delta\in(0,1)$, $m\ge C_{\rm DICE}\,d\log(p/\delta)/\kappa^2$ suffices for $\Pr\{\mathcal A(\mathbf{Z}_{0}^{m-1})\neq G\}\le\delta$, where $\kappa$ is the normalized edge-strength parameter and $C_{\rm DICE}$ is the universal constant from~\citep{misra2020information}.

\begin{theorem}[Sample complexity guarantee for \algtwo{}]
\label{thm:algorithm-2-pac}
Consider observations $\mathbf{X}_0^{n-1}$ from the Gaussian Glauber dynamics above,
started at $\mathbf{X}^{(0)}=x$.  Suppose Assumptions~\ref{assumption: Bounded edge strength}--\ref{assumption: bounded degree and sample decay} hold.  Then for every $\delta\in(0,1)$,
\algtwo{} satisfies
\begin{align}
\prob{\widetilde{\mathcal A}_{\mathfrak b,\mathfrak t}(\mathbf{X}_0^{n-1})\neq G}
\le \delta,
\label{eq:algorithm-2-pac-conclusion}
\end{align}
if the number of samples satisfies
\begin{align}
    n = \mathcal O \left(\frac{d\log(p/\delta)\,p}{\kappa^2(1-r)}\log\!\left(\frac{d\log(p/\delta)\,p^{3/2}\sqrt{\log p}}{\kappa^2\delta}\right)\right)
\label{eq:algorithm-2-update-complexity}
\end{align}
and we choose the burn-in and thinning lengths as
\begin{align}
    \mathfrak b,\mathfrak t
    \ge
    C_{x,\Theta,\mathcal{A}}\,\frac{p}{1-r}
    \log\!\left(\frac{d\log(p/\delta)\,p^{3/2}\sqrt{\log p}}{\kappa^2\delta}\right),
    \label{eq:algorithm-2-bt-choice}
\end{align}
for a constant $C_{x,\Theta,\mathcal{A}}$ depending on the initialization, model
parameters, and the universal constant of the base learner $\mathcal{A}$, but not on $p,d,\kappa$ or $\delta$ except through the displayed terms.

\end{theorem}

\begin{remark}
Equivalently, in normalized Glauber time, this corresponds to
\begin{align}
    T = \mathcal O \left(\frac{d\log(p/\delta)}{\kappa^2(1-r)}\log\left(\frac{d\log(p/\delta)\,p^{3/2}\sqrt{\log p}}{\kappa^2\delta}\right)\right)
\end{align}
time units.
\end{remark}

\begin{remark}[General base learner]\label{rem:algo2-general-learner}
Theorem~\ref{thm:algorithm-2-pac} is stated for DICE because it admits the explicit i.i.d.\ rate $m_{\rm DICE}(p,d,\delta)=C_{\rm DICE}\,d\log(p/\delta)/\kappa^{2}$. The same reduction applies to any i.i.d.\ base learner $\mathcal A$ with PAC sample complexity $m_{\mathcal A}(p,d,\delta)$: under the same standing assumptions, $\widetilde{\mathcal A}_{\mathfrak b,\mathfrak t}$ recovers $G$ with probability at least $1-\delta$ provided
\begin{align}
n=\mathcal O\!\left(\frac{m_{\mathcal A}(p,d,\delta)\,p}{1-r}\log\!\left(\frac{m_{\mathcal A}(p,d,\delta)\,p^{3/2}\sqrt{\log p}}{\delta}\right)\right),
\end{align}
with $\mathfrak b,\mathfrak t$ chosen of the same order as the inner logarithmic factor. Specialization to DICE recovers~\eqref{eq:algorithm-2-update-complexity}.
\end{remark}

\begin{remark}[Near-optimality]\label{rem:algo2-lower-bound}
The minimax lower bound of Theorem~\ref{thm:lower-bound-glauber} in Section~\ref{sec: lower bounds} shows that any consistent estimator requires $T=\Omega\bigl(\log(p-d)/\bmin^{2}\bigr)$ Glauber updates per node. In the regime $\kappa\asymp\bmin$ with $\bmin=\mathcal{O}(1/d)$, Theorem~\ref{thm:algorithm-2-pac} matches this lower bound up to a factor of $\poly(d,\log p,(1-r)^{-1})$, so \algtwo{} is also nearly minimax-optimal for this class.
\end{remark}

\textbf{Proof strategy for Theorem~\ref{thm:algorithm-2-pac}}
The full proof is given in Appendix~\ref{appsubsec:algorithm-2-proof}, with
the supporting mixing lemmas collected in
Appendix~\ref{subsec:supporting-lemmas}. At a high level, the proof first
compares the error of the burned-in/thinned procedure to the error of the same
base learner on an ideal i.i.d.\ sample, paying only the joint total-variation
distance between the retained trajectory and \(\pi^{\otimes m}\). It then
controls this joint distance by a hybrid argument, separating the contribution
of the initial burn-in from the accumulated dependence between consecutive
retained samples.

The remaining ingredient is a TV mixing estimate for the Gaussian random-scan
Gibbs sampler. Since the Gaussian transition kernel is not globally Lipschitz in
the sense needed for a direct Kantorovich--Rubinstein conversion, we use a
novel approximate-Lipschitz smoothing argument: after enough random scans, a
coupon-collector event ensures that all coordinates have been refreshed, and
the exceptional schedules contribute only a controlled additive error. Combining
this smoothing step with Dobrushin contraction in Wasserstein distance yields
the required TV bound. Choosing \(m\) from the DICE i.i.d.\ sample complexity
and choosing \(\mathfrak b,\mathfrak t\) from this mixing bound makes the
i.i.d.\ learning error and the decoupling error each at most \(\delta/2\). The
stated observation-length bound then follows from
\(n\ge \mathfrak b+(m-1)\mathfrak t+1\).

\section{Observation Time Lower Bounds}\label{sec: lower bounds}
In this section, we present the minimax lower bounds on the observation time $T$ required to learn the Gaussian graphical model from the Glauber dynamics $\{\mathbf{Y}^{(t)}\}_{t=0}^{T}$. We start by introducing a natural class of graphical models for which we can provide these lower bounds. We then define the minimax risk and state Theorem 2. We refer the reader to Section \ref{appsubsec: theorem 2 statement and proof} for detailed presentation of its proof.

\paragraph{Model Class. } Let $\mathcal{G}_{p,d}$ denote the set of undirected graphs with vertex set $[p]$ and maximum degree $d \leq p$. For a graph $G \in \mathcal{G}_{p,d}$, let $\Theta(G)$ denote any symmetric positive definite precision matrix whose graph is $G$, namely
  \begin{equation*}
  \Theta(G)_{ij} \neq 0 \quad \Longleftrightarrow \quad \{i,j\} \in E, \qquad i \neq j.
  \end{equation*}
  Let $\Sigma(G) := \Theta(G)^{-1}$ be the corresponding covariance matrix. With a slight abuse of notation, $\Theta(G)$ and $\Sigma(G)$ are not unique given
  $G$. We let $\mathcal{G}_{p,d}(\bmin,\theta_{\min},\theta_{\max})$ denote the set of pairs of probability distributions representing target and initial distributions
  $(\phi_{\Theta(G)},q)$ such that $G \in \mathcal{G}_{p,d}$ and $\Theta(G)$ is any precision matrix with graph $G$ satisfying: 
  \begin{itemize}
    \item For every $(\phi_{\Theta(G)}, q) \in \mathcal{G}_{p,d}(\bmin,\theta_{\min},\theta_{\max})$,
  \begin{equation*}
  \beta^*(\Theta)=\min_{\{i,j\}\in E}\frac{|\Theta_{ij}|}{\Theta_{ii}} \ge \bmin,
  \end{equation*}
  and the diagonal entries satisfy $\theta_{\min} \le \Theta_{ii} \le \theta_{\max}$ for all $i$.

\end{itemize}

We study a hard subclass $\mathcal{H}_{p,d}(\bmin,\theta_{\min},\theta_{\max}) \subset \mathcal{G}_{p,d}(\bmin,\theta_{\min},\theta_{\max})$ with initial distributions being the same as the target distribution (stationary start) and precision matrix of the target distribution additionally satisfying constant diagonal entry $\Theta_{ii} = \theta \in [\theta_{\min}, \theta_{\max}]$ for all $i \in V$.

\subsection{Graph Decoders and Maximum Probability of Error}

Let $\operatorname{GD}_{(n)}(\phi_{\Theta(G)})$ represent the joint probability distribution of observing $n$ vector samples from a Glauber dynamics process with target distribution $\phi_{\Theta(G)} \in \mathcal{H}_{p,d}(\beta_{\min}, \theta_{\min}, \theta_{\max})$ with a stationary start. Suppose we are given $n$ vector samples $\mathbf{X}_{0}^{n-1} = (\mathbf{X}^{(0)}, \mathbf{X}^{(1)}, \dots, \mathbf{X}^{(n-1)}) \in \mathbb{R}^{n \times p}$ from a stationary Glauber dynamics process $\operatorname{GD}_{(n)}(\phi_{\Theta(G)})$ with unknown targets. Our goal is to estimate the underlying graph $G \in \mathcal{G}_{p,d}$ associated with $\Theta(G)$ based on the observations $\mathbf{X}_{0}^{n-1}$.

To prove our minimax lower bound we consider decoders $\psi: \mathbb{R}^{n \times p} \to \mathcal{G}_{p,d}$ that map the observations $\mathbf{X}_{0}^{n-1}$ to an estimate $\widehat{G} = \psi(\mathbf{X}_{0}^{n-1})$. Any information about the update sequence that is identifiable from the observed trajectory
  $\mathbf{X}_0^{n-1}$ may be used implicitly by the decoder. In particular, under Gaussian Glauber
  dynamics, each update resamples exactly one coordinate from a non-degenerate conditional
  Gaussian distribution. Hence, with probability one, $\mathbf{X}^{(t)}$ and $\mathbf{X}^{(t-1)}$ differ in exactly
  one coordinate, so the identity of the updated node at time $t$ is determined by the observed
  pair $(\mathbf{X}^{(t-1)},\mathbf{X}^{(t)})$. The performance is evaluated using the 0-1 loss function $\mathds{1}[\psi(\mathbf{X}_{0}^{n-1}) \neq G]$, indicating whether the estimated graph differs from the true graph. For a decoder, the maximal expected risk over all admissible target distributions $\phi_{\Theta(G)}$ in the class $\mathcal{H}_{p,d}(\bmin, \theta_{\min}, \theta_{\max})$ is given by,
\begin{align}\label{eqn: maximal probability of error preliminary}
    \mathcal{R}_{n,p,d}(\psi; \mathcal{H}_{p,d}(\bmin, \theta_{\min}, \theta_{\max})) &:= \max_{\phi_{\Theta(G)}} \E[\Theta(G)]{\mathds{1}[\psi(\mathbf{X}_{0}^{n-1}) \neq G]} \nonumber \\
    &= \max_{\phi_{\Theta(G)}} \prob[\Theta(G)]{\psi(\mathbf{X}_{0}^{n-1}) \neq G}
\end{align}

The expectation is taken with respect to the joint probability measure $\prob[\Theta(G)]{\cdot} = \operatorname{GD}_{(n)}(\cdot;\phi_{\Theta(G)})$. The minimax risk is defined as
\begin{align*}
    \mathcal{M}_{n,p,d}(\mathcal{H}_{p,d}(\bmin, \theta_{\min}, \theta_{\max})) := \min_{\psi} \max_{\phi_{\Theta(G)}} \prob[\Theta(G)]{\psi(\mathbf{X}_{0}^{n-1}) \neq G}
\end{align*}

The minimax risk represents the minimum error achieved by the 
best estimator $\widehat{G}$ in the worst-case scenario, where distributions on graphs are hard to distinguish. Our main goal is to understand the scaling laws of the triplet $(n,p,d)$ where $n$ is the number of updates or samples observed, $d$ is the maximum degree of the graph, and $p$ is the number of vertices; such that the minimax risk $\mathcal{M}_{n,p,d}$ vanishes asymptotically.

An alternative viewpoint, which we incorporate to maintain consistency with the achievability theorem's style, is to consider the time interval $T$ over which the continuous-time Glauber process $\{\mathbf{Y}^{(t)}\}_{t=0}^{T}$ was observed. These two perspectives are equivalent in the sense that $T = \frac{n}{p}$, as there are, on average, $p$ updates over a unit interval of time. Recall that in Section~\ref{subsec: Glauber dynamics}, we treat the number of samples observed over $T$ time units to be deterministic. Therefore we denote it with lower case $n$. Moreover, as noted earlier, $\{\mathbf{Y}^{(t)}\}$ and the input data sequence $\{\mathcal{X}^{(n)}, \mathcal{I}^{(n)}\}$, where $\mathcal{I}^{(n)} = (I^{(1)},\dots,I^{(n)})$ can be used interchangeably.

Note that the minimax risk over the hard subclass lower bounds the minimax risk over the general class. That is, 
\begin{align*}
    \mathcal{M}_{n,p,d}(\mathcal{G}_{p,d}(\bmin, \theta_{\min}, \theta_{\max})) \geq \mathcal{M}_{n,p,d}(\mathcal{H}_{p,d}(\bmin, \theta_{\min}, \theta_{\max})).
\end{align*}

By using Fano's method, we will characterize $n_{\text{lb}}(p,d)$ such that for $n < n_{\text{lb}}(p,d)$ the minimax risk $\mathcal{M}_{n,p,d}(\mathcal{H}) \geq 1/2$ which implies $\mathcal{M}_{n,p,d}(\mathcal{G}) \geq 1/2$. Equivalently if, $\mathcal{M}_{n,p,d}(\mathcal{G}) < 1/2$, then $n \geq n_{\text{lb}}(p,d)$.

\subsection{Minimax Lower Bound}

We now state the lower bound on the observation time $T$ for graphical model selection.

\begin{theorem}[Lower bound for graph recovery from Glauber dynamics]\label{thm:lower-bound-glauber}
Assume $0 < \bmin \leq 1/2$ and $2 \leq d \leq p-2$. There exists a universal
constant $c>0$ such that if
\begin{equation*}
T \leq c \, \frac{\log \binom{p-d}{2}}{\bmin^2},
\end{equation*}
then
\begin{equation*}
\mathcal{M}_{n,p,d}\bigl(\mathcal{H}_{p,d}(\bmin,\theta_{\min},\theta_{\max})\bigr)
\geq \frac{1}{2}.
\end{equation*}
Consequently,
\begin{equation*}
\mathcal{M}_{n,p,d}\bigl(\mathcal{G}_{p,d}(\bmin,\theta_{\min},\theta_{\max})\bigr)
\geq \frac{1}{2}.
\end{equation*}
Equivalently, if some estimator achieves probability of error strictly smaller than
$1/2$ uniformly over
$\mathcal{G}_{p,d}(\bmin,\theta_{\min},\theta_{\max})$, then necessarily
\begin{equation*}
T \geq c \, \frac{\log(p-d)}{\bmin^2}.
\end{equation*}
\end{theorem}

\textbf{Remark:}
As a consequence of Theorem \ref{thm:lower-bound-glauber}, if the minimum edge strength $\bmin$ decays at the rate $\mathcal{O}\left(\frac{1}{d}\right)$, then the lower bound on the observation time becomes $T = \Omega(d^2 \log p)$. In this regime, the sample complexity $T$ from Theorem \ref{thm: final theorem} matches the information-theoretic lower bound up to a factor of $\poly(d\log\left(p\right))$.

\section{Numerical Experiments}\label{sec: experiments}

The preceding sections give worst-case, finite-sample guarantees for two
different ways of using a single Glauber trajectory. The purpose of this section
is more modest: we use synthetic experiments to illustrate the qualitative
tradeoffs predicted by the theory. In particular, we examine whether the local
statistic in \algone{} exhibits a sharp recovery transition as the observation
time increases, whether burn-in and thinning lead to behavior comparable to an
i.i.d.\ learner in the Dobrushin regime, and whether the edgewise form of
\algone{} gives the expected computational advantage.

Throughout the experiments, the underlying graph is a random $d$-regular graph
on $p$ nodes. The precision matrix has off-diagonal entries drawn uniformly from
$[0.2,0.4]$ with random signs and is rescaled so that the Dobrushin radius
\[
    r := \max_i \sum_{j\ne i} |\theta_{ij}/\theta_{ii}|
\]
matches a prescribed value. We simulate the random-scan Gaussian Gibbs sampler
with stationary distribution $\mathcal{N}(0,\boldsymbol{\Theta}^{-1})$. All
methods are evaluated on the same trajectory and use top-$d$ edge selection per
node with the true maximum degree $d$; thus differences in performance reflect
the statistics computed from the trajectory rather than different graph-size
selection rules.

We compare four estimators:
\begin{itemize}
    \item \textbf{\algone{}}: the local edge-testing estimator analyzed in
    Section~\ref{sec: theorem}.
    \item \textbf{BTR-GL instantiated with GLasso}: an instance of \algtwo{} that
    first subsamples the trajectory and then applies GLasso. The theorem for \algtwo{} is stated with DICE
    because of its explicit i.i.d.\ sample complexity~\citep{misra2020information};
    for experiments we use graphical lasso as a standard implementable base
    learner.
    \item \textbf{naive GLasso} and \textbf{naive Nodewise Lasso}: reference procedures that ignore temporal dependence and apply an
    i.i.d.\ learner (GLasso~\citep{friedman2008sparse}, nodewise Lasso~\citep{Meinshausen2006Neighborhood}) directly to the full trajectory. These methods are not
    covered by our guarantees; they serve only to contextualize the empirical
    cost of accounting for dependence.
\end{itemize}
We report F1 score under top-$d$ selection, AUROC as a threshold-free ranking
metric, the empirical observation threshold $N^{*}$ at which mean AUROC first
crosses $0.9$, and wall-clock runtime. Recovery curves are averaged over $3$
replicates and sample-complexity sweeps over $5$ replicates. As a check on the
simulation, at $p=8$ and $N=2{,}000{,}000$ updates the empirical covariance of
the generated trajectory agrees with $\boldsymbol{\Theta}^{-1}$ to relative
Frobenius error $0.6\%$.

\textbf{Compute environment.} The reported experiments were run on an Exxact TensorEX 2U rackmount machine with two AMD EPYC Rome 7542 processors (32 cores and 64 threads each), 1\,TB DDR4 ECC memory, four NVIDIA A100 SXM4 GPUs with 40\,GB memory each, a 2\,TB NVMe OS drive, a 15.36\,TB NVMe data drive, and Ubuntu 18.04. Some development and pilot runs were also executed on a 2021 Apple M1 MacBook with 64\,GB RAM, but the server configuration above is the conservative compute environment for reproducing the reported results. 

\subsection{Recovery Performance}\label{subsec: experiments-recovery}

Figures~\ref{fig: recovery-F1} and~\ref{fig: recovery-AUROC} show recovery as a
function of normalized observation time $N/p$ for $p=20$, $d=4$, and
$r\in\{0.5,0.7\}$. The two naive procedures recover with the shortest
trajectories in this finite-dimensional regime, and the burn-in/thinning
procedure follows them within a small constant factor. This is consistent with
the interpretation of \algtwo{} as a reduction to i.i.d.\ structure learning
after the chain has been sufficiently decorrelated.

The local estimator \algone{} displays a different profile. Its performance is
near chance for short trajectories and then improves rapidly once enough
informative update windows have accumulated, reaching essentially exact
recovery at the largest observation times shown. This transition is consistent
with the event-counting analysis in Theorem~\ref{thm: final theorem}: the
statistic is designed to be valid without waiting for global mixing, but it
requires enough local update patterns to average out the martingale noise.

\begin{figure}[!t]
    \centering
    \includegraphics[width=\linewidth]{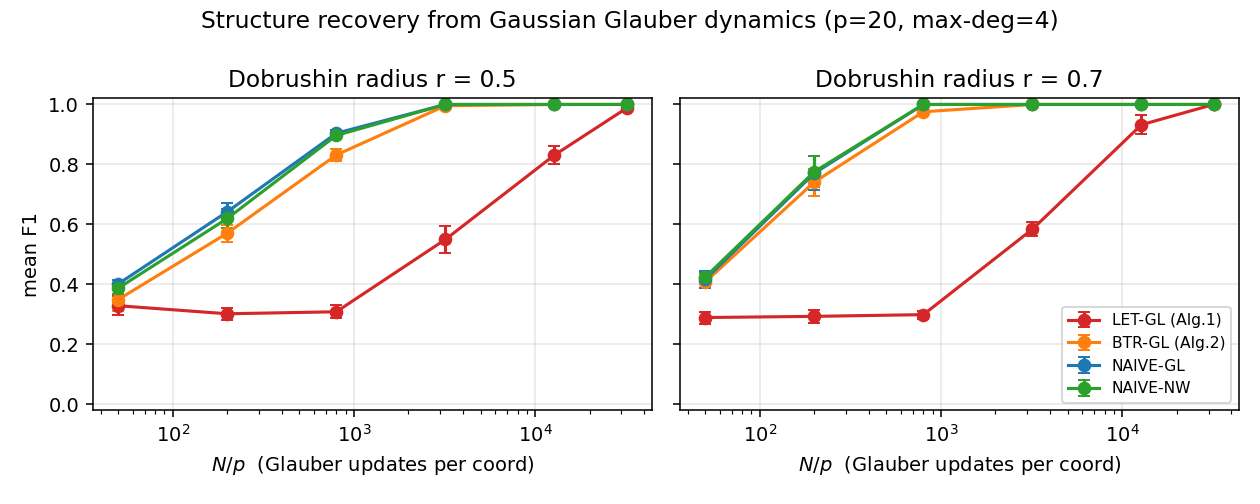}
    \caption{Finite-sample recovery as a function of normalized observation time. F1 score is shown for the four estimators on random $d$-regular graphs with $p=20$ and $d=4$. Left: Dobrushin radius $r=0.5$. Right: $r=0.7$. Curves are means over $3$ replicates; shading shows $\pm 1$ standard deviation.}
    \label{fig: recovery-F1}
\end{figure}

\begin{figure}[!t]
    \centering
    \includegraphics[width=\linewidth]{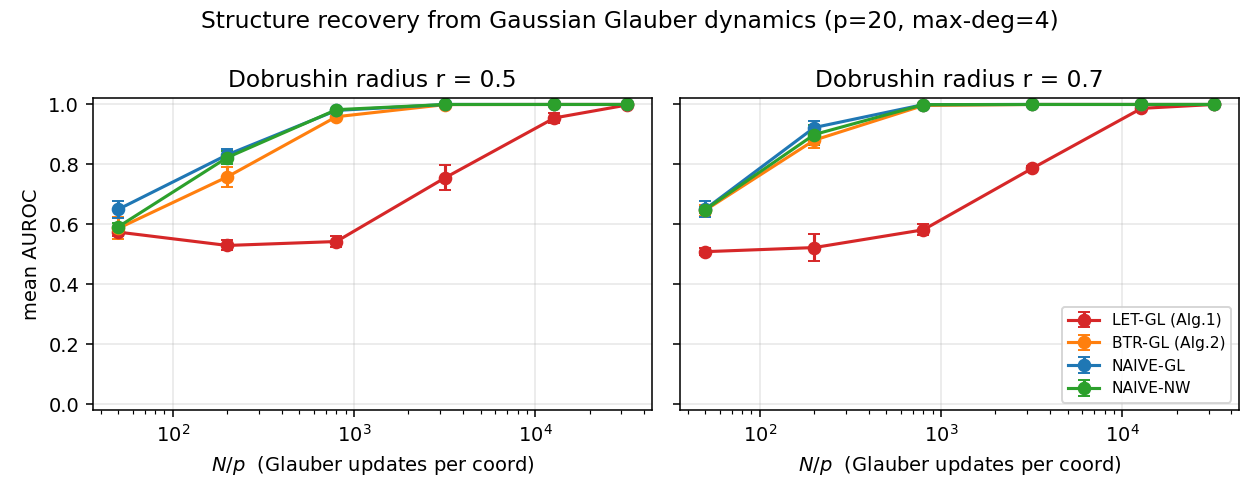}
    \caption{Threshold-free recovery for the same setup as Figure~\ref{fig: recovery-F1}. AUROC emphasizes the ranking behavior of the edge scores and displays the transition of the local statistic in \algone{}.}
    \label{fig: recovery-AUROC}
\end{figure}

\subsection{Computational Scaling and Parallelism}\label{subsec: experiments-compute}

Figure~\ref{fig: compute} examines the computational side of this tradeoff.
Panel~(A) varies $p \in \{40,80,160,320,640\}$ at $N=200p$ updates and measures
the wall-clock cost of applying each estimator to the full trajectory. The
runtime of \algone{} grows approximately linearly in $p$, in agreement with its
event-counting form under top-$d$ selection. Graphical lasso, by contrast,
operates on a dense empirical covariance matrix and tracks the displayed
$\mathcal{O}(p^{3})$ reference curve.

Panel~(B) illustrates the edgewise and additive structure of the local
statistic. The trajectory accumulator for \algone{} can be computed on
contiguous shards and summed at the end. When the trajectory is divided into
$W\in\{1,2,4,8,16\}$ shards, the total sequential work is essentially unchanged
while the per-shard work scales as $1/W$. This experiment isolates the
parallelism implied by the statistic itself; it is not intended as a benchmark
of any particular multiprocessing implementation.

\begin{figure}[!t]
    \centering
    \includegraphics[width=\linewidth]{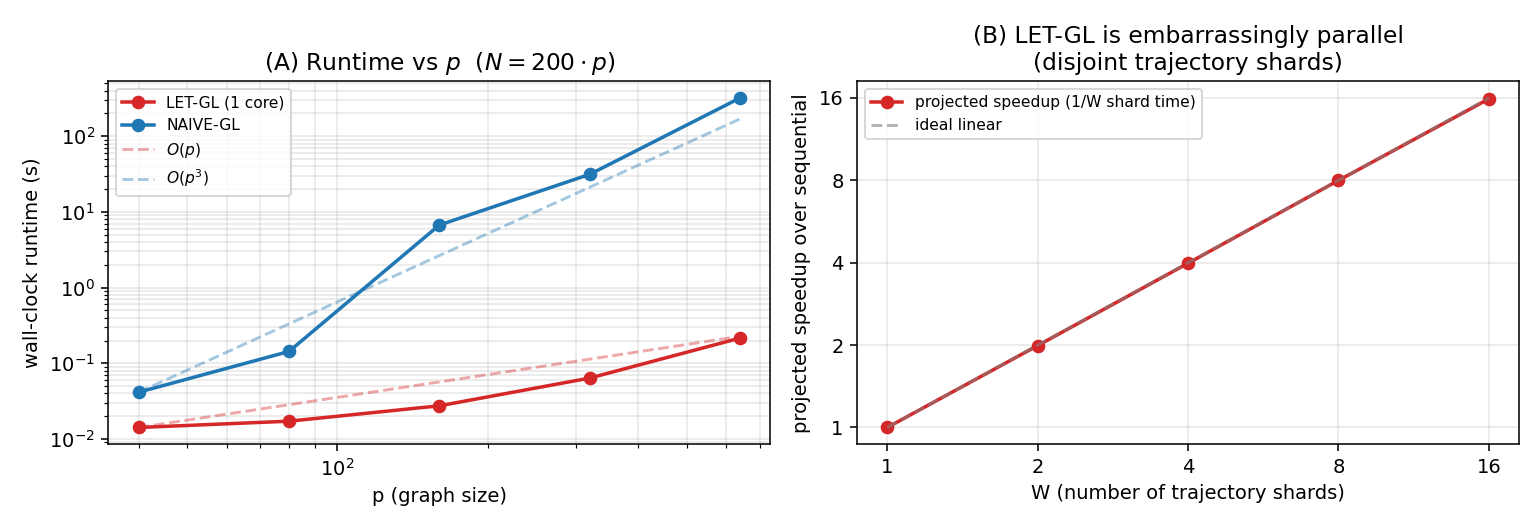}
    \caption{Computational behavior of the local statistic. Panel~(A): wall-clock runtime versus $p$ at $N = 200p$, with graphical lasso included as a covariance-based reference. The dashed line is the $\mathcal{O}(p^{3})$ scaling reference. Panel~(B): trajectory split into $W$ contiguous shards at $p=100$, $N \approx 1.5\!\cdot\!10^{6}$; the per-shard work scales as $1/W$ while total work is conserved.}
    \label{fig: compute}
\end{figure}

\subsection{Observation-Time Scaling}\label{subsec: experiments-samplecomplexity}

Finally, we consider how the observation time needed for reliable ranking grows
with $p$ in a fixed sparse regime. We set $d=4$ and $r=0.6$, vary
$p\in\{15,25,40\}$, and define $N^{*}$ to be the smallest observation length at
which the mean AUROC reaches $0.9$, using log-linear interpolation between the
two bracketing grid points. Figure~\ref{fig: samplecomplexity} displays the
AUROC curves and the resulting $N^{*}$ values; Table~\ref{tab: nstar} reports
the numerical values.

These experiments should not be read as estimating the sharp asymptotic
exponent in $p$. They are instead a finite-dimensional check that the relative
ordering of the procedures agrees with the theory. The burn-in/thinning
estimator is close to the naive reference in the Dobrushin regime,
whereas \algone{} requires longer trajectories but remains computationally much
cheaper per update and does not rely on treating consecutive states as
approximately independent.

\begin{figure}[!t]
    \centering
    \includegraphics[width=\linewidth]{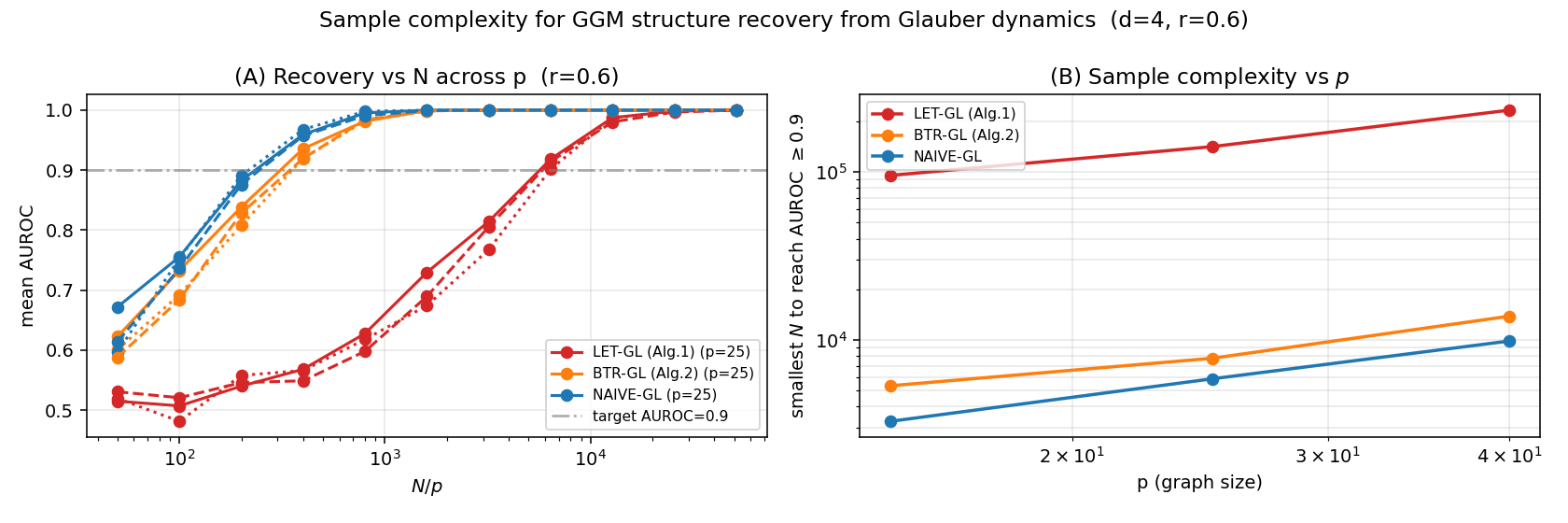}
    \caption{Empirical observation-time scaling at $d=4$, $r=0.6$. Panel~(A): AUROC versus $N/p$ for $p \in \{15,25,40\}$ over $5$ replicates; the horizontal line marks $\mathrm{AUROC}=0.9$. Panel~(B): interpolated $N^{*}$ versus $p$ on log-log axes.}
    \label{fig: samplecomplexity}
\end{figure}

\begin{table}[!t]
    \centering
    \caption{Empirical observation threshold $N^{*}$, defined as the smallest $N$ at which mean AUROC reaches $0.9$, for $d=4$ and $r=0.6$. Values are obtained by log-linear interpolation between bracketing grid points.}
    \label{tab: nstar}
    \begin{tabular}{rrrr}
        \toprule
        $p$ & naive GLasso & BTR (GLasso) & \algone{} \\
        \midrule
        $15$ & $3{,}260$ & $5{,}310$ & $95{,}252$ \\
        $25$ & $5{,}849$ & $7{,}751$ & $141{,}518$ \\
        $40$ & $9{,}813$ & $13{,}799$ & $233{,}410$ \\
        \bottomrule
    \end{tabular}
\end{table}

\subsection{Interpretation}\label{subsec: experiments-discussion}

Taken together, the experiments illustrate the complementarity of the two
theoretical approaches. The local estimator \algone{} is designed for validity
directly on the dependent trajectory and exposes substantial computational
parallelism, but pays a statistical price in the number of updates required in
the regimes tested here. The reduction-based estimator \algtwo{} is
statistically closer to ordinary i.i.d.\ graphical model selection when the
Dobrushin contraction condition holds, but it relies on burn-in/thinning and
inherits the computational profile of the chosen base learner. The
naive procedures are useful empirical references, but they do not
provide a substitute for the decoupling argument: their behavior in these
experiments does not imply a uniform guarantee for dependent Glauber data.

\section{Discussion}

We studied Gaussian graphical model selection from a single trajectory of
Glauber dynamics and developed two complementary algorithmic approaches.
LET-GL is a local edge-testing method whose main advantages are its transparent trajectory-level interpretation and
parallel implementation: candidate edges can be processed independently, and
the trajectory accumulator can be sharded and merged. Its main drawback is
statistical: the guarantee has a relatively unfavorable dependence on degree
and edge strength.

BTR-GL takes the opposite route. It uses burn-in and thinning to reduce the
dependent-data problem to i.i.d. Gaussian graphical model selection. This method
requires a Dobrushin contraction condition and sacrifices the fully local
edgewise structure of LET-GL, but it gives a modular path to sharper sample
complexity by leveraging existing i.i.d. structure learners. The main technical
ingredient is a total-variation decoupling bound for Gaussian Gibbs samplers
obtained through approximate Lipschitz smoothing. 
Relaxing the Dobrushin-style condition here, possibly via conductance based techniques, is a promising direction for future work.

The lower bound shows that any method must pay at least
\(\Omega(\log(p-d)/\beta_{\min}^2)\) normalized observation time over a natural
model class. Closing the remaining polynomial gaps in \(d\), \(\beta_{\min}\),
and \(1-r\), and comparing product-moment and ratio-aggregation local
estimators both theoretically and empirically, remain important open problems.

\section*{Acknowledgments}
This work was supported by the National Science Foundation (NSF) under award
CCF-2048223. We thank Mahbod Majid for carefully reading an earlier version of
this manuscript and helping identify a dependence gap in the local ratio-estimator
argument.

\bibliography{final}

@misc{drton2016structure,
      title={Structure Learning in Graphical Modeling}, 
      author={Mathias Drton and Marloes H. Maathuis},
      year={2016},
      eprint={1606.02359},
      archivePrefix={arXiv},
      primaryClass={stat.ME}
}

@inproceedings{wang2010information,
  title={Information-theoretic bounds on model selection for Gaussian Markov random fields},
  author={Wang, Wei and Wainwright, Martin J and Ramchandran, Kannan},
  booktitle={2010 IEEE International Symposium on Information Theory},
  pages={1373--1377},
  year={2010},
  organization={IEEE}
}

@misc{bresler2014learning,
      title={Learning graphical models from the Glauber dynamics}, 
      author={Guy Bresler and David Gamarnik and Devavrat Shah},
      year={2014},
      eprint={1410.7659},
      archivePrefix={arXiv},
      primaryClass={cs.LG}
}

@book{lauritzen1996,
  author = {Lauritzen, Steffen L.},
  isbn = {0-19-852219-3},
  publisher = {Oxford University Press},
  title = {Graphical Models},
  year = 1996
}

@book{LevinPeresWilmer2006,
  author = {Levin, David A. and Peres, Yuval and Wilmer, Elizabeth L.},
  publisher = {American Mathematical Society},
  title = {{Markov chains and mixing times}},
  year = 2006
}

@book{Horn_Johnson_2012_Matrix, place={Cambridge}, edition={2}, title={Matrix Analysis}, publisher={Cambridge University Press}, author={Horn, Roger A. and Johnson, Charles R.}, year={2012}}

@ARTICLE{chow1968tree,
  author={Chow, C. and Liu, C.},
  journal={IEEE Transactions on Information Theory}, 
  title={Approximating discrete probability distributions with dependence trees}, 
  year={1968},
  volume={14},
  number={3},
  pages={462-467},
  doi={10.1109/TIT.1968.1054142}}

@ARTICLE{chow1973polytrees,
  author={Chow, C. and Wagner, T.},
  journal={IEEE Transactions on Information Theory}, 
  title={Consistency of an estimate of tree-dependent probability distributions (Corresp.)}, 
  year={1973},
  volume={19},
  number={3},
  pages={369-371},
  doi={10.1109/TIT.1973.1055013}}

@article{Yuan2007Gaussian,
    author = {Yuan, Ming and Lin, Yi},
    title = "{Model selection and estimation in the Gaussian graphical model}",
    journal = {Biometrika},
    volume = {94},
    number = {1},
    pages = {19-35},
    year = {2007},
    month = {03},
    issn = {0006-3444},
    doi = {10.1093/biomet/asm018},
    url = {https://doi.org/10.1093/biomet/asm018},
    eprint = {https://academic.oup.com/biomet/article-pdf/94/1/19/617853/asm018.pdf},
}

@misc{banerjee2007model,
      title={Model Selection Through Sparse Maximum Likelihood Estimation}, 
      author={Onureena Banerjee and Laurent El Ghaoui and Alexandre d'Aspremont},
      year={2007},
      eprint={0707.0704},
      archivePrefix={arXiv},
      primaryClass={cs.AI}
}

@article{Meinshausen2006Neighborhood,
   title={High-dimensional graphs and variable selection with the Lasso},
   volume={34},
   ISSN={0090-5364},
   url={http://dx.doi.org/10.1214/009053606000000281},
   DOI={10.1214/009053606000000281},
   number={3},
   journal={The Annals of Statistics},
   publisher={Institute of Mathematical Statistics},
   author={Meinshausen, Nicolai and Bühlmann, Peter},
   year={2006},
   month=jun }

@article{
Montanari2010InformationSpread,
author = {Andrea Montanari  and Amin Saberi },
title = {The spread of innovations in social networks},
journal = {Proceedings of the National Academy of Sciences},
volume = {107},
number = {47},
pages = {20196-20201},
year = {2010},
doi = {10.1073/pnas.1004098107},
URL = {https://www.pnas.org/doi/abs/10.1073/pnas.1004098107},
eprint = {https://www.pnas.org/doi/pdf/10.1073/pnas.1004098107},
}

@article{Lara2019GlauberEpidemicModel,
author = {Lara, Cristina and Massad, Eduardo and Fernandez Lopez, Luis and Amaku, Marcos},
year = {2019},
month = {01},
pages = {1052-1066},
title = {Analogy between the Formulation of Ising-Glauber Model and Si Epidemiological Model},
volume = {07},
journal = {Journal of Applied Mathematics and Physics},
doi = {10.4236/jamp.2019.75071}
}

@article{Glauber1963Ising, title={Time-Dependent Statistics of the Ising Model}, volume={4}, ISSN={1089-7658}, url={http://dx.doi.org/10.1063/1.1703954}, DOI={10.1063/1.1703954}, number={2}, journal={Journal of Mathematical Physics}, publisher={AIP Publishing}, author={Glauber, Roy J.}, year={1963}, month=feb, pages={294–307} }

@article{Auletta2017,
  title = {Metastability of Logit Dynamics for Coordination Games},
  volume = {80},
  ISSN = {1432-0541},
  url = {http://dx.doi.org/10.1007/s00453-017-0371-8},
  DOI = {10.1007/s00453-017-0371-8},
  number = {11},
  journal = {Algorithmica},
  publisher = {Springer Science and Business Media LLC},
  author = {Auletta,  Vincenzo and Ferraioli,  Diodato and Pasquale,  Francesco and Persiano,  Giuseppe},
  year = {2017},
  month = sep,
  pages = {3078–3131}
}

@misc{zhan2021graphical,
      title={Graphical Models for Financial Time Series and Portfolio Selection}, 
      author={Ni Zhan and Yijia Sun and Aman Jakhar and He Liu},
      year={2021},
      eprint={2101.09214},
      archivePrefix={arXiv},
      primaryClass={cs.LG}
}

@article{Farasat2015,
	author = {Farasat, Alireza and Nikolaev, Alexander and Srihari, Sargur N. and Blair, Rachael Hageman},
	copyright = {Springer-Verlag Wien},
	doi = {10.1007/s13278-015-0289-6},
	journal = {Social Network Analysis and Mining},
	language = {English},
	number = {1},
	pages = {1-18},
	title = {Probabilistic graphical models in modern social network analysis},
	volume = {5},
	year = {2015},
}

@article{Murad2021Biology,
	address = {Center for Molecular Biology and Genetic Engineering, State University of Campinas, Campinas, S{\~a}o Paulo, Brazil.; Center for Molecular Biology and Genetic Engineering, State University of Campinas, Campinas, S{\~a}o Paulo, Brazil. brandaom@unicamp.br.},
	author = {Murad, Natalia Faraj and Brand{\~a}o, Marcelo Mendes},
	copyright = {{\copyright}2021. Springer Nature Switzerland AG.},
	crdt = {2022/02/03 12:18},
	date = {2021},
	dcom = {20220207},
	doi = {10.1007/978-3-030-80352-0{\_}7},
	edat = {2022/02/04 06:00},
	issn = {0065-2598 (Print); 0065-2598 (Linking)},
	jid = {0121103},
	journal = {Adv Exp Med Biol},
	jt = {Advances in experimental medicine and biology},
	language = {eng},
	lid = {10.1007/978-3-030-80352-0{\_}7 {$[$}doi{$]$}},
	lr = {20220504},
	mh = {*Algorithms; Computational Biology; Gene Regulatory Networks; Information Storage and Retrieval; Models, Biological; *Models, Statistical},
	mhda = {2022/02/08 06:00},
	oto = {NOTNLM},
	own = {NLM},
	pages = {119--130},
	phst = {2022/02/03 12:18 {$[$}entrez{$]$}; 2022/02/04 06:00 {$[$}pubmed{$]$}; 2022/02/08 06:00 {$[$}medline{$]$}},
	pl = {United States},
	pmid = {35113399},
	pst = {ppublish},
	pt = {Journal Article},
	sb = {IM},
	status = {MEDLINE},
	title = {Probabilistic Graphical Models Applied to Biological Networks.},
	volume = {1346},
	year = {2021}}

@article{Dennis2022Markov,
	author = {Dennis, Louise A and Fu, Yu and Slavkovik, Marija},
	journal = {Journal of Logic and Computation},
	month = {03},
	number = {6},
	pages = {1195-1211},
	title = {{Markov chain model representation of information diffusion in social networks}},
	volume = {32},
	year = {2022}}

@inbook{Martinelli1999,
  title = {Lectures on Glauber Dynamics for Discrete Spin Models},
  ISBN = {9783540481157},
  ISSN = {1617-9692},
  url = {http://dx.doi.org/10.1007/978-3-540-48115-7_2},
  DOI = {10.1007/978-3-540-48115-7_2},
  booktitle = {Lectures on Probability Theory and Statistics},
  publisher = {Springer Berlin Heidelberg},
  author = {Martinelli,  Fabio},
  year = {1999},
  pages = {93–191}
}

@book{rue2005gaussian,
  address = {London},
  author = {Rue, H. and Held, L.},
  publisher = {Chapman \& Hall},
  series = {Monographs on Statistics and Applied Probability},
  title = {Gaussian {M}arkov Random Fields: {T}heory and Applications},
  url = {http://dx.doi.org/10.1201/9780203492024},
  volume = 104,
  year = 2005
}

@misc{gaitonde2023unifiedapproachlearningising,
      title={A Unified Approach to Learning Ising Models: Beyond Independence and Bounded Width}, 
      author={Jason Gaitonde and Elchanan Mossel},
      year={2023},
      eprint={2311.09197},
      archivePrefix={arXiv},
      primaryClass={cs.LG},
      url={https://arxiv.org/abs/2311.09197}, 
}

@misc{dutt2021exponentialreductionsamplecomplexity,
      title={Exponential Reduction in Sample Complexity with Learning of Ising Model Dynamics}, 
      author={Arkopal Dutt and Andrey Y. Lokhov and Marc Vuffray and Sidhant Misra},
      year={2021},
      eprint={2104.00995},
      archivePrefix={arXiv},
      primaryClass={cs.LG},
      url={https://arxiv.org/abs/2104.00995}, 
}

@misc{gaitonde2024efficientlylearningmarkovrandom,
      title={Efficiently Learning Markov Random Fields from Dynamics}, 
      author={Jason Gaitonde and Ankur Moitra and Elchanan Mossel},
      year={2024},
      eprint={2409.05284},
      archivePrefix={arXiv},
      primaryClass={cs.LG},
      url={https://arxiv.org/abs/2409.05284}, 
}

@article{AMIT199182,
	author = {Yali Amit},
	journal = {Journal of Multivariate Analysis},
	number = {1},
	pages = {82-99},
	title = {On rates of convergence of stochastic relaxation for Gaussian and non-Gaussian distributions},
	volume = {38},
	year = {1991}}

@misc{zhou2011structurelearningprobabilisticgraphical,
      title={Structure Learning of Probabilistic Graphical Models: A Comprehensive Survey}, 
      author={Yang Zhou},
      year={2011},
      eprint={1111.6925},
      archivePrefix={arXiv},
      primaryClass={stat.ML},
      url={https://arxiv.org/abs/1111.6925}, 
}

@article{degroot,
 ISSN = {01621459, 1537274X},
 URL = {http://www.jstor.org/stable/2285509},
 author = {Morris H. DeGroot},
 journal = {Journal of the American Statistical Association},
 number = {345},
 pages = {118--121},
 publisher = {[American Statistical Association, Taylor & Francis, Ltd.]},
 title = {Reaching a Consensus},
 urldate = {2024-10-10},
 volume = {69},
 year = {1974}
}

@article{ChungConcentration,
	author = {Fan Chung and Linyuan Lu},
	journal = {Internet Mathematics},
	number = {1},
	pages = {79 -- 127},
	title = {{Concentration inequalities and martingale inequalities: a survey}},
	volume = {3},
	year = {2006}}

@article{roberts1997updating,
	author = {Roberts, G. O. and Sahu, S. K.},
	journal = {Journal of the Royal Statistical Society: Series B (Statistical Methodology)},
	number = {2},
	pages = {291-317},
	title = {Updating Schemes, Correlation Structure, Blocking and Parameterization for the Gibbs Sampler},
	volume = {59},
	year = {1997}}

@article{liu1994covariance,
	author = {Jun S. Liu and Wing Hung Wong and Augustine Kong},
	journal = {Biometrika},
	number = {1},
	pages = {27--40},
	title = {Covariance Structure of the Gibbs Sampler with Applications to the Comparisons of Estimators and Augmentation Schemes},
	volume = {81},
	year = {1994}}

@article{shamir2011variant,
  title={A variant of azuma's inequality for martingales with subgaussian tails},
  author={Shamir, Ohad},
  journal={arXiv preprint arXiv:1110.2392},
  year={2011}
}

@article{friedman2008sparse,
  title={Sparse inverse covariance estimation with the graphical lasso},
  author={Friedman, Jerome and Hastie, Trevor and Tibshirani, Robert},
  journal={Biostatistics},
  volume={9},
  number={3},
  pages={432--441},
  year={2008},
  publisher={Oxford University Press}
}

@book{10.7551/mitpress/1754.001.0001,
    author = {Spirtes, Peter and Glymour, Clark and Scheines, Richard},
    title = {Causation, Prediction, and Search},
    publisher = {The MIT Press},
    year = {2001},
    month = {01},
    isbn = {9780262284158},
    doi = {10.7551/mitpress/1754.001.0001},
    url = {https://doi.org/10.7551/mitpress/1754.001.0001},
}

@article{wang2014convergence,
  title={Convergence rate and concentration inequalities for Gibbs sampling in high dimension},
  author={Wang, Neng-Yi and Wu, Liming},
  year={2014}
}

@article{wang2017convergence,
  title={Convergence rates of the random scan Gibbs sampler under the Dobrushin's uniqueness condition},
  author={Wang, Neng-Yi},
  year={2017}
}

@inproceedings{TRD25,
  title={Structure learning in gaussian graphical models from glauber dynamics},
  author={Tirukkonda, Vignesh and Rayas, Anirudh and Dasarathy, Gautam},
  booktitle={2025 IEEE International Symposium on Information Theory (ISIT)},
  pages={1--6},
  year={2025},
  organization={IEEE}
}

@inproceedings{SWMM26,
  title={Learning Gaussian Graphical Models from a Glauber Trajectory Without Mixing},
  author={Shen, Eric and Wu, Tony and Majid, Mahbod and Moitra, Ankur},
  booktitle={Proceedings of the 43rd International Conference on Machine Learning},
  year={2026},
  note={To appear}
}

@inproceedings{misra2020information,
  title={Information theoretic optimal learning of gaussian graphical models},
  author={Misra, Sidhant and Vuffray, Marc and Lokhov, Andrey Y},
  booktitle={Conference on Learning Theory},
  pages={2888--2909},
  year={2020},
  organization={PMLR}
}

@article{jones2004sufficient,
  title={Sufficient burn-in for Gibbs samplers for a hierarchical random effects model},
  author={Jones, Galin L and Hobert, James P},
  year={2004}
}

@article{jones2001honest,
  title={Honest exploration of intractable probability distributions via Markov chain Monte Carlo},
  author={Jones, Galin L and Hobert, James P},
  journal={Statistical Science},
  pages={312--334},
  year={2001},
  publisher={JSTOR}
}

@article{maceachern1994subsampling,
  title={Subsampling the Gibbs sampler},
  author={MacEachern, Steven N and Berliner, L Mark},
  journal={The American Statistician},
  volume={48},
  number={3},
  pages={188--190},
  year={1994},
  publisher={Taylor \& Francis}
}

@article{link2012thinning,
  title={On thinning of chains in MCMC},
  author={Link, William A and Eaton, Mitchell J},
  journal={Methods in ecology and evolution},
  volume={3},
  number={1},
  pages={112--115},
  year={2012},
  publisher={Wiley Online Library}
}

@book{gamerman2006markov,
  title={Markov chain Monte Carlo: stochastic simulation for Bayesian inference},
  author={Gamerman, Dani and Lopes, Hedibert F},
  year={2006},
  publisher={Chapman and Hall/CRC}
}

@article{owen2017statistically,
  title={Statistically efficient thinning of a Markov chain sampler},
  author={Owen, Art B},
  journal={Journal of Computational and Graphical Statistics},
  volume={26},
  number={3},
  pages={738--744},
  year={2017},
  publisher={Taylor \& Francis}
}

@article{riabiz2022optimal,
  title={Optimal thinning of MCMC output},
  author={Riabiz, Marina and Chen, Wilson Ye and Cockayne, Jon and Swietach, Pawel and Niederer, Steven A and Mackey, Lester and Oates, Chris J},
  journal={Journal of the Royal Statistical Society Series B: Statistical Methodology},
  volume={84},
  number={4},
  pages={1059--1081},
  year={2022},
  publisher={Oxford University Press}
}

@article{rosenthal1995minorization,
  title={Minorization conditions and convergence rates for Markov chain Monte Carlo},
  author={Rosenthal, Jeffrey S},
  journal={Journal of the American Statistical Association},
  volume={90},
  number={430},
  pages={558--566},
  year={1995},
  publisher={Taylor \& Francis}
}

@article{meyn1994computable,
  author  = {Meyn, Sean P. and Tweedie, Richard L.},
  title   = {Computable bounds for geometric convergence rates of {M}arkov chains},
  journal = {Annals of Applied Probability},
  volume  = {4},
  number  = {4},
  pages   = {981--1011},
  year    = {1994}
}

@article{rajaratnam2015mcmc,
  author  = {Rajaratnam, Bala and Sparks, Doug},
  title   = {{MCMC}-based inference in the era of big data: A fundamental analysis of the convergence complexity of high-dimensional chains},
  journal = {arXiv preprint arXiv:1508.00947},
  year    = {2015}
}

@article{qin2022wasserstein,
  author  = {Qin, Qian and Hobert, James P.},
  title   = {{W}asserstein-based methods for convergence complexity analysis of {MCMC} with applications},
  journal = {Annals of Applied Probability},
  volume  = {32},
  number  = {1},
  pages   = {124--166},
  year    = {2022},
  doi     = {10.1214/21-AAP1673}
}

@article{ascolani2024entropy,
  author  = {Ascolani, Filippo and Lavenant, Hugo and Zanella, Giacomo},
  title   = {Entropy contraction of the {G}ibbs sampler under log-concavity},
  journal = {arXiv preprint arXiv:2410.00858},
  year    = {2024}
}

@article{wadia2024gibbs,
  author  = {Wadia, Neha S.},
  title   = {A mixing time bound for {G}ibbs sampling from log-smooth log-concave distributions},
  journal = {arXiv preprint arXiv:2412.17899},
  year    = {2024}
}

@book{Villani2009,
  author    = {Villani, C{\'e}dric},
  title     = {Optimal Transport: Old and New},
  publisher = {Springer},
  address   = {Berlin},
  year      = {2009}
}

@article{RobertsRosenthal2002,
  author  = {Roberts, Gareth O. and Rosenthal, Jeffrey S.},
  title   = {One-shot coupling for certain stochastic recursive sequences},
  journal = {Stochastic Processes and their Applications},
  volume  = {99},
  number  = {2},
  pages   = {195--208},
  year    = {2002}
}

@article{MadrasSezer2010,
  author  = {Madras, Neal and Sezer, Deniz},
  title   = {Quantitative bounds for {M}arkov chain convergence: {W}asserstein and total variation distances},
  journal = {Bernoulli},
  volume  = {16},
  number  = {3},
  pages   = {882--908},
  year    = {2010}
}

@article{GibbsSu2002,
  author  = {Gibbs, Alison L. and Su, Francis Edward},
  title   = {On choosing and bounding probability metrics},
  journal = {International Statistical Review},
  volume  = {70},
  number  = {3},
  pages   = {419--435},
  year    = {2002}
}

\appendix
    
\section{Proofs of Theorems}

In this section, we will state and prove Theorem 1 and Theorem 2 under the assumptions \ref{assumption: Bounded edge strength}-\ref{assumption: bounded degree and sample decay} specified in Section~\ref{sec: theorem}. The statements and proofs of all the lemmas used to establish the theorems are provided in Appendix~\ref{sec: proofs of lemmas}.  

\subsection{Theorem~\ref{thm: final theorem}: Statement \& proof}\label{appsubsec: theorem 1 statement and proof}

The first theorem establishes achievability results for the observation time needed for model selection by characterizing the performance of \algone{}.
\setcounter{theorem}{0}

\begin{theorem}\label{appsecthm: final theorem}
    Consider observations from Glauber dynamics $\{\mathbf{Y}^{(t)}\}_{t=0}^{T}$ associated with a Gaussian graphical model $(G,f_{\mathbf{X}})$. Under the assumptions~\ref{assumption: Bounded edge strength}-\ref{assumption: bounded degree and sample decay}, for any $\delta \in (0,1/2)$  and sufficiently large $p$ there exists constants $C_{2}, C_{3}$ and $C_{4}$ such that if the observation time $T$ satisfies ($T_{\max}$ is as in Lemma~\ref{lemma: High probability event B})
    \begin{align}
        T_{\max} \geq T \geq  \frac{C_{2}}{(\bmin\sigmamin)^5}d^{3}\log\left(\frac{p}{\delta}\right)^{5},  
    \end{align}
    
    and the interval length $\tau$ and threshold $\rho$ are set as 
    \begin{align*}
        \tau = d^{-1}\left[\frac{C_{3}{\log\left(\frac{p}{\delta}\right)}}{\sigmamin\bmin} + 1\right]^{-1}, \qquad
        \rho = C_{4}\tau {\log\left(\frac{p}{\delta}\right)},
    \end{align*}
    then \algone{} recovers the true edge set with probability at least $1-\delta$.
\end{theorem}

\begin{proof}
To bound the probability of error, which is the probability of the event that the estimated edge set is not equal to the true edge set, we first condition on the high probability event $B_{\delta}$ and its complement as defined in Section~\ref{sec: the algorithm}. As the algorithm returns an empty edge set when $B_{\delta}^{c}$ occurs, the error probability, in this case, is $\prob[]{B_{\delta}^c} < \delta/2$. In the sequel, we analyze the probability of error when the event $B_{\delta}$ has occurred and show that it is also less than $\delta/2$.

Considering each pair $\{i,j\}$ separately, we consider the cases where $\{i,j\} \in E$ and $\{i,j\} \notin E$. We show in Lemma~\ref{appseclemma: separation2} that for an appropriately chosen $\tau$, when the edge does not exist, the expectation of the test statistic is bounded around zero and is bounded away from zero when the edge exists. Furthermore,  the bounds are such that there is a separation between them, allowing us to set a threshold between the cases. 

Error occurs, when the algorithm claims the existence of an edge between nodes $i,j$ when there is none and vice versa. However, due to separation, with enough samples, we can guarantee the recovery of the true edge set. We use a Bernstein-style martingale concentration inequality, Lemma~\ref{appseclemma: Bernstein style} to establish bounds on the probability of errors. To do this, we need our test statistic to be bounded, almost surely, and we need a bound on the conditional variance of the submartingale. Conditioning on the event $B_{\delta}$ allows us to do both. We expect to be able to strengthen these results using a sub-Gaussian version of martingale concentration~\citep{shamir2011variant} and using an appropriate stopping time construction, however, we leave this for future work.

We begin implementing the above strategy below. We will use the same notation $\mathbb{P}_{B}[\cdot] = \prob[]{\cdot|B_{\delta}}$ as in Section~\ref{sec: the algorithm} for the conditional measure with respect to the event $B_{\delta}$. Fix any $\delta>0$. Then the probability of error decomposes as follows
\begin{align}\label{eqn: error decomposition}
    \prob[]{\hat{E} \neq E} &=\prob[B]{\hat{E} \neq E}\prob[]{B_{\delta}} + \prob[B^{c}]{\hat{E} \neq E } \prob[] {B_{\delta}^{c}} \nonumber\\
    &\leq \prob[B]{\hat{E} \neq E} + \prob[] {B_{\delta}^{c}}.
\end{align} 
We show in Lemma~\ref{lemma: High probability event B} that the event $B_{\delta}$ is a high probability event and therefore $\prob[]{B_{\delta}^{c}} < \delta/2$ holds. Now we focus on bounding the other term in \ref{eqn: error decomposition} which is the probability of error conditioned on $B_{\delta}$. 

Using union bound we write the error as a sum of the probability of error for every $\{i,j\} \in \binom{V}{2}$, that is
\begin{align}\label{eqn: errors type1 and type2}
    \prob[B]{\hat{E}\neq E} & = \prob[B]{\exists \{i,j\} \in \binom{V}{2} : \mathds{1}_{\hat{E}}(i,j) \neq \mathds{1}_{E}(i,j)} \nonumber\\
    & \leq \sum_{\{i,j\}} \prob[B]{\mathds{1}_{\hat{E}}(i,j) \neq \mathds{1}_{E}(i,j)} \nonumber\\
    & = \sum_{\{i,j\}} \mathds{1}_{E}(i,j)\prob[B]{\mathds{1}_{\hat{E}}(i,j) = 0}  + \mathds{1}_{E^{c}}(i,j)\prob[B]{\mathds{1}_{\hat{E}}(i,j) = 1} \nonumber\\
    & = \sum_{\{i,j\}} \Bigg( \mathds{1}_{E}(i,j)\prob[B]{ \abs{T_{ij}} \leq \rho \text{ and } \abs{T_{ji}} \leq \rho} + 
    \mathds{1}_{E^{c}}(i,j)\cdot\prob[B]{\abs{T_{ij}} > \rho \text{ or } \abs{T_{ji}} > \rho} \Bigg). \nonumber\\
    & \leq 2\sum_{\{i,j\}} \Bigg( \mathds{1}_{E}(i,j)\prob[B]{\abs{T_{ij}} \leq \rho} + 
    \mathds{1}_{E^{c}}(i,j)\cdot\prob[B]{\abs{T_{ij}} > \rho} \Bigg). \nonumber
\end{align}

The last equality follows from the definition of \algone{} as it asserts the existence of an edge $\{i,j\}$ if and only if $\max\{\abs{T_{ij}}, \abs{T_{ji}}\} > \rho$. Here $\rho$ is a threshold parameter taken as input by the algorithm. Next, we show how to choose this threshold parameter.

Recall that from Section~\ref{sec: the algorithm} under the event $B_{\delta}$ we have, $\max_{i \in [p], t \leq T_{\max}} \abs{Y_{i}^{(t)}} \leq \ymax $. As discussed in the proof strategy, selecting a threshold involves selecting an appropriate interval length $\tau$ such that a separation exists. In Lemma~\ref{appseclemma: separation2} we establish that if we take  
\begin{align*}
     \tau = d^{-1}\left[\frac{12\ymax^{2}}{\sigmamin\bmin} + 1\right]^{-1},
\end{align*}
then we have the following for every $k \in [k_{\max}]$ and $i,j \in [p]$:
\begin{enumerate}
    \item if $\{i,j\} \notin E$, then $\abs{\E[B]{T_{ij}^{k} \middle | \mathcal{F}_{k-1}}} \leq \eta$ where $\eta = 4\ymax^{2}q\tau d > 0$, and
    \item if $\{i,j\} \in E$, then $ \operatorname{sign}(\beta_{ij})\E[B]{T_{ij}^{k} \middle | \mathcal{F}_{k-1}} \geq \operatorname{sign}(\beta_{ij})\eta'$ where $\eta' = 2(4\ymax^{2}q\tau d) > 0$
\end{enumerate}
where we define the sigma algebra $\mathcal{F}_{k-1} = \sigma\left(\left\{\mathbf{Y}^{(t)}\right\}_{0 \leq t < (k-1)\tau}\right)$. We therefore take the threshold to be 
\begin{align*}
    \rho = \frac{3\eta}{2} = 6\ymax^{2}q\tau d.
\end{align*}

We will now consider the cases of the pair $\{i,j\}$ being connected and not connected separately and bound the terms in \ref{eqn: errors type1 and type2}.

\underline{\bf Case 1:  $\{i,j\} \in E$.} 
We begin by observing that the sequence $\{Z_{k}\}_{k \geq 1}$ defined below is adapted to the filtration $\{\mathcal{F}_{k}\}_{k \geq 1} = \left(\sigma(\{\mathbf{Y}^{(t)}\}_{t=0}^{k\tau}) \right)_{k \geq 1}$ as $Z_k$ is $\mathcal{F}_k$-measureable.
\begin{align*}
    \{Z_{k}\}_{k \geq 1} = \left\{\sum_{l=1}^{k}\operatorname{sign}(\beta_{ij})T_{ij}^{k} - k\eta'\right\}_{k \geq 1}
\end{align*}
Specifically, $T_{ij}^k$ is deterministic given the data in the $k$-th interval. Furthermore, from our choice of $\tau$, the following holds due to Lemma~\ref{appseclemma: separation2}:
\begin{align*}
    \E[B]{Z_{k} \middle | \mathcal{F}_{k-1}} & = \E[B]{Z_{k-1} + \operatorname{sign}(\beta_{ij})T_{ij}^{k} - \eta' \middle | \mathcal{F}_{k-1}} \\
    &  \geq Z_{k-1}
\end{align*}  
as  $\E[B]{\operatorname{sign}(\beta_{ij})T_{ij}^{k} - \eta' \middle | \mathcal{F}_{k-1}} > 0$. Therefore, $\{Z_{k}\}$ is the submartingale adapted to the filtration $\{\mathcal{F}_{k}\}_{k \geq 1}$. We will next show that given enough samples we can control the probability of error. To do this we note that the sequence has bounded difference (a.s.) which follows from Lemma~\ref{appseclemma: Test statistic bounded as}. Indeed, 
\begin{align*}
    \abs{\operatorname{sign}(\beta_{ij})T_{ij}^{k} - \eta'} \leq \abs{T_{ij}^{k}} + \eta' \leq 4\ymax^{2}(1 + 2q\tau d) = c.
\end{align*}

Next, we show in Lemma~\ref{appseclemma: variance bound} we have a bound on the variance of the appropriately conditioned submartingale sequence, i.e $\operatorname{var}\left(Z_{k} | \mathcal{F}_{k-1}\right) \leq 16\ymax^{4}q = s^{2}$. The probability of error can now be written as 
\begin{align*}
    \prob[B]{\abs{\frac{\sum_{k=1}^{k_{\max}}T_{ij}^{k}}{k_{\max}}} \leq \rho} & \leq \prob[B]{\frac{\sum_{k=1}^{k_{\max}}\operatorname{sign}(\beta_{ij})T_{ij}^{k}}{k_{\max}} \leq \rho}.
\end{align*}
Multiplying $k_{\max}$, then subtracting $\eta'k_{\max}$ both sides allows us to apply the Bernstein style martingale concentration inequality (stated in Lemma~\ref{appseclemma: Bernstein style}) with $Z_0 = 0$. This gives us the following. 
\begin{align}
    \prob[B]{Z_{k} \leq (\rho - \eta') k_{\max}} & =  \prob[B]{Z_{k} \leq -\frac{\eta k_{\max}}{2}}\\
           &\leq \exp\left(-\frac{\eta^{2}k_{\max}^{2}/4}{2k_{\max}s^{2} + c\eta k_{\max}/6}\right) \\
        &= \exp\left(-\frac{k_{\max}}{8(s/\eta)^{2} + 2/3(c/\eta)}\right).
\end{align}
Taking the union bound over $p^{2}$ edges, setting $\delta/2$ as the upper bound and solving for $k_{\max}$we get, 
\begin{align*}
\left[8\left(\frac{s}{\eta}\right)^{2} + \frac{2}{3}\left(\frac{c}{\eta}\right)\right]\log\left(\frac{2p^{2}}{\delta}\right) 
        &\leq 16\left(\frac{1}{q\tau^{2}d^{2}} + \frac{1}{q\tau d}\right)\log\left(\frac{p}{\delta}\right)  \\ 
        &\leq 32\left(\frac{1}{q\tau^{2}d^{2}}\right)\log\left(\frac{p}{\delta}\right) \\
        & \leq k_{\max}.
    \end{align*}
The first inequality follows from the fact that $\log\left(\frac{2p^{2}}{\delta}\right) < 2\log\left(\frac{p}{\delta}\right)$ when $\delta < 1/2$ and using $(1) : \left(\frac{c}{\eta}\right) = 1/q\tau d + 2 < 2/q\tau d$ as both $q<1, \tau d < 1$ and $(2): \left(\frac{s}{\eta}\right)^{2} = \frac{q}{(q\tau d)^{2}} = \frac{1}{q\tau^{2}d^{2}}$.  To convert this bound for $k_{\max}$ to $T$ we use $k_{\max} = \floor{\frac{T}{\tau}}$ we get
\begin{align*}
    32\left(\frac{1}{q\tau d^{2}}\right)\log\left(\frac{p}{\delta}\right) 
        &\stackrel{(a)}{\leq} 2^{13}e\left(\frac{1}{\tau^{5}d^{2}}\right)\log\left(\frac{p} {\delta}\right) \\ 
        &\stackrel{(b)}{\leq} 2^{13}e \cdot d^{3}\left(\frac{12\ymax^{2}}{ \bmin \sigmamin} + 1\right)^{5}\log\left(\frac{p} {\delta}\right) \\ 
        &\stackrel{(c)}{\leq}\frac{2^{13}e24^{5}}{b^{4}} \cdot d^{3}\left(\frac{\ymax^{2}}{\bmin \sigmamin}\right)^{5}\log\left(\frac{p}{\delta}\right) \\
        &\stackrel{(d)}{\leq} \frac{C_{2}}{(\bmin\sigmamin)^5}d^{3}\log\left(\frac{p}{\delta}\right)^{5} < T.
\end{align*}
Recall that $q = [(1-e^{-\tau/4})e^{-\tau/4}]^{4}$. For (a) we use $1-e^{-x} > xe^{-a}$ for $x \in [0,a]$ with $\tau < 1/2$. Therefore $q > \tau^{4}e^{-0.5}e^{-\tau}/(4^{4})$ and $e^{\tau} < e^{1/2}$. For (c), recall the choice for $\tau = d^{-1}\left(\frac{12\ymax^{2}}{\bmin\sigmamin} + 1\right)^{-1} < d^{-1}\left(\frac{24\ymax^{2}}{\bmin\sigmamin}\right)^{-1}$. Finally for (d) using $\ymax = C_{1}\sigma_{\max}\sqrt{\log\left(\frac{p}{\delta}\right)}$ we get the final bound for $T$.

\underline{\bf Case 2:  $\{i,j\} \notin E$.} Consider the sequences (1): $\{Z_{k}\}_{k \geq 1} = \{\sum_{l=1}^{k} T_{ij}^{l} - k\eta\}_{k\geq1}$ and (2): $\{\tilde{Z}_{k}\}_{k \geq 1} = \{\sum_{l=1}^{k}T_{ij}^{l} + k\eta\}_{k \geq 1}$. From Lemma~\ref{appseclemma: separation2}, it again follows that these are supermartingale and submartingale sequences respectively, adapted to the filtration $\{\mathcal{F}_{k}\}_{k \geq 1}$ as defined before. Using $\eta' > \eta$, we indeed have the same bound as in case 1 for the difference sequences associated with $\{Z_{k}\}$ and $\{\tilde{Z}_{k}\}$. That is, 
\begin{align*}
    \abs{T_{ij}^{k} - \eta} \leq \abs{T_{ij}^{k}} + \abs{\eta}\leq \abs{T_{ij}^{k}} + \eta' \leq 4\ymax^{2} + \eta' = c.
\end{align*} 
From Lemma~\ref{appseclemma: variance bound}, we have $\var{Z_{k} | \mathcal{F}_{k-1}} = \var{\tilde{Z_{k}} | \mathcal{F}_{k-1}}=16\ymax^{4}q = s^{2}$.

The probability of error in this case can be written as,
\begin{align*}
    \prob[B]{\abs{\frac{\sum_{k=1}^{k_{\max}}T_{ij}^{k}}{k_{\max}}} > \rho} &= \prob[B]{Z_{k} > (\rho - \eta)k_{\max}} + \prob[B]{\tilde{Z}_{k} < (\eta -\rho)k_{\max}}
\end{align*}
and from Lemma~\ref{appseclemma: Bernstein style} we have,
\begin{align}
    \prob[B]{\abs{\frac{\sum_{k=1}^{k_{\max}}T_{ij}^{k}}{k_{\max}}} > \rho} 
        & \leq 2\exp\left(-\frac{\eta^{2}k_{\max}^{2}/4}{2k_{\max}s^{2} + c\eta k_{\max}/6}\right) 
        = 2\exp\left(-\frac{k_{\max}}{8(s/\eta)^{2} + 2/3(c/\eta)}\right).
\end{align}

Taking the union bound over $\binom{p}{2} \leq \frac{p^{2}}{2}$ edges, setting $\delta/2$ as the upper bound and solving for $k_{\max}$ we get the same bound for T as in case 1.

\end{proof}

\subsection{Proof of Theorem~\ref{thm:algorithm-2-pac}}\label{appsubsec:algorithm-2-proof}

\begin{theorem}[Sample complexity guarantee for \algtwo{}]
Consider observations $\mathbf{X}_0^{n-1}$ from the Gaussian Glauber dynamics above,
started at $\mathbf{X}^{(0)}=x$.  Suppose Assumptions~\ref{assumption: Bounded edge strength}--\ref{assumption: bounded degree and sample decay} hold.  Then for every $\delta\in(0,1)$,
\algtwo{} satisfies
\begin{align}
\prob{\widetilde{\mathcal A}_{\mathfrak b,\mathfrak t}(\mathbf{X}_0^{n-1})\neq G}
\le \delta,
\end{align}
if the number of samples satisfies
\begin{align}
    n = \mathcal O \left(\frac{d\log(p/\delta)\,p}{\kappa^2(1-r)}\log\!\left(\frac{d\log(p/\delta)\,p^{3/2}\sqrt{\log p}}{\kappa^2\delta}\right)\right)
\end{align}
and we choose the burn-in and thinning lengths as
\begin{align}
    \mathfrak b,\mathfrak t
    \ge
    C_{x,\Theta,\mathcal{A}}\,\frac{p}{1-r}
    \log\!\left(\frac{d\log(p/\delta)\,p^{3/2}\sqrt{\log p}}{\kappa^2\delta}\right),
\end{align}
for a constant $C_{x,\Theta,\mathcal{A}}$ depending on the initialization, model
parameters, and the universal constant of the base learner $\mathcal{A}$, but not on $p,d,\kappa$ or $\delta$ except through the displayed terms.
\end{theorem}

\begin{remark}
Equivalently, in normalized Glauber time, this corresponds to
\begin{align}
    T = \mathcal O \left(\frac{d\log(p/\delta)}{\kappa^2(1-r)}\log\left(\frac{d\log(p/\delta)\,p^{3/2}\sqrt{\log p}}{\kappa^2\delta}\right)\right)
\end{align}
time units.
\end{remark}

We prove the theorem in four steps. First, total variation reduces the error of
\algtwo{} to the error of the base i.i.d. algorithm.  Second, the total
variation distance between the burned-in/thinned chain and an i.i.d. chain is
decomposed into burn-in and thinning errors.  Third, these errors are bounded by
a total-variation mixing estimate for the Gaussian Gibbs sampler.  Finally, the
parameters are chosen so that the two sources of error sum to at most $\delta$.
The proofs of the supporting lemmas are deferred to Section~\ref{subsec:supporting-lemmas}.

The following lemma shows the first step:
\begin{lemma}[Total variation controls the change in error probability]
~\label{applem:tv-controls-error}
Let $\mathcal A$ be any structure learning algorithm. Then
\begin{align}
    &\abs{\prob{\mathcal A(\mathbf{Y}_0^{m-1})\neq G} - \prob{\mathcal A(\mathbf{Z}_0^{m-1})\neq G}} \le \norm{\mathcal L(\mathbf{Y}_0^{m-1}\mid \mathbf{X}^{(0)}=x)-\pi^{\otimes m}}.
\end{align}
\end{lemma}

Next, we decompose the total variation into two error terms; namely burn-in error and thinning error.
\begin{lemma}[Burn-in and thinning decomposition]
~\label{lem:burn-thin-decomposition}
For the thinned sequence in~\eqref{eq:burned-thinned-samples},
\begin{align}
&\norm{\mathcal L(\mathbf{Y}_0^{m-1}\mid \mathbf{X}^{(0)}=x)-\pi^{\otimes m}} \le \norm{K^{\mathfrak b}(x,\cdot)-\pi}
+
(m-1)\E[\mathbf{Z}\sim\pi]{\norm{K^{\mathfrak t}(\mathbf{Z},\cdot)-\pi}}.
\label{eq:burn-thin-decomposition}
\end{align}
\end{lemma}

We next record the mixing estimate used to choose $\mathfrak b$ and $\mathfrak t$.  We work on the product space $\mathbb{R}^{p}$, equip each site with the absolute-difference metric $\mathsf{d}(x,y) := |x-y|,$
and equip $\mathbb{R}^{p}$ with the product metric $\mathsf{d}_{L^1}(x,y) := \sum_{i=1}^p \mathsf{d}(x_i,y_i), x,y\in \mathbb{R}^{p}$.

For two probability measures $\nu_1$ and $\nu_2$ on $\mathbb{R}^{p}$, define the
$L^1$-Wasserstein distance by
\begin{align*}
W_{1,\mathsf{d}_{L^1}}(\nu_1,\nu_2)
:=
\inf_{(U,V)} \mathbb{E}\bigl[\mathsf{d}_{L^1}(U,V)\bigr],
\end{align*}
where the infimum runs over all couplings $(U,V)$ of $(\nu_1,\nu_2)$. By the
Kantorovich--Rubinstein duality,
\begin{align*}
W_{1,\mathsf{d}_{L^1}}(\nu_1,\nu_2)
=
\sup_{\|f\|_{\mathrm{Lip}(\mathsf{d}_{L^1})}\le 1}
\int f \, d(\nu_1-\nu_2), \qquad \norm{f}_{\operatorname{Lip}(\mathsf{d}_{L^1})} := \sup_{x \neq y} \frac{\abs{f(x) - f(y)}}{\mathsf{d}_{L^1}(x,y)}
\end{align*}

Finally, define the $d$-Dobrushin interdependence coefficients $C \coloneq (c_{ij})_{i,j = 1,\dots,p}$ as in \citep{wang2014convergence}.  Define the Dobrushin interdependence matrix norm
\begin{align*}
r := \max_{i \in [p]} \sum_{j \neq i} c_{ij}.
\end{align*}
The Dobrushin uniqueness condition requires $r < 1$ (Assumption~\ref{assumption: bounded degree and sample decay}). This is the row-sum (equivalently, $\ell_\infty$ operator-norm) form of the Dobrushin condition --- condition (H1) in~\citet{wang2017convergence}. An alternative column-sum form (H2), with $r_1 := \max_j \sum_i c_{ij}$, is also available there but is not the version we use below.

For the Gaussian conditional laws, the Dobrushin coefficient is exactly $c_{ij} = |\beta_{ij}|$ for $j \neq i$ (with $c_{ii} = 0$, since $\mu_i(\cdot \mid x)$ does not depend on $x^i$), because changing coordinate $j$ shifts the conditional mean of coordinate $i$ by $\beta_{ij}(u_j - v_j)$ while leaving the conditional variance fixed.

Under the Dobrushin uniqueness condition (H1),~\citet[Theorem~2.1(b)]{wang2017convergence} establishes the following Wasserstein convergence rate for the random-scan Gibbs sampler with one-step kernel $K$: for any initial distribution $\nu$ on $\mathbb{R}^p$ and any $t \ge 1$,
\begin{align}
W_{1,\mathsf{d}_{L^1}}(\nu K^t, \pi)
\le
p\left(1 - \frac{1-r}{p}\right)^t \max_{1 \le i \le p} \E[]{\mathsf{d}(X_i^{(0)}, Y_i^{(0)})},
\label{eq:wang-random-scan}
\end{align}
where $(X^{(0)}, Y^{(0)})$ is any coupling of $(\nu, \pi)$.

\begin{lemma}[Dobrushin contraction in Wasserstein distance]
\label{lem:dobrushin-wasserstein}
Under Assumption~\ref{assumption: bounded degree and sample decay}, for an initialization $\mathbf{X}^{(0)} \sim \delta_{x}$, $x \in \mathbb{R}^{p}$, and any $t \ge 1$,
\begin{align}
W_{1,\mathsf{d}_{L^1}}(\delta_{x} K^t,\pi)
\le
p\left(1-\frac{1-r}{p}\right)^t \left( \norm{x}_{\infty} + \Theta_{\min}^{-1/2}\sqrt{\frac{2}{\pi}}\right).
\label{eq:wasserstein-contraction}
\end{align}
\end{lemma}

It is well known that under the discrete site metric, i.e. $d(x,y) = \mathds{1}_{x \neq y}$ the Wasserstein distance and total variation distance are equivalent. However using this metric does not satisfy the Dobrushin uniqness condition ($r < 1$) which leads to a vacuous upper bound.

We will now define the notion of ``approximate-Lipschitzness'' under which this Wasserstein distance can be converted to total variation for a high-dimensional result.

\begin{definition}[Approximate Lipschitz continuity]
\label{def:approx-lipschitz}
For $L,\varepsilon\ge 0$, a function
$\Phi:\mathbb{R}^p\to\mathbb{R}$ is called $(L,\varepsilon)$-approximately
Lipschitz if
\begin{align}
\abs{\Phi(x)-\Phi(y)}
\le
L\,\mathsf{d}_{L^1}(x,y)+\varepsilon,
\qquad x,y\in\mathbb{R}^p.
\label{eq:approx-lipschitz}
\end{align}
\end{definition}

Thus ordinary Lipschitz continuity is the special case $\varepsilon=0$.
The parameter $\varepsilon$ measures a uniform defect: the function may fail to
be exactly Lipschitz, but only up to an additive error that is independent of
the pair $(x,y)$.  This is useful for Markov kernels because a transition block
may smooth a bounded measurable test function except on a small exceptional
event, such as the event that not all coordinates have been refreshed.

\begin{lemma}[Approximate Lipschitz smoothing]
\label{lem:approx-lipschitz-smoothing}
For $t \ge 1$, let $\delta(t) := pe^{-t/p}$.  Then for any bounded measurable $f : \mathbb{R}^p \to \mathbb{R}$, the function $K^t f$ is $(L_f,\, 2\norm{f}_\infty\,\delta(t))$-approximately Lipschitz, that is,
\begin{align}
\abs{K^t f(u)-K^t f(v)}
\le
L_f\,\mathsf d_{L^1}(u,v)
+
2\norm{f}_\infty\,\delta(t),
\label{eq:approx-lipschitz-smoothing}
\end{align}
where one may take
\begin{align}
L_f = \sqrt{\frac{2p}{\pi}}\,\norm{f}_\infty\,\sqrt{\Theta_{\max}}\,d\bmax.
\end{align}
\end{lemma}

Now using the above bounds we can select the burn-in and thinning parameters such that the respective approximation errors are small.

\begin{lemma}[Total variation mixing bound]
\label{lem:tv-mixing-bound}
Under the assumptions of Theorem~\ref{thm:algorithm-2-pac}:
\begin{enumerate}
    \item[\textup{(i)}] \textbf{(Deterministic start.)} There is a constant
    $C_{x,\Theta}$ depending on the initialization and model parameters but not on $p$, $\kappa$, or $\delta$, such that for every $\varepsilon\in(0,1)$, the choice
    \begin{align}
    t\ge
    C_{x,\Theta}\,\frac{p}{1-r}\log\!\left(\frac{p^{3/2}}{\varepsilon}\right)
    \label{eq:one-block-tv-choice}
    \end{align}
    implies $\norm{K^t(x,\cdot)-\pi}\le \varepsilon$.

    \item[\textup{(ii)}] \textbf{(Stationary start.)} There is a constant
    $C_{\Theta}$ depending on the model parameters but not on $p$, $\kappa$, or $\delta$, such that for every $\varepsilon\in(0,1)$, the choice
    \begin{align}
    t\ge
    C_{\Theta}\,\frac{p}{1-r}\log\!\left(\frac{p^{3/2}\sqrt{\log p}}{\varepsilon}\right)
    \label{eq:one-block-tv-choice-stationary}
    \end{align}
    implies $\E[\mathbf{Z}\sim\pi]{\norm{K^t(\mathbf{Z},\cdot)-\pi}}\le \varepsilon$.
\end{enumerate}
\end{lemma}

\begin{proof}[Proof of Theorem~\ref{thm:algorithm-2-pac}]
By Lemma~\ref{applem:tv-controls-error},
\begin{align}
    \prob{\widetilde{\mathcal A}_{\mathfrak b,\mathfrak t}(\mathbf{X}_0^{n-1})\neq G} &\le \prob{\mathcal A(\mathbf{Z}_0^{m-1})\neq G} + \norm{\mathcal L(\mathbf{Y}_0^{m-1}\mid \mathbf{X}^{(0)}=x)-\pi^{\otimes m}}.
\end{align}
By the DICE i.i.d.\ sample complexity~\citep{misra2020information}, choosing
$m\ge C_{\rm DICE}\,d\log(2p/\delta)/\kappa^{2}$ makes the first term at most $\delta/2$.
By Lemma~\ref{lem:burn-thin-decomposition}, the second term satisfies
\begin{align}
    \norm{\mathcal L(\mathbf{Y}_0^{m-1}\mid \mathbf{X}^{(0)}=x)-\pi^{\otimes m}}
    \le \norm{K^{\mathfrak b}(x,\cdot)-\pi} + (m-1)\,\E[\mathbf{Z}\sim\pi]{\norm{K^{\mathfrak t}(\mathbf{Z},\cdot)-\pi}}.
\end{align}

It remains to bound the right-hand side by $\delta/2$.  Apply
Lemma~\ref{lem:tv-mixing-bound}(i) with accuracy $\varepsilon = \delta/(2m)$ to the burn-in term
and Lemma~\ref{lem:tv-mixing-bound}(ii) to each of the $m-1$ thinning terms.  Since the stationary-start condition~\eqref{eq:one-block-tv-choice-stationary} is the more stringent of the two (it has the additional $\sqrt{\log p}$ factor), both are ensured by
\begin{align*}
    \mathfrak b,\mathfrak t \ge C_{\Theta}\,\frac{p}{1-r}\log\!\left(\frac{2C_{\mathrm{DICE}}\,d\log(p/\delta)\,p^{3/2}\sqrt{\log p}}{\kappa^2\delta}\right),
\end{align*}
where the factor $2C_{\mathrm{DICE}}$ inside the logarithm arises from substituting $\varepsilon = \delta/(2m)$ with $m = C_{\mathrm{DICE}}\,d\log(p/\delta)/\kappa^2$.  Since $\log(2C_{\mathrm{DICE}})$ is a universal constant, the same absorption argument as in the proof of Lemma~\ref{lem:tv-mixing-bound} yields
\begin{align}
    \mathfrak b,\mathfrak t \ge C_{x,\Theta,\mathcal{A}}\,\frac{p}{1-r}\log\!\left(\frac{d\log(p/\delta)\,p^{3/2}\sqrt{\log p}}{\kappa^2\delta}\right),
\end{align}
where $C_{x,\Theta,\mathcal{A}} := (1 + \log(2C_{\mathrm{DICE}}))\max(C_{x,\Theta},\, C_{\Theta})$ combines the mixing constants from Lemma~\ref{lem:tv-mixing-bound} with the universal constant of DICE.  This gives
\begin{align}
    \norm{K^{\mathfrak b}(x,\cdot)-\pi} + (m-1)\,\E[\mathbf{Z}\sim\pi]{\norm{K^{\mathfrak t}(\mathbf{Z},\cdot)-\pi}} \le \frac{\delta}{2m} + (m-1)\cdot\frac{\delta}{2m} \le \frac{\delta}{2}.
\end{align}
Combining the two bounds yields
\begin{align}
    \prob{\widetilde{\mathcal A}_{\mathfrak b,\mathfrak t}(\mathbf{X}_0^{n-1})\neq G} \le \delta.
\end{align}
The observation length condition follows from the requirement that all retained samples lie in the observed trajectory: $n\ge \mathfrak b+(m-1)\mathfrak t+1$.
\end{proof}

\subsection{Proofs of supporting lemmas}
\label{subsec:supporting-lemmas}

\begin{proof}[Proof of Lemma~\ref{applem:tv-controls-error}]
Let $\mu_1 = \mathcal L(\mathbf{Y}_0^{m-1}\mid \mathbf{X}^{(0)}=x)$ and $\mu_2 = \pi^{\otimes m}$.
Define the measurable set $A = \{y_0^{m-1} \in \mathbb{R}^{p \times m} : \mathcal A(y_0^{m-1})\neq G\}$.
Then
\begin{align*}
    \abs{\prob{\mathcal A(\mathbf{Y}_0^{m-1})\neq G} - \prob{\mathcal A(\mathbf{Z}_0^{m-1})\neq G}}
    &= \abs{\mu_1(A) - \mu_2(A)} \\
    &\le \sup_{B \text{ measurable}} \abs{\mu_1(B) - \mu_2(B)} \\
    &= \norm{\mu_1 - \mu_2} \\
    &= \norm{\mathcal L(\mathbf{Y}_0^{m-1}\mid \mathbf{X}^{(0)}=x)-\pi^{\otimes m}},
\end{align*}
where the inequality follows because $A$ is one particular measurable set, and the last line is the definition of total variation.
\end{proof}

\begin{proof}[Proof of Lemma~\ref{lem:burn-thin-decomposition}]
Define hybrid distributions $H^{(0)},H^{(1)},\ldots,H^{(m)}$ on $(\mathbb{R}^p)^m$ as follows.  Let $H^{(0)} = \mathcal L(\mathbf{Y}_0^{m-1}\mid \mathbf{X}^{(0)}=x)$ be the law of the burned-in and thinned chain, and let $H^{(m)} = \pi^{\otimes m}$ be the i.i.d.\ target law.  For $1 \le i \le m-1$, let $H^{(i)}$ be the law of the sequence whose first $i$ components are drawn i.i.d.\ from $\pi$, and whose remaining $m-i$ components are drawn from the chain started at stationarity: concretely, $(\mathbf{W}_0,\ldots,\mathbf{W}_{i-1}) \sim \pi^{\otimes i}$ independently, $\mathbf{W}_i \sim K^{\mathfrak t}(\mathbf{Z}, \cdot)$ with $\mathbf{Z} \sim \pi$, and $\mathbf{W}_{i+s} = \mathbf{X}^{(\mathfrak t \cdot s)}$ for $s \ge 1$ continuing the chain from $\mathbf{W}_i$.

By the triangle inequality for total variation,
\begin{align*}
\norm{H^{(0)} - H^{(m)}} \le \sum_{i=0}^{m-1} \norm{H^{(i)} - H^{(i+1)}}.
\end{align*}

\textbf{The $i=0$ term (burn-in).}
The distributions $H^{(0)}$ and $H^{(1)}$ differ only in the marginal law of the first component: under $H^{(0)}$ it is $K^{\mathfrak b}(x,\cdot)$, while under $H^{(1)}$ it is $\pi$.  Conditioned on the first component, the remaining components evolve identically in both cases (via the same chain with thinning gap $\mathfrak t$).  Therefore, by the data-processing inequality for total variation,
\begin{align*}
\norm{H^{(0)} - H^{(1)}} \le \norm{K^{\mathfrak b}(x,\cdot) - \pi}.
\end{align*}

\textbf{The $i \ge 1$ terms (thinning).}
For $i \ge 1$, the distributions $H^{(i)}$ and $H^{(i+1)}$ share the same marginal on the first $i$ components (all i.i.d.\ from $\pi$) and the same conditional evolution after the $(i+1)$th component.  They differ only in the conditional law of the $(i+1)$th component given the $i$th: under $H^{(i)}$ it is $K^{\mathfrak t}(\mathbf{Z},\cdot)$ with $\mathbf{Z} \sim \pi$, while under $H^{(i+1)}$ it is $\pi$.  By the data-processing inequality,
\begin{align*}
\norm{H^{(i)} - H^{(i+1)}} \le \E[\mathbf{Z} \sim \pi]{\norm{K^{\mathfrak t}(\mathbf{Z},\cdot) - \pi}}.
\end{align*}

Summing the $m$ terms gives
\begin{align*}
\norm{H^{(0)} - H^{(m)}} &\le \norm{K^{\mathfrak b}(x,\cdot) - \pi} + (m-1)\E[\mathbf{Z} \sim \pi]{\norm{K^{\mathfrak t}(\mathbf{Z},\cdot) - \pi}}.  \qedhere
\end{align*}
\end{proof}

\begin{proof}[Proof of Lemma~\ref{lem:dobrushin-wasserstein}]
Specialize~\eqref{eq:wang-random-scan} to $\nu = \delta_{x}$. The optimal coupling takes $\mathbf{Y} \sim \pi$, so
$\max_i \E[]{\mathsf{d}(X_i^{(0)}, Y_i^{(0)})} = \max_i \E[\mathbf{Y} \sim \pi]{|x_i - Y_i|} \le \|x\|_\infty + \Theta_{\min}^{-1/2}\sqrt{2/\pi}$,
where the last step uses the triangle inequality and the bound $\E[]{|Y_i|} \le \Theta_{\min}^{-1/2}\sqrt{2/\pi}$ on the mean absolute value of a zero-mean Gaussian with variance at most $\Theta_{\min}^{-1}$.
\end{proof}

\subsubsection*{From Wasserstein to total variation via approximate Lipschitzness}

\begin{proof}[Proof of Lemma~\ref{lem:approx-lipschitz-smoothing}]
Given a sequence of updates $\sigma := \left(I^{(1)}, I^{(2)}, \dots, I^{(t)}\right),$
let $\mathcal{I}_t := \{i \in [p] : I^{(k)} = i \text{ for some } 1 \le k \le t \}$
be the set of unique coordinates updated up to round $t$. Define the covering event by
\begin{align*}
    \mathcal{C}_t = \{\omega \, | \, \mathcal{I}_t(\omega) = \{1,2,\dots,p\}\}.
\end{align*}
For any bounded $f : \mathbb{R}^{p} \to \mathbb{R}$ let
\begin{align*}
\widetilde{K}^{t}f(x) := \E[x]{f(\mathbf{X}^{(t)}) \cdot \mathds{1}_{\mathcal{C}_t}}
\qquad\text{and}\qquad
R^t f(x) := \E[x]{f(\mathbf{X}^{(t)}) \cdot \mathds{1}_{\mathcal{C}_t^c}}.
\end{align*}
Then we can write
\begin{align*}
    \abs{K^{t}f(x) - K^{t}f(y)} & \leq \abs{\widetilde{K}^t f(x)-\widetilde{K}^t f(y)} + \abs{R^t f(x)-R^t f(y)} \\
    & \leq \abs{\widetilde{K}^t f(x)-\widetilde{K}^t f(y)} + 2\norm{f}_{\infty} \, p e^{-t/p}.
\end{align*}
The second inequality follows from $\abs{R^t f(x)} \leq \norm{f}_{\infty}$ and from well-known results for the coupon collector problem. Next we show that $\widetilde{K}^t f$ is $L$-Lipschitz.

Now let $x,y \in \mathbb{R}^p$ be two initialization points and let the chains started from these points be $\mathbf{X}^{(t)}(x),\mathbf{X}^{(t)}(y), t \geq 0$ respectively. Set $\delta := y-x$. Assume that the schedule $\sigma$ touches every coordinate at least once, and let $\tau_i$ denote the last update time of coordinate $i$. Relabel the coordinates as $J=\{j_1,\dots,j_p\}$ so that
$$
\tau_{j_1} < \tau_{j_2} < \cdots < \tau_{j_p}.
$$

Recall that after $t$ updates following $\sigma$, $\mathbf{X}^{(t)}(x) = M_\sigma x + V_\sigma \mathbf{Z}$. We construct a shifted noise vector $h=(h^{(1)},\dots,h^{(t)})^\top$ so that the chain started from $y$ with noise $\mathbf{Z}+h$ agrees at time $t$ with the chain started from $x$ with noise $\mathbf{Z}$. Define $\Delta_0=\delta$ and, whenever $I^{(s)}=i$, let
$$
\Delta_s = M_i \Delta_{s-1} + \Theta_{ii}^{-1/2} e_i h^{(s)}.
$$

We choose $h$ inductively by

$$
h^{(s)} =
\begin{cases}
0, & s \notin \{\tau_{j_1},\dots,\tau_{j_p}\}, \\[0.5ex]
-\sqrt{\Theta_{j_k j_k}}\,\bigl(M_{j_k}\Delta_{\tau_{j_k}-1}\bigr)_{j_k},
& s=\tau_{j_k},\ k\in [p].
\end{cases}
$$
Equivalently, at the last update time of coordinate $j_k$,
$$
\Delta_{\tau_{j_k}}
= M_{j_k}\Delta_{\tau_{j_k}-1}
- e_{j_k}\bigl(M_{j_k}\Delta_{\tau_{j_k}-1}\bigr)_{j_k}.
$$
In particular, $\bigl(\Delta_{\tau_{j_k}}\bigr)_{j_k} = 0$. We claim that after the $k$th last-touch time, the coordinates $j_1,\dots,j_k$ have all been eliminated:

$$
\bigl(\Delta_{\tau_{j_k}}\bigr)_{j_\ell}=0,
\qquad \ell=1,\dots,k.
$$

This is immediate for $k=1$. Assume it holds for $k-1$. Between times $\tau_{j_{k-1}}+1$ and $\tau_{j_k}-1$, none of the coordinates $j_1,\dots,j_{k-1}$ are updated again, because their last update times are earlier than $\tau_{j_k}$. Since each single-site update $M_i$ only changes coordinate $i$, the entries $j_1,\dots,j_{k-1}$ of $\Delta_s$ remain zero throughout this interval. At time $\tau_{j_k}$ we additionally set the $j_k$th coordinate to zero by the above choice of $h^{(\tau_{j_k})}$. This proves the claim by induction.

Taking $k=p$, we obtain

$$
\Delta_{\tau_{j_p}} = 0.
$$

Since $\tau_{j_p}$ is the last update time among all coordinates, we have $\tau_{j_p}=t$, hence $\Delta_t=0$. Therefore the two trajectories coalesce at time $t$:

$$
M_\sigma y + V_\sigma (\mathbf{Z}+h) = M_\sigma x + V_\sigma \mathbf{Z}.
$$

Equivalently, on any schedule that updates every coordinate at least once, the effect of changing the initial condition from $x$ to $y$ can be absorbed into a deterministic shift of the Gaussian innovation vector. Define the affine map $T:\mathbb{R}^{t} \to \mathbb{R}^{p}$ by
\[
T(z):= M_{\sigma}x + V_{\sigma}z.
\]
Then
\begin{align}
    \mathbf{X}^{(t)}(x) \sim T_{\#}(\mathcal{N}(0, I_{t})) \text{ and }\mathbf{X}^{(t)}(y) \sim T_{\#}(\mathcal{N}(-h, I_{t})).
\end{align}
Since total variation distance contracts under measurable pushforwards,
\begin{align*}
    \norm{T_{\#}(\mathcal{N}(0, I_{t})) - T_{\#}(\mathcal{N}(-h, I_{t}))}
    &\leq \norm{\mathcal{N}(0, I_{t}) - \mathcal{N}(-h, I_{t})} \\
    &\leq \frac{\norm{h}_{2}}{\sqrt{2\pi}}.
\end{align*}
The last inequality is the standard total variation bound for Gaussian measures with the same covariance.

Therefore, for any bounded $f$ and any covering schedule $\sigma$,
\begin{align}
\abs{\E[]{f(\mathbf{X}^{(t)}(x)) \mid \sigma} - \E[]{f(\mathbf{X}^{(t)}(y)) \mid \sigma}}
&\leq 2\norm{f}_{\infty}
\norm{T_{\#}(\mathcal{N}(0, I_{t})) - T_{\#}(\mathcal{N}(-h, I_{t}))} \\
&\leq \sqrt{\frac{2}{\pi}}\norm{f}_{\infty}\norm{h}_{2} \\
&\leq \sqrt{\frac{2p}{\pi}}\norm{f}_{\infty}\norm{h}_{\infty}.
\end{align}

Now it remains to bound $\norm{h}_{\infty}$. For $s=\tau_{j_k}$, the definition of $h$ and the form of the single-site update give
\begin{align*}
    h^{(\tau_{j_k})}
    &= -\sqrt{\Theta_{j_k j_k}}\,\bigl(M_{j_k}\Delta_{\tau_{j_k}-1}\bigr)_{j_k} \\
    &= \sqrt{\Theta_{j_k j_k}}\sum_{m \neq j_k} \beta_{j_k m}\bigl(\Delta_{\tau_{j_k}-1}\bigr)_m.
\end{align*}
Hence
\begin{align*}
    \abs{h^{(\tau_{j_k})}}
    &\leq \sqrt{\Theta_{j_k j_k}} \sum_{m \neq j_k} \abs{\beta_{j_k m}}\, \norm{\Delta_{\tau_{j_k}-1}}_{\infty} \\
    &\leq \sqrt{\Theta_{\max}}\, d\, \bmax \, \norm{\Delta_{\tau_{j_k}-1}}_{\infty}.
\end{align*}
It therefore suffices to show that $\norm{\Delta_s}_{\infty}$ is nonincreasing in $s$.

If $h^{(s)}=0$ and $I^{(s)}=i$, then $\Delta_s = M_i\Delta_{s-1}$. Since $M_i$ only changes coordinate $i$,
\[
(\Delta_s)_j = (\Delta_{s-1})_j, \qquad j \neq i,
\]
and
\[
(\Delta_s)_i = \sum_{m \neq i} \beta_{im}(\Delta_{s-1})_m.
\]
Therefore,
\[
\abs{(\Delta_s)_i}
\leq \sum_{m \neq i} \abs{\beta_{im}}\,\norm{\Delta_{s-1}}_{\infty}
\leq d\,\bmax\,\norm{\Delta_{s-1}}_{\infty}
\leq \norm{\Delta_{s-1}}_{\infty},
\]
where the last step uses the standing assumption $d\,\bmax < 1$.

If $h^{(s)} \neq 0$, then necessarily $s=\tau_{j_k}$ for some $k$, and by construction
\[
\Delta_s = M_{j_k}\Delta_{s-1} - e_{j_k}\bigl(M_{j_k}\Delta_{s-1}\bigr)_{j_k}.
\]
Thus $(\Delta_s)_{j_k}=0$, while the remaining coordinates are unchanged from $\Delta_{s-1}$. Hence
\[
\norm{\Delta_s}_{\infty} \leq \norm{\Delta_{s-1}}_{\infty}.
\]
Combining the two cases yields
\[
\norm{\Delta_s}_{\infty} \leq \norm{\Delta_{s-1}}_{\infty} \leq \cdots \leq \norm{\delta}_{\infty}.
\]
Consequently,
\[
\norm{h}_{\infty} \leq \sqrt{\Theta_{\max}}\, d\, \bmax \, \norm{\delta}_{\infty}.
\]

Now taking a expectation over the covering schedules $\sigma$ we get,
\begin{align*}
\abs{{\E[]{f(\mathbf{X}^{(t)}(x))}} - {\E[]{f(\mathbf{X}^{(t)}(y))}}}
& \leq \E[\sigma]{\abs{{\E[]{f(\mathbf{X}^{(t)}(x)) \mid \sigma}} - {\E[]{f(\mathbf{X}^{(t)}(y)) \mid \sigma}}}
} \\
& \leq \sqrt{\frac{2p}{\pi}}\norm{f}_{\infty}\sqrt{\Theta_{\max}}\, d\, \bmax \, \norm{\delta}_{\infty}.
\end{align*}

Let $\mathsf{d}(u,v) = |u-v|$ be the site metric on $\mathbb{R}$, and $\mathsf{d}_{L^1}(x,y) := \sum_{i=1}^p |x_i-y_i| = \|x-y\|_1, \qquad x,y \in \mathbb{R}^p.
$ Therefore, for any bounded $f$, large enough $t$, we have
\begin{align*}
\abs{K^{t}f(x) - K^{t}f(y)} & \leq \sqrt{\frac{2p}{\pi}}\norm{f}_{\infty}\sqrt{\Theta_{\max}}\, d\, \bmax \, \norm{x - y}_{\infty} + 2\norm{f}_{\infty} \, pe^{-t/p} \\
& \leq L \mathsf{d}_{L^1}(x,y) + 2 \norm{f}_{\infty} \delta(t)
\end{align*}
where, $L = \sqrt{\frac{2p}{\pi}}\norm{f}_{\infty}\sqrt{\Theta_{\max}}\, d\, \bmax$. This proves the Theorem.
\end{proof}

\subsubsection*{Proof of the total variation mixing bound}

\begin{proof}[Proof of Lemma~\ref{lem:tv-mixing-bound}]
Let $f$ be any measurable function with $\|f\|_\infty\le 1$.  Split $t=k+s$.
We write
\begin{align*}
\abs{K^{k+s}f(x)-\pi f}
= \abs{\int K^s f \, d(\delta_{x} K^k) - \int K^s f \, d\pi}.
\end{align*}
By Lemma~\ref{lem:approx-lipschitz-smoothing}, $K^s f$ is $(L_f,\, 2\delta(s))$-approximately Lipschitz with $L_f = \sqrt{2p/\pi}\,\sqrt{\Theta_{\max}}\,d\bmax$ and $\delta(s) = pe^{-s/p}$.  This means $g := K^s f$ can be decomposed as $g = g_{\mathrm{Lip}} + g_{\mathrm{err}}$ where $g_{\mathrm{Lip}}$ is $L_f$-Lipschitz with respect to $\mathsf{d}_{L^1}$ and $\|g_{\mathrm{err}}\|_\infty \le 2\delta(s)$.  Therefore
\begin{align*}
\abs{\int g \, d(\delta_{x} K^k) - \int g \, d\pi}
&\le \abs{\int g_{\mathrm{Lip}} \, d(\delta_{x} K^k - \pi)} + 2\|g_{\mathrm{err}}\|_\infty \\
&\le L_f \, W_{1,\mathsf{d}_{L^1}}(\delta_{x} K^k, \pi) + 2\cdot 2\delta(s),
\end{align*}
where the first term uses the Kantorovich--Rubinstein duality.  Applying Lemma~\ref{lem:dobrushin-wasserstein} to the Wasserstein term,
\begin{align}
    \abs{K^{k+s}f(x)-\pi f}
    &\le \underbrace{\sqrt{\tfrac{2}{\pi}}\,\sqrt{\Theta_{\max}}\,d\bmax\bigl(\norm{x}_\infty + \Theta_{\min}^{-1/2}\sqrt{\tfrac{2}{\pi}}\bigr)}_{=:\,\gamma_{x}}\, p^{3/2}\left(1-\frac{1-r}{p}\right)^k + 4pe^{-s/p}.
    \label{eq:mixing-proof-bound}
\end{align}
Here $\gamma_{x}$ collects all model- and initialization-dependent quantities: the Lipschitz constant $L_f/\sqrt{p}$ (with $\|f\|_\infty = 1$) and the initial Wasserstein coupling distance from Lemma~\ref{lem:dobrushin-wasserstein}, with the factor $p^{3/2}$ displayed explicitly.  Taking the supremum over $\|f\|_\infty\le1$ gives
\begin{align*}
    \norm{K^t(x,\cdot)-\pi} \le \gamma_{x}\, p^{3/2}\left(1-\frac{1-r}{p}\right)^k + 4pe^{-s/p}.
\end{align*}

We now choose $k$ and $s$.  For the first term to be at most $\varepsilon/2$, it suffices to take
\begin{align*}
    k \ge \frac{p}{1-r}\log\!\left(\frac{2\gamma_{x}\, p^{3/2}}{\varepsilon}\right).
\end{align*}
For the second term to be at most $\varepsilon/2$, it suffices to take
\begin{align*}
    s \ge p\log\!\left(\frac{8p}{\varepsilon}\right).
\end{align*}
Since $\frac{p}{1-r} \ge p$ and $\log(p/\varepsilon) \le \log(p^{3/2}/\varepsilon)$, we can upper-bound $s \le \frac{p}{1-r}[\log 8 + \log(p^{3/2}/\varepsilon)]$.  Combining with $k$,
\begin{align*}
    t = k + s \ge \frac{p}{1-r}\bigl[\log(16\gamma_{x}) + 2\log\tfrac{p^{3/2}}{\varepsilon}\bigr].
\end{align*}
Since $\varepsilon < 1$ and $p \ge 2$, we have $\log(p^{3/2}/\varepsilon) \ge 1$, so $\log(16\gamma_{x}) \le \log(16\gamma_{x}) \cdot \log(p^{3/2}/\varepsilon)$.  Therefore both conditions are satisfied by
\begin{align*}
    t \ge C_{x,\Theta}\,\frac{p}{1-r}\log\!\left(\frac{p^{3/2}}{\varepsilon}\right),
\end{align*}
where
\begin{align}
    C_{x,\Theta} := 2 + \log(16\gamma_{x})
    \label{eq:C-constant-def}
\end{align}
and $\gamma_{x} = \sqrt{2/\pi}\,\sqrt{\Theta_{\max}}\,d\bmax\bigl(\norm{x}_\infty + \Theta_{\min}^{-1/2}\sqrt{2/\pi}\bigr)$ depends only on the initialization and model parameters (not on $p$, $\kappa$, or $\delta$).  This proves part~(i).

\medskip\noindent\textbf{Part~(ii): Stationary start.}
The same argument applies with $x$ replaced by a draw $\mathbf{Z} \sim \pi$.  The Wasserstein coupling distance in Lemma~\ref{lem:dobrushin-wasserstein} replaces $\norm{x}_\infty$ with $\norm{\mathbf{Z}}_\infty$.  Taking expectations, $\gamma_{x}$ is replaced by
\begin{align*}
    \bar\gamma := \sqrt{\tfrac{2}{\pi}}\,\sqrt{\Theta_{\max}}\,d\bmax\bigl(\E[\mathbf{Z}\sim\pi]{\norm{\mathbf{Z}}_\infty} + \Theta_{\min}^{-1/2}\sqrt{\tfrac{2}{\pi}}\bigr).
\end{align*}
Since each $Z_i$ is zero-mean with variance at most $\Theta_{\min}^{-1}$, the sub-Gaussian maximal inequality gives $\E[]{\norm{\mathbf{Z}}_\infty} \le \Theta_{\min}^{-1/2}\sqrt{2\log(2p)}$.  Consequently,
\begin{align*}
    \bar\gamma \le \bar\gamma_0\,\sqrt{\log p},
\end{align*}
where $\bar\gamma_0 := 2\sqrt{2/\pi}\,\sqrt{\Theta_{\max}}\,d\bmax\,\Theta_{\min}^{-1/2}(1 + \Theta_{\min}^{-1/2}\sqrt{2/\pi})$ is independent of~$p$.  Substituting $\bar\gamma$ for $\gamma_{x}$ in~\eqref{eq:mixing-proof-bound} gives the condition
\begin{align*}
    k \ge \frac{p}{1-r}\log\frac{2\bar\gamma\, p^{3/2}}{\varepsilon}
    \le \frac{p}{1-r}\biggl[\log(2\bar\gamma_0) + \log\frac{p^{3/2}\sqrt{\log p}}{\varepsilon}\biggr].
\end{align*}
Repeating the same combination as in part~(i), with $\log(p^{3/2}\sqrt{\log p}/\varepsilon)$ in place of $\log(p^{3/2}/\varepsilon)$, yields part~(ii) with constant $C_{\Theta} := 2 + \log(16\bar\gamma_0)$.
\end{proof}

\subsection{Theorem~\ref{thm:lower-bound-glauber}: Statement \& Proof}\label{appsubsec: theorem 2 statement and proof}
In this section, we present the proof for the minimax lower bounds on the observation time $T$ required to learn the Gaussian graphical model from the Glauber dynamics $\{\mathbf{Y}^{(t)}\}_{t=0}^{T}$. We start by introducing a natural class of graphical models for which we can provide these lower bounds. We then define the minimax risk and state Theorem 2, followed by a detailed presentation of its proof.

Let $\mathcal{G}_{p,d}$ denote the class of undirected graphs with vertex set $[p]$ and maximum degree $d \leq p$. For a graph $G \in \mathcal{G}_{p,d}$, let $\Theta(G)$ denote any symmetric positive definite precision matrix whose graph is $G$, namely
\begin{equation*}
\Theta(G)_{ij} \neq 0 \quad \Longleftrightarrow \quad \{i,j\} \in E, \qquad i \neq j.
\end{equation*}
Let $\Sigma(G) := \Theta(G)^{-1}$ be the corresponding covariance matrix. With a slight abuse of notation, $\Theta(G)$ and $\Sigma(G)$ are not unique given $G$. We let $\mathcal{G}_{p,d}(\bmin,\theta_{\min},\theta_{\max})$ denote the set of pairs of probability distributions representing target and initial distribution $(\phi_{\Theta(G)},q)$ such that $G \in \mathcal{G}_{p,d}$ and $\Theta(G)$ is any precision matrix with graph $G$ satisfying: 
\begin{itemize}
\item For every $(\phi_{\Theta(G)}, q) \in \mathcal{G}_{p,d}(\bmin,\theta_{\min},\theta_{\max})$,
\begin{equation*}
\beta^*(\Theta)=\min_{\{i,j\}\in E}\frac{|\Theta_{ij}|}{\Theta_{ii}} \ge \bmin,
\end{equation*}
and the diagonal entries satisfy $\theta_{\min} \le \Theta_{ii} \le \theta_{\max}$ for all $i$.
\end{itemize}

We study a hard subclass $\mathcal{H}_{p,d}(\bmin,\theta_{\min},\theta_{\max}) \subset \mathcal{G}_{p,d}(\bmin,\theta_{\min},\theta_{\max})$ with initial distributions being the same as the target distribution (stationary start) and precision matrix of the target distribution additionally satisfying constant diagonal entry $\Theta_{ii} = \theta \in [\theta_{\min}, \theta_{\max}]$ for all $i \in V$.

\subsubsection{Graph Decoders and Maximum Probability of Error}

Let $\operatorname{GD}_{(n)}(\phi_{\Theta(G)})$ represent the joint probability distribution of observing $n$ vector samples from a Glauber dynamics process with target distribution $\phi_{\Theta(G)} \in \mathcal{H}_{p,d}(\beta_{\min}, \theta_{\min}, \theta_{\max})$ with a stationary start. Suppose we are given $n$ vector samples $\mathbf{X}_{0}^{n-1} = (\mathbf{X}^{(0)}, \mathbf{X}^{(1)}, \dots, \mathbf{X}^{(n-1)}) \in \mathbb{R}^{n \times p}$ from a stationary Glauber dynamics process $\operatorname{GD}_{(n)}(\phi_{\Theta(G)})$ with unknown targets. Our goal is to estimate the underlying graph $G \in \mathcal{G}_{p,d}$ associated with $\Theta(G)$ based on the observations $\mathbf{X}_{0}^{n-1}$.

To prove our minimax lower bound we consider decoders $\psi: \mathbb{R}^{n \times p} \to \mathcal{G}_{p,d}$ that map the observations $\mathbf{X}_{0}^{n-1}$ to an estimate $\widehat{G} = \psi(\mathbf{X}_{0}^{n-1})$. Any information about the update sequence that is identifiable from the observed trajectory
  $\mathbf{X}_0^{n-1}$ may be used implicitly by the decoder. In particular, under Gaussian Glauber
  dynamics, each update resamples exactly one coordinate from a non-degenerate conditional
  Gaussian distribution. Hence, with probability one, $\mathbf{X}^{(t)}$ and $\mathbf{X}^{(t-1)}$ differ in exactly
  one coordinate, so the identity of the updated node at time $t$ is determined by the observed
  pair $(\mathbf{X}^{(t-1)},\mathbf{X}^{(t)})$. The performance is evaluated using the 0-1 loss function $\mathds{1}[\psi(\mathbf{X}_{0}^{n-1}) \neq G]$, indicating whether the estimated graph differs from the true graph. For a decoder, the maximal expected risk over all admissible target distributions $\phi_{\Theta(G)}$ in the class $\mathcal{H}_{p,d}(\bmin, \theta_{\min}, \theta_{\max})$ is given by,
\begin{align}\label{eqn: maximal probability of error}
    \mathcal{R}_{n,p,d}(\psi; \mathcal{H}_{p,d}(\bmin, \theta_{\min}, \theta_{\max})) &= \max_{\phi_{\Theta(G)}} \E[\Theta(G)]{\mathds{1}[\psi(\mathbf{X}_{0}^{n-1}) \neq G]} \nonumber \\
    &= \max_{\phi_{\Theta(G)}} \prob[\Theta(G)]{\psi(\mathbf{X}_{0}^{n-1}) \neq G} = p_{err}(\psi)
\end{align}

The expectation is taken with respect to the joint probability measure $\prob[\Theta(G)]{\cdot} = \operatorname{GD}_{(n)}(\cdot;\phi_{\Theta(G)})$. The minimax risk is defined as
\begin{align*}
    \mathcal{M}_{n,p,d}(\mathcal{H}_{p,d}(\bmin, \theta_{\min}, \theta_{\max})) = \min_{\psi} \max_{\phi_{\Theta(G)}} \prob[\Theta(G)]{\psi(\mathbf{X}_{0}^{n-1}) \neq G}
\end{align*}

The minimax risk represents the minimum error achieved by the 
best estimator $\widehat{G}$ in the worst-case scenario, where distributions on graphs are hard to distinguish. Our main goal is to understand the scaling laws of the triplet $(n,p,d)$—where $n$ is the number of updates or samples observed, $d$ is the maximum degree of the graph, and $p$ is the number of vertices—such that the minimax risk $\mathcal{M}_{n,p,d}$ vanishes asymptotically.

An alternative viewpoint, which we incorporate to maintain consistency with the achievability theorem's style, is to consider the time interval $T$ over which the continuous-time Glauber process $\{\mathbf{Y}^{(t)}\}_{t=0}^{T}$ was observed. These two perspectives are equivalent in the sense that $\lfloor Tp \rfloor = n$, as there are, on average, $p$ updates over a unit interval of time. Recall that in Section~\ref{subsec: Glauber dynamics}, we treat the number of samples observed over $T$ time units to be deterministic. Therefore we denote it with lower case $n$. Moreover, as noted earlier, $\{\mathbf{Y}^{(t)}\}$ and the input data sequence $\{\mathcal{X}^{(n)}, \mathcal{I}^{(n)}\}$, where $\mathcal{I}^{(n)} = (I^{(1)},\dots,I^{(n)})$ can be used interchangeably.

Note that the minimax risk over the hard subclass lower bounds the minimax risk over the general class. That is, 
\begin{align*}
    \mathcal{M}_{n,p,d}(\mathcal{G}_{p,d}(\bmin, \theta_{\min}, \theta_{\max})) \geq \mathcal{M}_{n,p,d}(\mathcal{H}_{p,d}(\bmin, \theta_{\min}, \theta_{\max})).
\end{align*}

By using Fano's method, we will characterize $n_{\text{lb}}(p,d)$ such that for $n < n_{\text{lb}}(p,d)$ the minimax risk $\mathcal{M}_{n,p,d}(\mathcal{H}) \geq 1/2$ which implies $\mathcal{M}_{n,p,d}(\mathcal{G}) \geq 1/2$. Equivalently if, $\mathcal{M}_{n,p,d}(\mathcal{G}) < 1/2$, then $n \geq n_{\text{lb}}(p,d)$.

\subsubsection{Fano's Method}
We employ Fano's method to derive information-theoretic lower bounds on the observation time $T$. This approach involves constructing a restricted ensemble of graphical models such that for any graph decoder $\psi$, the maximum probability of error, as specified in equation $\ref{eqn: maximal probability of error}$, is bounded from below by $1/2$. The construction of this restricted ensemble and the associated proof strategy is a modification of the methodology used in~\citep{wang2010information} for graphical model selection from i.i.d. data. 
\\

To construct a restricted ensemble, fix a set $T \subset [p]$ with $|T| = d$ and define a base graph $G_0 \in \mathcal{G}_{p,d}$ with edge set $K_T$, the clique on set $T$.  For some $\theta > a >0$ define the base precision matrix 
\begin{align}
    \Theta_0 = \theta I + a (\mathds{1}_T\mathds{1}_T^{\top} - Diag(\mathds{1}_T))
\end{align}

For each unordered pair $r = \{u,v\} \subset T^c, r = 1, 2, \dots, M={p-d \choose 2}$ define a graph $G_r$ with edge set $K_T \cup r$ and an associated precision matrix
\begin{align}
    \Theta_r = \Theta_0 + a(e_ue_v^{\top} + e_ve_u^{\top}).
\end{align} 

The next lemma verifies that this restricted ensemble lies in the hard subclass.

\begin{lemma}[Restricted ensemble is admissible]
\label{lem:restricted-ensemble-admissible}
For every $r \in [M]$, the matrix $\Theta_r$ is symmetric positive definite, has
constant diagonal equal to $\theta$, and its graph $G_r$ has maximum degree at most
$d$. Moreover,
\begin{equation*}
\beta^*(\Theta_r) = \frac{a}{\theta}.
\end{equation*}
Hence, if $\theta \in [\theta_{\min},\theta_{\max}]$ and
$a/\theta \ge \bmin$, then
$\phi_{\Theta_r} \in \mathcal{H}_{p,d}(\bmin,\theta_{\min},\theta_{\max})$ for every
$r \in [M]$.
\end{lemma}

\begin{proof}
First we show that $\Theta_{r}$ is symmetric positive definite. Let $P$ be a permutation matrix such that $P\Theta_{r}P^{\top}$ is block diagonal with blocks
\begin{equation*}
\theta I_{p-d-2},
\qquad
\begin{pmatrix}
\theta & a\\
a & \theta
\end{pmatrix},
\qquad
(\theta-a)I_d + aJ_d.
\end{equation*}
Where, $I_{n}, n > 0$ is an $n \times n$ identity matrix and $J_{n}$ is an $n \times n$ matrix with all ones.  We next show that each block of the matrix $P\Theta_{r}P^{\top}$ is symmetric positive definite. Since $a < \theta$, the $2 \times 2$ block has eigenvalues $\theta - a > 0$, and the clique block has eigenvalues $\theta-a$ with multiplicity $d-1$ and $\theta+(d-1)a$, which are also positive. This implies, $P\Theta_{r}P^{\top} \succ 0$. Since $P\Theta_{r}P^{\top}$ and $\Theta$ are similar, i.e. have the same eigenvalues, $\Theta_r$ is also positive definite. Since the blocks are all symmetric, and remain symmetric under simultaneous permutation or rows and columns, the first part of the claim follows. 

For the second part, notice that its diagonal entries are all equal to $\theta$. The graph $G_r$ consists of the fixed
clique on $T$ together with one additional edge in $T^c$, so its maximum degree is
$d-1 \le d$. Finally, every nonzero off-diagonal entry of $\Theta_r$ is equal to $a$,
while every diagonal entry is equal to $\theta$, so
\begin{equation*}
\beta^*(\Theta_r) = \min_{\{i,j\}\in E(G_r)} \frac{|\Theta_{r,ij}|}{\Theta_{r,ii}} = \frac{a}{\theta}.
\end{equation*}
\end{proof}

For the lower bound, we may reveal the fixed set $T$ and the parameter
pair $(\theta, a)$ to the decoder. This can only make the inference problem easier. Let $V$ be a uniform random variable over the index set $[M]$. It is therefore enough to lower bound the error probability for estimating the index $V \in [M]$ of the extra edge. Let $I_1,\dots,I_{n-1}$ be the updated coordinates of the Glauber dynamics. Under Gaussian Glauber dynamics, each transition changes exactly one coordinate with
probability one, so $\{I_t\}$ is recoverable from the observed path. For the KL
calculations it is convenient to work with the augmented observation
\begin{equation*}
\mathbf{Z}^{(n)} := (\mathbf{X}_0^{n-1}, I_1^{n-1}).
\end{equation*}
Let $\Pbb_v^{(n)}$ denote the law of $\mathbf{Z}^{(n)}$ when the target is $\phi_{\Theta_v}$,
and let $\pi_v$ be the stationary Gaussian law $\mathcal{N}(0,\Theta_v^{-1})$.

\begin{lemma}[Pathwise KL divergence to the base model]
\label{lem:path-kl}
For every $v\in[M]$,
\begin{equation*}
\KL\left(\Pbb_v^{(n)}  \middle \| \Pbb_0^{(n)}\right) \leq \frac{a^2}{\theta^2-a^2}\left(1+\frac{n}{p}\right) = \frac{\bmin^2}{1-\bmin^2}(1+T),
\end{equation*}
Let $\Pbb_v^{(n)}$ denote the law of $\mathbf{Z}^{(n)}$ when the target is $\phi_{\Theta_v}$,
with the last equality occurring when $a = \theta\bmin$.
\end{lemma}

\begin{proof}
By the chain rule for KL divergence,
\begin{align*}
\KL\!\left(\Pbb_v^{(n)} \,\middle\|\, \Pbb_0^{(n)}\right)
&=
\KL(\pi_v\|\pi_0)
+ \sum_{t=1}^{n-1}
\Ebb_v\!\left[
\KL\!\left(
\Pbb_v\!\left(X^{(t)} \mid X^{(t-1)}, I_t\right)
\middle\|
\Pbb_0\!\left(X^{(t)} \mid X^{(t-1)}, I_t\right)
\right)
\right].
\end{align*}

KL divergence between two Gaussians $\Pbb = \mathcal{N}(\mu_{1}, \Theta_{1}), \mathbb{Q} = \mathcal{N}(\mu_{2}, \Theta_{2})$ is given by,
\begin{equation}\label{eqn: KL}
        \KL\left(\mathbb{P} || \mathbb{Q}\right) = \frac{1}{2} \left[\log\left(\det (\Theta_{1}\Theta_{2}^{-1})\right) - p  + (\mu_{1} - \mu_{2})^{\top}\Theta_{1}(\mu_{1} - \mu_{2}) + tr\{\Theta_{2}\Theta_{1}^{-1}\}\right].
\end{equation}

Note that KL divergence is invariant under simultaneous permutation because the determinant term, the quadratic term and the trace term are invariant to index relabeling. Further for Gaussians, block-diagonal precision matrix implies block-diagonal covariance so the KL divergence decomposes into a sum over the independent blocks. Therefore, the first term in Equation~\ref{eqn: KL} depends only on the $2\times 2$ block corresponding to the unknown
edge $e_v=\{u_v,w_v\}$. Since under $\Pbb_0$ that block is $\theta I_2$ and under
$\Pbb_v$ it is
\begin{equation*}
\begin{pmatrix}
\theta & a\\
a & \theta
\end{pmatrix},
\end{equation*}
a direct Gaussian KL computation gives
\begin{equation*}
\KL(\pi_v\|\pi_0)
=
\frac{1}{2}
\left[
\frac{2a^2}{\theta^2-a^2}
+ \log\!\left(1-\frac{a^2}{\theta^2}\right)
\right]
\leq
\frac{a^2}{\theta^2-a^2}.
\end{equation*}

For the transition terms, the Glauber kernels under $\Pbb_v$ and $\Pbb_0$ coincide
unless $I_t \in \{u_v,w_v\}$. If $I_t=u_v$, then under $\Pbb_v$,
\begin{equation*}
X_{u_v}^{(t)} \mid X^{(t-1)}, I_t=u_v
\sim
N\!\left(-\frac{a}{\theta}X_{w_v}^{(t-1)},\, \frac{1}{\theta}\right),
\end{equation*}
whereas under $\Pbb_0$,
\begin{equation*}
X_{u_v}^{(t)} \mid X^{(t-1)}, I_t=u_v
\sim
N\!\left(0,\,\frac{1}{\theta}\right).
\end{equation*}
Since the variances agree, the conditional KL divergence equals
\begin{equation*}
\frac{\theta}{2}
\left(\frac{a}{\theta}X_{w_v}^{(t-1)}\right)^2
=
\frac{a^2}{2\theta}\bigl(X_{w_v}^{(t-1)}\bigr)^2.
\end{equation*}
Taking expectation under stationarity yields
\begin{equation*}
\Ebb_v\!\left[
\KL\!\left(
\Pbb_v\!\left(X^{(t)} \mid X^{(t-1)}, I_t=u_v\right)
\middle\|
\Pbb_0\!\left(X^{(t)} \mid X^{(t-1)}, I_t=u_v\right)
\right)
\right]
=
\frac{a^2}{2(\theta^2-a^2)},
\end{equation*}
because
\begin{equation*}
\Var_{\pi_v}(X_{w_v}) = \frac{\theta}{\theta^2-a^2}.
\end{equation*}
The same bound holds when $I_t=w_v$, and the conditional KL is zero for all other
update coordinates. Since $I_t$ is uniform on $[p]$,
\begin{equation*}
\Ebb_v\!\left[
\KL\!\left(
\Pbb_v\!\left(X^{(t)} \mid X^{(t-1)}, I_t\right)
\middle\|
\Pbb_0\!\left(X^{(t)} \mid X^{(t-1)}, I_t\right)
\right)
\right]
= 
\frac{a^2}{p(\theta^2-a^2)}.
\end{equation*}
Summing over $t=1,\dots,n-1$ gives the claim.
\end{proof}

We now state the lower bound on the observation time $T$ for graphical model selection.

\begin{theorem}[Lower bound for graph recovery from Glauber dynamics]
Assume $0 < \bmin \leq b < 1$ for some $b > 0$ and $2 \leq d \leq p-2$. There exists a universal
constant $c>0$ such that if
\begin{equation*}
T \leq c \, \frac{\log \binom{p-d}{2}}{\bmin^2},
\end{equation*}
then
\begin{equation*}
\mathcal{M}_{n,p,d}\bigl(\mathcal{H}_{p,d}(\bmin,\theta_{\min},\theta_{\max})\bigr)
\geq \frac{1}{2}.
\end{equation*}
Consequently,
\begin{equation*}
\mathcal{M}_{n,p,d}\bigl(\mathcal{G}_{p,d}(\bmin,\theta_{\min},\theta_{\max})\bigr)
\geq \frac{1}{2}.
\end{equation*}
Equivalently, if some estimator achieves probability of error strictly smaller than
$1/2$ uniformly over
$\mathcal{G}_{p,d}(\bmin,\theta_{\min},\theta_{\max})$, then necessarily
\begin{equation*}
T \geq c \, \frac{\log(p-d)}{\bmin^2}.
\end{equation*}
\end{theorem}

\begin{proof}[Proof of Theorem~\ref{thm:lower-bound-glauber}]
Let $V$ be uniform on $[M]$, where $M=\binom{p-d}{2}$ indexes the location of the
extra edge in the restricted ensemble. Since revealing $(T,\theta,a)$ can only help
the decoder, any lower bound for estimating $V$ from $\mathbf{Z}^{(n)}$ also lower bounds the
original graph recovery problem.

By Fano's inequality,
\begin{equation*}
\inf_{\psi}
\max_{v\in[M]}
\Pbb_v^{(n)}[\psi(\mathbf{Z}^{(n)}) \neq v]
\geq
1 - \frac{I(V;\mathbf{Z}^{(n)}) + \log 2}{\log M}.
\end{equation*}

Using
\begin{equation*}
I(V;\mathbf{Z}^{(n)})
\leq
\frac{1}{M}\sum_{v=1}^M \KL\!\left(\Pbb_v^{(n)}\middle\|Q\right)
\end{equation*}
valid for any reference measure $Q$; take $Q=\Pbb_0^{(n)}$. Then
\begin{equation*}
I(V;\mathbf{Z}^{(n)})
\leq
\frac{1}{M}\sum_{v=1}^M \KL\!\left(\Pbb_v^{(n)}\middle\|\Pbb_0^{(n)}\right)
\leq
\frac{\bmin^2}{1-\bmin^2}(1+T).
\end{equation*}
Putting it together we have, 
\begin{equation*}
\inf_{\psi}
\max_{v\in[M]}
\Pbb_v^{(n)}[\psi(\mathbf{Z}^{(n)}) \neq v]
\geq
1 - \frac{\frac{\bmin^2}{1-\bmin^2}(1+T) + \log 2}{\log M}.
\end{equation*}
Given that $\bmin \leq b < 1$ there exists a constant $c>0$ such 
\begin{equation*}
\frac{\bmin^2}{1-\bmin^2} \leq c \bmin^2.
\end{equation*}
Therefore, for the same $c$, if 
$T \leq c \, \frac{\log M}{\bmin^2},$
then the right-hand side above is at least $1/2$. This proves
\begin{equation*}
\mathcal{M}_{n,p,d}\bigl(\mathcal{H}_{p,d}(\bmin,\theta_{\min},\theta_{\max})\bigr)
\geq \frac{1}{2}.
\end{equation*}
Since the hard subclass is contained in the full class, the same lower bound holds
for $\mathcal{G}_{p,d}(\bmin,\theta_{\min},\theta_{\max})$.

Finally, since
\begin{equation*}
\log \binom{p-d}{2} \asymp \log(p-d)
\qquad \text{for } p-d\geq 4,
\end{equation*}
the equivalent observation-time lower bound
\begin{equation*}
T \geq c \, \frac{\log(p-d)}{\bmin^2}
\end{equation*}
follows after adjusting the universal constant.
\end{proof}

\section{Proofs of Lemmas}\label{sec: proofs of lemmas}
\setcounter{proposition}{0}

\subsection{Proof of Lemma \ref{lemma: intuition}}
\begin{lemma}
Under the idealized conditions on $n_{0},n_{1},n_{2},n_{3},n_{4}, n_5$ described above, for any $\bar{x}\in \mathbb{R}^{\abs{N(i)\setminus j}}$ we have,
\begin{align}\label{appeqn: intuition}
	\abs{\E[]{\left(X_{i}^{(n_{4})}-X_{i}^{(n_{1})}\right)\left( X_{j}^{(n_{3})}-X_{j}^{(n_{1})} \right) \middle |  \mathbf{X}^{(n_{0})}_{N(i)\setminus j}=\bar{x}}} \geq \abs{\beta_{ij}}\theta_{jj}^{-1}.
\end{align}
\end{lemma}
\begin{proof}
We write $\E[\bar{x}]{\cdot} = \E[]{\cdot \middle | \mathbf{X}_{N(i) \setminus j} = \bar{x}} $. Let $\Delta_{i}:=X_{i}^{(n_{4})}-X_{i}^{(n_{1})}$, $\Delta_{j}:=X_{j}^{(n_{3})}-X_{j}^{(n_{1})}$, and $\Delta_{\epsilon}:=\epsilon_{i}^{(n_{4})}-\epsilon_{i}^{(n_{1})}$.
Note that $X_{j}^{(n_{0})} = X_{j}^{(n_{1})}$ since $X_{j}$ remains the same until its first update at $n_{3}$ in the interval. Further, given that node $i$ was updated only at $n_{1}, n_{2}$ and $n_{4}$, we have the following:
\begin{align*}
    \Delta_{i} = \beta_{ij}\Delta_{j} + \Delta_{\epsilon}.
\end{align*}
So, 
\begin{align*}
     \E[\bar{x}]{\Delta_{i}\Delta_{j}} & = \E[\bar{x}]{\beta_{ij}\Delta_{j}^{2}} + \E[\bar{x}]{\Delta_{\epsilon}\Delta_{j}} \\
     & = \beta_{ij} \E[\bar{x}]{\Delta_{j}^{2}}
\end{align*}

Note that the second term is zero as,
\begin{align*}
\mathbb{E}_{\bar{x}}\left[\Delta_\varepsilon \Delta_j\right] 
&= \mathbb{E}_{\bar{x}}\left[\left(\varepsilon_i^{(n_4)} - \varepsilon_i^{(n_1)}\right) \cdot \left(X_j^{(n_3)} - X_j^{(n_1)}\right)\right] \\
&\stackrel{(a)}{=} \mathbb{E}_{\bar{x}}\left[-\varepsilon_i^{(n_1)} \cdot X_j^{(n_3)}\right] \\
&\stackrel{(b)}{=} \mathbb{E}_{\bar{x}}\left[-\varepsilon_i^{(n_1)} \cdot \left[\sum_{r \in N(j) \setminus \{i\}} \beta_{jr} X_r^{(n_3-1)} + \beta_{ji} X_i^{(n_2)} + \varepsilon_j^{(n_3)}\right]\right] \\
&\stackrel{(c)}{=} \mathbb{E}_{\bar{x}}\left[-\varepsilon_i^{(n_1)} \cdot \beta_{ji} X_i^{(n_2)}\right] \\
&\stackrel{(d)}{=} \mathbb{E}_{\bar{x}}\left[-\varepsilon_i^{(n_1)} \cdot \beta_{ji}\left[\sum_{r \in N(i)} \beta_{ir} X_r^{(n_1)} + \varepsilon_i^{(n_2)}\right]\right]
\end{align*}

Equality (a) follows from the fact that noise term $\epsilon_{i}^{(n_{4})}$ added at node $i$ at round $n_{4}$ is independent of the $X_{k}^{(n)}$, $n < n_{4}$ and all $k \in  [p]$. That is $\epsilon_{i}^{(n_{4})} \perp X_{j}^{(n_{3})}, X_{j}^{(n_{1})}$. Equations (b) - (c) expand the $X_j^{(n_3)}$ and $X_i^{(n_2)}$ while all other terms zero out being independent of $\epsilon_i^{(n_1)}$. In particular, $\varepsilon_{i}^{(n_{1})}$ enters only $X_{i}^{(n_{1})}$, which is overwritten by the second $i$-update at $n_{2}$; further propagation to $X_{r}^{(n_{3}-1)}$ for $r\in N(j)\setminus\{i\}$ would require an update of some $k\in N(i)$ on $(n_{1},n_{3})$, which is forbidden by the idealized window. Hence each $X_{r}^{(n_{3}-1)}$ in the sum is independent of $\varepsilon_{i}^{(n_{1})}$, as is the fresh noise $\varepsilon_{j}^{(n_{3})}$. Finally in (d) since $r \neq i$ and none of the neighbors of node $i$ were updated, $X_r^{(n_1)}$ is independent of $\varepsilon_i^{(n_1)}$ and the two noises added during the back-to-back updates at $i$ are also independent. 

Now we look at the first term. Suppose there are $n > 0$ updates between $n_{0}$ and $n_{5}$. Let $\sigma = \left(\sigma_{1}, \dots, \sigma_{n}\right)$ be the sequence of nodes that are updated in the interval $[n_{0}, n_{5}]$. Then, 
\begin{align*}
\E[\bar{x}]{\Delta_{j}^{2}} 
 {=} \E[]{\E[\bar{x}]{\Delta_{j}^{2} |\sigma}} &\stackrel{(a)}{=} \E[]{\E[\bar{x}]{\left(X_{j}^{(n_{3})} - X_{j}^{(n_{1})}\right)^2 \mid \sigma}} \\
& \stackrel{(b)}{=} \E[]{\E[\bar{x}]{ \left(\underbrace{ \sum_{k \neq j} \beta_{jk}X_{k}^{(n_{3} - 1)} - X_{j}^{(n_{1})} }_{ A } + \epsilon_{j}^{(n_{3})} \right)^2 \mid \sigma}} \\
& {=} \E[]{\E[\bar{x}]{\left(A + \epsilon_{j}^{(n_{3})} \right)^2 \mid \sigma}} \\
& {=} \E[]{\E[\bar{x}]{A^2 + (\epsilon_{j}^{(n_{3})})^2 + 2 A \epsilon_{j}^{(n_{3})} \mid \sigma}} \\
& \stackrel{(c)}{\geq} \E[]{\E[\bar{x}]{\left(\epsilon_{j}^{(n_{3})}\right)^2 \mid \sigma}} \\
& \geq \theta_{jj}^{-1}
\end{align*}    

(a) follows from definition the of $\Delta_{j}$. Expanding $X_{j}^{(n_{3})}$ and defining $A$ gives (b).
Inequality (c) follows from the fact that the random variable $A$ is measurable given information upto round $n_{3} - 1$. The noise added at $n_{3}$ is independent of $A$. Lower bounding with variance of the noise terms gives us the result.

Combining the equality $\E[\bar x]{\Delta_{i}\Delta_{j}} = \beta_{ij}\,\E[\bar x]{\Delta_{j}^{2}}$ with the bound $\E[\bar x]{\Delta_{j}^{2}}\geq \theta_{jj}^{-1}$ and taking absolute values,
\begin{align*}
\abs{\E[\bar x]{\Delta_{i}\Delta_{j}}} = \abs{\beta_{ij}}\,\E[\bar x]{\Delta_{j}^{2}} \geq \abs{\beta_{ij}}\theta_{jj}^{-1}.
\end{align*}
\end{proof}

\subsection{Proof of Lemma 2 \& Lemma 3}

As discussed in Section~\ref{subsec: Glauber dynamics}, the update process can be seen as associating an independent Poisson clock with rate 1 with each random variable, and updating the corresponding random variable at the arrival time accordingly. This view is helpful to calculate the probabilities associated with events $U_{ij}^{k}$ and $Q_{ij}^{k}$ for any $i,j \in [p]$ and $k \in [k_{\max}]$. We restate and prove both Lemma~\ref{appseclemma: Probability of event U} and \ref{appseclemma: Probability of event Q} in this section.
\begin{lemma}[Probability of update event $U_{ij}^k$]\label{appseclemma: Probability of event U}
    For $k \in [k_{\max}]$, $i,j \in [p]$ we have $$\prob[]{U_{ij}^{k}} = \left[(1-e^{-\tau/4})e^{-\tau/4}\right]^{4} \triangleq q.$$
\end{lemma}
\begin{proof}
Recall that for $i,j \in [p]$ and any $k\in [k_{\max}]$, we define
    \begin{align*}
       U_{ij}^{k} =
        \bigg\{&\exists t \in W_{1} : I_{t}= i,  \forall t \in W_{1} : I_{t} \neq j \bigg\}  \bigcap
        \bigg\{\exists t \in W_{2}:I_{t}  = i , \forall t \in W_{2}:I_{t} \neq j  \bigg\}  \\ &\bigcap
        \bigg\{\exists t \in W_{3}:I_{t}  = j , \forall t \in W_{2}:I_{t} \neq i  \bigg\} \\ &\bigcap
        \bigg\{\exists t \in W_{4}: I_{t} = i, \forall t \in W_{3}: I_{t} \neq j 
        \bigg\}.
    \end{align*}

where $W_{1} = [(k-1)\tau,(k-3/4)\tau), 
         W_{2} = [(k-3/4)\tau,(k-1/2)\tau), 
         W_{3} = [(k-1/2)\tau,(k-1/4)\tau), 
         \text{ and } 
         W_{4} = [(k-1/4)\tau,k\tau),$.The $k-$th $\tau$ length interval being divided into fourths gives us the probability observing at least one update of node $i$ and no updates on node $j$ in the first fourth as 
\begin{align*}
    \prob[]{\exists t \in W_{1} : I^{(t)} = i,  \forall t \in W_{1} : I^{(t)} \neq j} & = \prob[]{\exists t \in W_{1} : I^{(t)} = i} \cdot \prob[]{  \forall t \in W_{1} : I^{(t)} \neq j} \\ &=
    \prob[]{N>0}\cdot\prob[]{N=0} \text{ where }N\sim \operatorname{Pois}(\tau/4) \\
    &= (1-e^{-\tau/4})e^{-\tau/4}.
\end{align*}
Recall that $I_{t}$ represents the identity of the last updated node on or before time $t$ . That is, if $I_{t} = i$ then there exists an update times $S_{n} \leq t < S_{n+1}$ such that $I^{(n)} = i$. As the number of update times (per node) occurring in any interval $\tau$ has Poisson distribution with mean $\tau$ we get the second equality. We get the required result by using independence of the associated Poisson clocks for $W_{2}, W_{3}$ and $W_{4}$ sections of the considered $k$-th interval.
\end{proof}

\begin{lemma} \label{appseclemma: Probability of event Q}The probability of event $Q_{ij}^{k}$ for any $k \in [k_{\max}]$ we have $\prob[]{Q_{ij}^{k}} \geq e^{-\tau d}.$
\end{lemma}
\begin{proof}
    The proof simply follows from the Poisson distribution and the definition of the event $Q_{ij}^{k}$. We have,
    \begin{align*}
        \prob[]{Q_{ij}^{k}} = \prob[]{N = 0}^{\abs{N(i)\setminus \{j\}}} \geq e^{-\tau d}
    \end{align*}
    where $N \sim \operatorname{Pois}(\tau)$.
\end{proof}

\subsection{Proof of Lemma 4}
We state and prove Lemma~\ref{lemma: High probability event B} in this section. We use the fact that if the maximum over some $N_{\max} > 0$ updates of $\epsilon_{I^{(n)}}^{(n)}, n \leq N_{\max}$ is bounded, then the maximum of the data samples $\mathbf{Y}^{(n)}$ are also bounded by the same quantity up to a factor of $1/(1-d\bmax)$. Assumption A3 plays an important role in the proof.
\begin{lemma}[Event B]\label{lemma: High probability event B}
    For any $\delta > 0$ and sufficiently large $p$ there exists constants $C_{1}, C_{1}', C_{1}'' > 0$ such that 
    \begin{align*}
        \prob[]{\max_{i\in [p], t < C_1' p^{C_1''}} \left| Y_i^{(t)} \right| \leq C_1 \sigma_{\max} \sqrt{\log \left( \frac{p}{\delta} \right)}} \geq 1 - \delta/2
    \end{align*}
\end{lemma}
\begin{proof}

We first establish a few claims to prove this lemma. 
\begin{claim}\label{claim: prob of eps bounded}
    Let $N_{\max} = T_{\max}p = C_{1}'p^{C_{1}''+1}$ and let $Z = \max_{n \leq N_{\max}}\abs{\epsilon_{I^{(n)}}^{(n)}}$. Then there exists constants $\tilde{C_{1}}$ such that  
    \begin{align*}
    \prob[]{Z \leq \varepsilon_{\max}} \geq 1 - \delta / 2.
    \end{align*}
    where $\varepsilon_{\max} = \tilde{C}_{1}\sigmamax\sqrt{\log\left(\frac{p}{\delta}\right)}$.
\end{claim}
\begin{proof}
First we bound MGF of $Z$ using the MGF of a folded normal distribution, which  is given by $M(t) = 2e^{\frac{\sigma^2 t^2}{2}} \Phi(\sigma t)$,
where $\Phi(\cdot)$ is the CDF of the standard normal distribution. Therefore,
\begin{align*}
\E[]{\exp(tZ)} &= \E[]{ \exp\left( t \max_{n \leq N_{\max}} \abs{\epsilon_{I^{(n)}}^{(n)}} \right)}\\
&\stackrel{(a)}{=} \E[]{\max_{n \leq N_{\max}} \exp\left( t \abs{\epsilon_{I^{(n)}}^{(n)}} \right)} \\
&\stackrel{(b)}{<} \sum_{i=1}^{N_{\max}} \E[]{ \exp\left( t \abs{\epsilon_{I^{(n)}}^{(n)}} \right) } \\
&\stackrel{(c)}{\leq} N_{\max} \left( 2 e^{\frac{t^2 \sigmamax^{2}}{2}} \Phi( \sigmamax t) \right) \\
&\stackrel{(d)}{\leq} N_{\max} \left( 2 e^{\frac{t^2  \sigmamax^2}{2}} \right).
\end{align*}

Equality (a) follows from the exponential being monotonically increasing. Inequality (b) follows from the fact $\max_{i} p_{i} < \sum_{i}p_{i}$ when $p_{i >0}$. Inequality (c) follows from using the MGF of the folded normal distribution and the fact that variance of $\epsilon_{I^{(n)}}^{(n)}$ is atmost $\sigmamax^{2}$. While (d) follows from $\Phi(x) \leq 1, x \in \mathbb{R}$.

Next using Chernoff's technique we get $\Pr(Z \geq \varepsilon) \leq 2N_{\max}e^{-\frac{\varepsilon^2}{2 \sigmamax^{2}}}$. Setting this bound to $\delta/2$ implies $\varepsilon > \sigmamax \sqrt{2 \log \left( \frac{4C_{1}'p^{C_{1}''+1}}{\delta} \right)}$. Therefore we have,
\begin{align*}
    \Pr\left\{ Z > \sigmamax \sqrt{2 \log \left( \frac{4C_{1}'p^{C_{1}''+1}}{\delta} \right)} \right\} \leq \delta / 2.
\end{align*}
Let $X^{(n)}_{\max} \triangleq \max_{i \in [p]} |X_{i}^{(n)}|$. Now for sufficiently large $p$ there exists constant $\tilde{C}_{1}$ such that 
\begin{align}\label{eqn: emax}
    \max\left\{\sigmamax \sqrt{2 \log \left( \frac{4C_{1}'p^{C_{1}''+1}}{\delta} \right)}, X^{(0)}_{\max}\right\} < \tilde{C}_{1}\sigmamax\sqrt{\log\left(\frac{p}{\delta}\right)}.
\end{align}
 Let $\varepsilon_{\max} \triangleq \tilde{C}_{1}\sigmamax\sqrt{\log\left(\frac{p}{\delta}\right)}$. Therefore we have, 
\begin{align*}
    \prob[]{Z \leq \varepsilon_{\max}} \geq 1 - \delta / 2.
\end{align*}
\end{proof}

Now we relate this event $Z \leq \varepsilon_{\max}$ to the event in the statement of this Lemma using the following claim.

\begin{claim}\label{claim: eps subset X}
    For $n \leq N_{\max}, \varepsilon_{\max} > 0$ the following relation between the events holds:
    \begin{align}
        \left\{\max_{n \leq N_{\max}} \abs{\epsilon_{I^{(n)}}^{(n)}} < \varepsilon_{\max}\right\} \subseteq \left\{\max_{n \leq N_{\max}} |X_{I^{(n)}}^{(n)}| \leq \frac{\varepsilon_{\max}}{1-d\beta_{\max}}\right\}.
    \end{align}
\end{claim}

\begin{proof}
    
Recall that $\varepsilon_{\max}$ chosen in \ref{eqn: emax} is such that  $X^{(0)}_{\max} \leq \varepsilon_{\max}$, where $X^{(0)}_{\max} \geq \abs{X_{i}^{(0)}}$ for all $i \in [p]$. Suppose for the next round node $i$ was updated, that is $I^{(1)} = i$.  Therefore, for node $i$ we have 
\begin{align*}
    \abs{X^{(1)}_{i}} \leq \sum_{j \in N(i)}\beta_{ij}\abs{X_{j}^{(0)}} + \abs{\epsilon_{i}^{(1)}}  \leq d\bmax X^{(0)}_{\max} + \varepsilon_{\max} \leq \varepsilon_{\max}(1 + d\beta_{\max})
\end{align*}
while all other nodes remain unchanged. Then, it follows that 
\begin{align*}
    X_{\max}^{(1)} = \max\left\{X_{\max}^{(0)}, \abs{X_{i}^{(1)}}\right\} \leq \varepsilon_{\max}(1 + d\beta_{\max}).
\end{align*}
Continuing this way for $n > 1$, and using the fact that $d\beta_{\max} < 1$ from Assumption~\ref{assumption: bounded degree and sample decay} we get, 
\begin{align*}
    X^{(n)}_{\max} & < \varepsilon_{\max}\left(1 + d\bmax + (d\bmax)^{2} + \dots + (d\bmax)^{n}\right) \leq \frac{\varepsilon_{\max}}{1-d\bmax}.
\end{align*}
\end{proof}
Now we just combine the two claims and the result follows as
\begin{align*}
    \prob[]{\max_{n \leq N_{\max}} |X_{I^{(n)}}^{(n)}| \leq \frac{\varepsilon_{\max}}{1-d\bmax}} \geq \prob[]{Z \leq \varepsilon_{\max}} \geq  1 - \delta / 2.
\end{align*}
where $C_{1} = \tilde{C}_{1}/(1 -d\bmax)$.
\end{proof}

\subsection{Proof of Lemma 6}

In this section we state and prove Lemma~\ref{appseclemma: Test statistic bounded as}.
\begin{lemma}\label{appseclemma: Test statistic bounded as}
    For $k \in [k_{\max}]$ and $i,j \in [p]$ the $k$-th term of the test statistic is bounded almost surely with respect to the conditional measure $\mathbb{P}_{B}$. That is,
    \begin{align*}
         \abs{T_{ij}^{k}} \leq 4 C_1^{2}\sigmamax^{2} \log \left( \frac{p}{\delta} \right) \quad(\text{a.s.}).
    \end{align*}
\end{lemma}
\begin{proof}
    Under the conditional measure $\mathbb{P}_{B}$ we show that the complement of the above event occurs with probability 0. That is,
    \begin{align*}
        \prob[B]{\abs{T_{ij}^{k}} >
        4 C_1^{2}\sigmamax^{2} \log \left( \frac{p}{\delta} \right)} & \stackrel{(a)}{=}\prob[B]{\abs{\Delta Y_{i}^{k}\Delta Y_{j}^{k}} >
       4 C_1^{2}\sigmamax^{2} \log \left( \frac{p}{\delta} \right) \middle | U} \prob[B]{U} \\
        & \stackrel{(b)}{\leq} \prob[B]{\max_{t = i,j}\left\{\abs{\Delta Y_{t}^{k}}\right\} >
        2C_1 \sigma_{\max} \sqrt{\log \left( \frac{p}{\delta} \right)} \middle | U} = 0 \\
    \end{align*}
Equation $(a)$ follows from conditioning on events $U_{ij}^{k}$ and using the definition of the test statistic. Inequality $(b)$ follows from bounding the probability of the update event $U$ with 1.
\end{proof}

\subsection{Proof of Lemma 7 (Separation)}
We state Lemma\ref{appseclemma: separation2} and prove it.
\begin{lemma}[Separation]\label{appseclemma: separation2}
    Let $\{\mathcal{F}_{k}\}_{k \geq 1} = \left(\sigma(\{\mathbf{Y}^{(t)}\}_{t=0}^{k\tau}) \right)_{k \geq 1}$ be a filtration. For any $k$,
    if a pair $\{i,j\} \notin E$, we have
    \begin{align}\label{eqn: eta}
        \abs{\E[B]{T_{ij}^{k} \middle | \mathcal{F}_{k-1}}} \leq 4\ymax^{2}q\tau d
    \end{align}
    and if $\{i,j\} \in E$,
    \begin{align}\label{eqn : eta'}
        \operatorname{sign}(\beta_{ij})\E[B]{T_{ij}^{k} \middle | \mathcal{F}_{k-1}} > \operatorname{sign}(\beta_{ij})\bmin\sigmamin(1-\tau d)q -4\ymax^{2}q\tau d
    \end{align}
    Furthermore, if \begin{align*}
   \tau = d^{-1}\left[\frac{12\ymax^{2}}{\sigmamin\bmin} + 1\right]^{-1}
\end{align*}
then 
\begin{align*}
    \bmin\sigmamin(1-\tau d)q = 3\left(4\ymax^{2}q\tau d\right).
\end{align*}

That is, 
\begin{align*}
     \operatorname{sign}(\beta_{ij})\E[B]{T_{ij}^{k} \middle | \mathcal{F}_{k-1}} > 2 (4\ymax^{2}q\tau d).
\end{align*}
\end{lemma}
\begin{proof}

For convenience, we will fix a pair $i,j \in [p]$ and $k \in [k_{\max}]$, and we will write the events without the subscripts or superscripts. Additionally, we write the conditional expectation of the $k^{th}$ test statistic as $\E[B,\mathcal{F}]{T_{ij}} = \E[B]{T_{ij}^{k} | \mathcal{F}_{k-1}}$. We begin by noting that conditioning on the events $U$ and $Q$ gives us
\begin{align*}
   \E[B,\mathcal{F}]{T_{ij}} & =  \E[B,\mathcal{F}]{T_{ij} | UQ}\prob[B,\mathcal{F}]{UQ} + \E[B,\mathcal{F}]{T_{ij} | UQ^{c}}\prob[B,\mathcal{F}]{UQ^{c}}.
\end{align*}
This follows from the fact that the test statistic is zero when the event $U_{ij}^{k}$ does not occur. Additionally, when the event $Q$ occurs we have,
\begin{align*}
    \E[B,\mathcal{F}]{T_{ij} | UQ} = \E[B,\mathcal{F}]{\Delta Y_{i}\Delta Y_{j} | UQ} \geq \E[B,\mathcal{F}]{\beta_{ij}\Theta_{jj}^{-1} + \Delta \epsilon_{i}\Delta Y_{j} | UQ} \geq \beta_{ij}\Theta_{jj}^{-1}.
\end{align*}
That is, the conditional expectation of the test statistic is atleast $\bmin\sigmamin$. Now consider the case $\{i,j\} \notin E$, then $\beta_{ij} = 0$ and we have
\begin{align*}
    \E[B,\mathcal{F}]{T_{ij}} = \E[B,\mathcal{F}]{T_{ij} \middle | UQ^{c}}\prob[B,\mathcal{F}]{U}\prob[B,\mathcal{F}]{Q^{c}}
\end{align*}
The probabilities related to events $U$ and $Q^{c}$ factorize as above due to independent uniform selection of nodes during updates. Therefore we have,
\begin{align*}
    \abs{\E[B,\mathcal{F}]{T_{ij}}} 
    & \stackrel{(a)}{\leq} \abs{\E[B,\mathcal{F}]{T_{ij} \middle | UQ^{c}}}q(1-e^{-\tau d}) \\
    & \stackrel{(b)}{\leq}4\ymax^{2}q\tau d
\end{align*}
where $(a)$ follows from the fact that $\prob[B,\mathcal{F}]{U} = \prob[]{U}$ due to the independent selection of nodes and by definition equal to $q$. Inequality $(b)$ follows from Lemma~\ref{lemma: Test statistic bounded as} and Lemma~\ref{lemma: Probability of event Q} and using $1-e^{-x} < x$. Next, for the case $\{i,j\} \in E$, we have,
\begin{align*}
     \E[B,\mathcal{F}]{T_{ij}} = \beta_{ij}\prob[B,\mathcal{F}]{UQ} + \E[B,\mathcal{F}]{T_{ij} | UQ^{c}}\prob[B,\mathcal{F}]{UQ^{c}}.
\end{align*}
Multiplying $\operatorname{sign}(\beta_{ij})$ both sides and noticing that $\operatorname{sign}(\beta_{ij})\E[B,\mathcal{F}]{T_{ij} \middle | UQ^{c}} \leq \abs{\E[B,\mathcal{F}]{T_{ij} \middle | UQ^{c}}}$ we get,
\begin{align*}
     \operatorname{sign}(\beta_{ij})\E[B,\mathcal{F}]{T_{ij}} 
     & > \operatorname{sign}(\beta_{ij})\beta_{ij}q(e^{-\tau d}) - \abs{\E[B,\mathcal{F}]{T_{ij} \middle | UQ^{c}}} \\
     & > \operatorname{sign}(\beta_{ij})\bmin\sigmamin (1-\tau d)q-4\ymax^{2}q\tau d
\end{align*}
This proves the separation in cases where the edge between nodes $i,j$ exists and otherwise. Now for the chosen $\tau = d^{-1}\left[\frac{12\ymax^{2}}{\sigmamin\bmin} + 1\right]^{-1}$ we have, 
\begin{align*}
    1 &= \tau d \left(\frac{12\ymax^{2}}{\sigmamin\bmin} + 1\right)  \\
    1 - \tau d & = \frac{12\ymax^{2}}{\sigmamin\bmin} \tau d \\
    \bmin \sigmamin(1-\tau d) & = 3\left[4\ymax^{2}\tau d\right].
\end{align*}
Multiplying $q$ on both sides we have the required result.
\end{proof}

\subsection{Bernstein-style Concentration Inequality}

In the proof of Theorem 1, we use the following Bernstein-style concentration inequality. We state Lemma~\ref{appseclemma: Bernstein style} but we do not prove it here. One can find it here as an implication of~\cite[Theorem 27]{ChungConcentration} 

\begin{lemma}[Bernstein style concentration for martingales]\label{appseclemma: Bernstein style}
Let $Z_{1}, Z_{2}, \dots, Z_{k}$ be a submartingale adapted to filtration $\{\mathcal{F}_{k}\}_{k \geq 1}$
such that $\abs{Z_{k} - Z_{k-1}} \leq c, c\geq 0$ almost surely and $\operatorname{var}(Z_{k} | \mathcal{F}_{k-1}) \leq s^{2}, s\geq 0$, then for $t > 0$,
\begin{align*}
\prob{Z_n < Z_{0} - t} \leq \exp \left(-\frac{t^2}{2ns^{2} + ct/3}\right)
\end{align*}
\end{lemma}

\subsection{Proof of Lemma 9}
We now state and prove Lemma~\ref{appseclemma: variance bound} in this section. Recall that $\ymax$ is bound on data samples and $q$ is the probability of the event $U_{ij}^{k}$ as defined in Section~\ref{sec: the algorithm}.

\begin{lemma}[Variance bound]\label{appseclemma: variance bound}
For an adapted process $\left(Z_{k}, 
\mathcal{F}_{k}\right)_{k\geq1}$, where $Z_{k}$ is either $\sum_{l=1}^{k} \operatorname{sign}(\beta_{ij})T_{ij}^{l} - k\eta'$, $\sum_{l=1}^{k} T_{ij}^{l} - k\eta$ or $\sum_{l=1}^{k} T_{ij}^{l} + k\eta$ and $\{\mathcal{F}_{k}\}_{k\geq1} = \left(\sigma(\{\mathbf{Y}^{(t)}\}_{t=0}^{k\tau}) \right)_{k \geq 1}$, then under the conditional measure $\prob[B]{\cdot}$ we have 
\begin{align*}
    \var{Z_{k} | \mathcal{F}_{k-1}} \leq 16\ymax^{4}q.
\end{align*}
\end{lemma}
\begin{proof}
    Let $D_{k} = Z_{k} - Z_{k-1}$. We have, 
    \begin{align*}
        \var{Z_{k} | \mathcal{F}_{k-1}} & \stackrel{(a)}{=} \var{Z_{k-1} + D_{k} | \mathcal{F}_{k-1}} \\
        & \stackrel{(b)}{=}  \var{Z_{k-1} | \mathcal{F}_{k-1}} + \var{D_{k} | \mathcal{F}_{k-1}} + 2\operatorname{cov}\left(Z_{k-1}, D_{k} | \mathcal{F}_{k-1}\right) \\
        & \stackrel{(c)}{=}  \var{D_{k} | \mathcal{F}_{k-1}} + 2\left[\E[\mathcal{F}]{Z_{k-1}, D_{k}} - Z_{k-1}\E[\mathcal{F}]{D_{k}}\right] \\
        & \stackrel{(d)}{=}  \var{D_{k} | \mathcal{F}_{k-1}}
    \end{align*}
Note that for any $Z_{k}$, the conditional variance of its corresponding difference sequence is
\begin{align*}
    \var{D_{k}|\mathcal{F}_{k-1}} &= \var{T_{ij}^{k}|\mathcal{F}_{k-1}} \\
    & \leq \E[\mathcal{F}]{\left(T_{ij}^{k}\right)^{2}} \\
    & = \E[\mathcal{F}]{\Delta Y_{i}^{2}\Delta Y_{j}^{2} \middle | U}\prob[\mathcal{F}]{U} \\
    & \leq 16\ymax^{4}q.
\end{align*}
\end{proof}

\end{document}